\documentclass[journal]{IEEEtran} 

\usepackage{cite}
\usepackage{array}
\usepackage{stfloats}
\usepackage{amsthm,amsmath,amsfonts,amssymb}
\usepackage{graphicx}
\usepackage[colorlinks,citecolor=blue,urlcolor=blue]{hyperref}
\usepackage{enumerate}
\usepackage{booktabs}
\usepackage{multirow}
\usepackage{float}
\usepackage[mathscr]{euscript}
\usepackage{pifont}
\theoremstyle{plain}
\newtheorem{theorem}{Theorem}[section]
\newtheorem{lemma}[theorem]{Lemma}
\newtheorem{corollary}[theorem]{Corollary}
\newtheorem{assumption}{Assumption}
\theoremstyle{remark}
\newtheorem{definition}{Definition}
\newtheorem{example}{Example}
\newtheorem*{remark}{Remark}
\def\vn{{\mathbf{n}}}
\def\calA{{\mathcal{A}}}
\def\calF{{\mathcal{F}}}
\def\calI{{\mathcal{I}}}
\def\calN{{\mathcal{N}}}
\def\calO{{\mathcal{O}}}

\def\bbE{{\mathbb{E}}}
\def\bbN{{\mathbb{N}}}
\def\bbR{{\mathbb{R}}}
\usepackage{mathrsfs}
\def\scrB{{\mathscr{B}}}
\def\euD{{\EuScript{D}}}
\def\euS{{\EuScript{S}}}

\DeclareMathOperator*{\argmin}{arg\,min}
\DeclareMathOperator{\sign}{sign}

\DeclareMathOperator{\pr}{Pr}
\DeclareMathOperator{\vcdim}{VCdim}
\DeclareMathOperator{\app}{app}
\DeclareMathOperator{\gen}{gen}
\DeclareMathOperator{\reg}{reg}
\DeclareMathOperator{\conv}{conv}

\DeclareMathOperator{\supp}{supp}

\newcommand{\what}{\widehat}

\begin{document}

\title{Semi-Supervised Deep Sobolev Regression: \\ Estimation and Variable Selection by ReQU Neural Network}

\author{
Zhao Ding, Chenguang Duan, Yuling Jiao, and Jerry Zhijian Yang
\thanks{This work is supported by the National Key Research and Development Program of China (No. 2024YFA1014202), by the National Natural Science Foundation of China (No. 12125103, No. U24A2002, No. 12371441), and by the Fundamental Research Funds for the Central Universities. (Corresponding author:
 Jerry Zhijian Yang.)}
\thanks{Zhao Ding is with the School of Mathematics and Statistics, Wuhan University, Wuhan, Hubei 430072, China (email: zd1998@whu.edu.cn).}%
\thanks{Chenguang Duan is with the School of Mathematics and Statistics, Wuhan University, Wuhan, Hubei 430072, China (email: cgduan.math@whu.edu.cn).}
\thanks{Yuling Jiao is with the School of Artificial Intelligence, National Center for Applied Mathematics in Hubei, Hubei Key Laboratory of Computational Science, and School of Mathematics and Statistics, Wuhan University, Wuhan, Hubei 430072, China (email: yulingjiaomath@whu.edu.cn).}
\thanks{Jerry Zhijian Yang is with National Center for Applied Mathematics in Hubei, Wuhan Institute for Math \& Al, School of Mathematics and Statistics, and Hubei Key Laboratory of Computational Science, Wuhan University, Wuhan, Hubei 430072, China (email: zjyang.math@whu.edu.cn).}
}



\maketitle

\begin{abstract}
We propose SDORE, a \underline{S}emi-supervised \underline{D}eep S\underline{O}bolev \underline{RE}gressor, for the nonparametric estimation of the underlying regression function and its gradient. SDORE employs deep ReQU neural networks to minimize the empirical risk with gradient norm regularization, allowing the approximation of the regularization term by unlabeled data. Our study includes a thorough analysis of the convergence rates of SDORE in $L^{2}$-norm, achieving the minimax optimality. Further, we establish a convergence rate for the associated plug-in gradient estimator, even in the presence of significant domain shift. These theoretical findings offer valuable insights for selecting regularization parameters and determining the size of the neural network, while showcasing the provable advantage of leveraging unlabeled data in semi-supervised learning. To the best of our knowledge, SDORE is the first provable neural network-based approach that simultaneously estimates the regression function and its gradient, with diverse applications such as nonparametric variable selection. The effectiveness of SDORE is validated through an extensive range of numerical simulations.
\end{abstract}


\section{Introduction}
\IEEEPARstart{N}{onparametric} regression plays a pivotal role in both statistics and machine learning, possessing an illustrious research history as well as a vast compendium of related literature~\cite{gyorfi2002distribution,Wasserman2006all,Tsybakov2009Introduction}. Let $\Omega\subseteq\bbR^{d}$, $d\geq 1$, be a bounded and connected domain with sufficiently smooth boundary $\partial\Omega$. Consider the following nonparametric regression model
\begin{equation}\label{eq:nonparametric:regression}
Y=f_{0}(X)+\xi,
\end{equation}
where $Y\in\bbR$ is the response associated with the covariate $X\in\Omega$, and $f_{0}$ is the unknown regression function. Here $\xi$ represents a random noise term satisfying $\bbE[\xi|X]=0$ and $\bbE[\xi^{2}|X]<\infty$. The primary task of nonparametric regression involves estimating the conditional expectation $f_{0}(x)$ of the response $Y$, given a covariate $X=x$. This estimation is typically achieved through empirical least-squares risk minimization:
\begin{equation*}
\min_{f\in\calF}\frac{1}{n}\sum_{i=1}^{n}(f(X_{i})-Y_{i})^{2},
\end{equation*}
where $\{(X_{i},Y_{i})\}_{i=1}^{n}$ is a set of independently and identically distributed random copies of $(X,Y)$, and $\calF$ is a pre-specific hypothesis class, such as deep ReQU neural network class in this paper. While empirical least-squares risk minimization is straightforward to implement and comes with solid theoretical guarantees, it does not fully meet all desired criteria. One major drawback is that the method places no constraints on the gradient of the estimator, allowing for the possibility of an arbitrarily large gradient norm. This can make the least-squares estimator highly sensitive to the input perturbations. Furthermore, while the least-squares estimator ensures convergence in terms of function values, the convergence in terms of derivatives can not be guaranteed.

\par To address these challenges, Sobolev regularization, also known as gradient penalty, was introduced in deep learning by~\cite{Drucker1991Double,Drucker1992Improving}:
\begin{equation}\label{eq:gp}
\min_{f\in\calF}\frac{1}{n}\sum_{i=1}^{n}(f(X_{i})-Y_{i})^{2}+\frac{\lambda}{n}\sum_{i=1}^{n}\sum_{k=1}^{d}|D_{k}f(X_{i})|^{2},
\end{equation}
where $\lambda>0$ is the regularization parameter, and $D_{k}f$ denotes the partial derivative of $f$ with respect to the $k$-th input variable. Substantial numerical experiments have consistently demonstrated that the imposition of a gradient penalty contributes to the enhancement of the stability and generalization of deep learning models. The strategy surrounding gradient penalty was adopted by~\cite{Rifai2011Contractive} as a technique to learn robust features using auto-encoders. This method was further utilized to augment the stability of deep generative models as highlighted in the work of~\cite{Nagarajan2017Gradient,Roth2017Stabilizing,Gulrajani2017Improved,Gao2022Deep}. Significantly, the gradient norm, being a local measure of sensitivity to input perturbations, has seen a plethora of research focusing on its use for adversarial robust learning. This is reflected in studies conducted by~\cite{Lyu2015Unified,Hein2017Formal,Ororbia2017Unifying,novak2018sensitivity,Jakubovitz2018Improving,Ross2018Improving}.

\IEEEpubidadjcol
\par Simultaneously estimation the regression function and its of gradient (derivatives) carries a wide span of applications across various fields, including the factor demand and cost estimation in economics~\cite{Shephard1981Cost}, trend analysis for time series data~\cite{Rondonotti2007SiZer}, the analysis of human growth data~\cite{Ramsay2002Applied}, and the modeling of spatial process~\cite{Banerjee2003Directional}. Furthermore, estimating gradient  plays a pivotal role in the modeling of functional data~\cite{Muller2010additive,Dai2018Derivative}, variable selection in nonparametric regression~\cite{Rosasco2010regularization,Mosci2012Is,Rosasco2013Nonparametric}, and inverse problems~\cite{Hu2018new}. There are four classical approaches to nonparametric gradient  estimation: local polynomial regression~\cite{Brabanter2013Derivative}, smoothing splines~\cite{Nancy2000Penalized}, kernel ridge regression~\cite{Liu2023Estimation}, and difference quotients~\cite{Muller1987Bandwidth}. However, local polynomial regression and smoothing spline regression are only applicable to fixed-design setting and low-dimensional problems. The generalization of these methodologies to address high-dimensional problems is met with a significant challenge popularly known as the computational curse of dimensionality~\cite{Bellman1961Adaptive,Wasserman2006all,Goodfellow2016Deep}. This phenomenon refers to the fact that the computational complexity can increase exponentially with dimension. In contrast, deep neural network-based methods, which are mesh-free, exhibit direct applicability to high-dimensional problems, providing a solution to mitigate this inherent challenge. The plug-in kernel ridge regression estimators have demonstrated applicability for estimating derivatives across both univariate and multivariate regressions within a random-design setting~\cite{Liu2023Estimation,Liu2023EstimationHypothesis}. However, these estimators present certain inherent limitations compared to deep neural networks. From a computational complexity standpoint, the scale of the kernel grows quadratically or even cubically with the number of samples. In contrast, deep neural networks exhibit the ability to handle larger datasets, especially when deployed on modern hardware architectures.

\par Recently, there has been a substantial literature outlining the convergence rates of deep nonparametric regression~\cite{bauer2019deep,nakada2020Adaptive,schmidt2020nonparametric,kohler2021rate,Farrell2021Deep,Kohler2022Estimation,Jiao2023deep}. However, the theoretical foundation of Sobolev regularized  least-squares using deep neural networks remains relatively underdeveloped. Consequently, two fundamental questions need to be addressed:
\begin{quote}
\emph{\normalsize What accounts for the enhanced stability and superior generalization capacity of the Sobolev penalized estimator compared to the standard least-squares estimator? Furthermore, does the plug-in gradient estimator of the Sobolev penalized regressor close to  the true gradient of the regression function, and if so, what is the corresponding convergence rate?}
\end{quote}

\par In this paper, we introduce SDORE, a \underline{\textbf{S}}emi-supervised \underline{\textbf{D}}eep S\underline{\textbf{O}}bolev \underline{\textbf{RE}}gressor, for simultaneously estimation of both the regression function and its gradient. SDORE leverages deep neural networks to minimize an empirical risk, augmented with unlabeled-data-driven Sobolev regularization:
\begin{equation}\label{eq:semigp}
\min_{f\in\calF}\frac{1}{n}\sum_{i=1}^{n}(f(X_{i})-Y_{i})^{2}+\frac{\lambda}{m}\sum_{i=1}^{m}\sum_{k=1}^{d}|D_{k}f(Z_{i})|^{2},
\end{equation}
where $\{Z_{i}\}_{i=1}^{m}$ is a set of unlabeled data independently and identically drawn from a distribution on $\Omega$. Notably, our methodology does not necessitate alignment of the unlabeled data distribution with the marginal distribution of the labeled data, remaining effective even under significant domain shifts. In the context of semi-supervised learning, data typically consists of a modestly sized labeled dataset supplemented with vast amounts of unlabeled data. As a result, the empirical semi-supervised deep Sobolev regression risk aligns tightly with the following deep Sobolev regression problem:
\begin{equation*}
\min_{f\in\calF}\frac{1}{n}\sum_{i=1}^{n}(f(X_{i})-Y_{i})^{2}+\lambda\|\nabla f\|_{L^{2}(\Omega)}^{2},
\end{equation*}
plays a pivotal role in nonparametric regression and has been investigated by~\cite{Wahba1990Spline,Kohler2001Nonparametric,Kohler2002Application,gyorfi2002distribution}.  We establish non-asymptotic convergence rates for the deep Sobolev regressor and demonstrate that the norm of its gradient is uniformly bounded, shedding light on the considerable stability and favorable generalization properties of the estimator. Furthermore, under certain mild conditions, we derive non-asymptotic convergence rates for the plug-in derivative estimator based on SDORE. This illustrates how abundant unlabeled data used in SDORE~\eqref{eq:semigp} improves the performance of the standard gradient penalized regressor~\eqref{eq:gp}. We subsequently apply SDORE to nonparametric variable selection. The efficacy of this method is substantiated through numerous numerical examples.

\subsection{Contributions}
\par Our contributions can be summarized in four folds:
\begin{enumerate}[(i)]
\item We introduce a novel semi-supervised deep estimator within the framework of Sobolev penalized regression. A large amount of unlabeled data is employed to estimate the Sobolev penalty term. We demonstrate that this deep ReQU neural network-based estimator achieves the minimax optimal rate (Theorem~\ref{theorem:erm:rate:regression}). Meanwhile, with the appropriate selection of the regularization parameter, the norm of the estimator's gradient can be uniformly bounded, thereby illustrating its remarkable stability and generalization capacities from a theoretical standpoint.

\item Under certain mild conditions, we establish an oracle inequality for gradient  estimation using the plug-in deep Sobolev regressor (Lemma~\ref{lemma:oracle:erm}). Notably, this oracle inequality is applicable to any convex hypothesis class. This represents a significant theoretical advancement beyond existing nonparametric plug-in gradient estimators, which are based on linear approximation~\cite{Stone1985Additive,Liu2023Estimation}, by extending the framework to handle more complex hypothesis classes involved in nonlinear approximation~\cite{DeVore1993Constructive}. Furthermore, we derive a convergence rate for the gradient of the deep ReQU neural network-based estimator, providing valuable a priori guidance for selecting regularization parameters and choosing the size of the neural network (Theorem~\ref{theorem:erm:rate}).

\item We derive  a convergence rate for semi-supervised estimator (Theorem~\ref{theorem:erm:rate:data}), which sheds light on the quantifiable advantages of incorporating unlabeled data into the supervised learning. This improvement is actualized under the condition that density ratio between the marginal distribution of the labeled data and the distribution of the unlabeled data remains uniformly bounded. This novel finding promises to enrich our theoretical comprehension of semi-supervised learning, particularly in the context of deep neural networks.

\item The gradient estimator introduces a novel tool with potential applications in areas such as nonparametric variable selection. In the case where the regression function exhibits sparsity structure (Assumption~\ref{assumption:sparse}), we prove that the convergence rate depends only on the number of relevant variables, rather than the data dimension (Corollary~\ref{corollary:variable:selection}). Moreover, we establish the selection consistency of the deep Sobolev regressor (Corollary~\ref{corollary:selection:consistency}), showing that, with a sufficiently large number of labeled data pairs, the estimated relevant set is highly likely to match the ground truth relevant set. To validate our approach, we conduct a series of numerical experiments, which confirm the effectiveness and reliability of our proposed methodology.

\end{enumerate}

\subsection{Main Results Overview}

In this work, we focus on two estimators in the setting of nonparametric regression~\eqref{eq:nonparametric:regression}. The \underline{\bfseries D}eep S\underline{\bfseries O}bolev \underline{\bfseries RE}gressor (DORE) is derived from the regularized empirical risk minimization:
\begin{multline}\tag{DORE}\label{eq:dore}
\what{f}_{\euD}^{\lambda}\in\argmin_{f\in\calF}\what{L}_{\euD}^{\lambda}(f)=\frac{1}{n}\sum_{i=1}^{n}(f(X_{i})-Y_{i})^{2} \\
+\lambda\|\nabla f\|_{L^{2}(\nu_{X})}^{2},
\end{multline}
where $\euD=\{(X_{i},Y_{i})\}_{i=1}^{n}$ is a set of independent copies of $(X,Y)$, $\lambda>0$ is the regularization parameter, and $\calF$ is a class of deep ReQU neural networks. In some application scenarios, the regularization term in~\eqref{eq:dore} is intractable analytically. To address this issue, we approximate the regularization term by its data-driven counterpart, yielding the following semi-supervised empirical risk minimizer
\begin{multline}\tag{SDORE}
\what{f}_{\euD,\euS}^{\lambda}\in\argmin_{f\in\calF}\what{L}_{\euD,\euS}^{\lambda}(f)=\frac{1}{n}\sum_{i=1}^{n}(f(X_{i})-Y_{i})^{2} \\
+\frac{\lambda}{m}\sum_{i=1}^{m}\sum_{k=1}^{d}|D_{k}f(Z_{i})|^{2},
\end{multline}
where $\euS=\{Z_{i}\}_{i=1}^{m}$ is a set of independently and identically random variables drawn from $\nu_{X}$.

\par The main theoretical results derived in this paper are summarized in Table~\ref{tab:rates}. As shown in Theorem~\ref{theorem:erm:rate:regression}, the convergence rate of the deep Sobolev regressor in $L^{2}$-norm achieves the minimax optimality. However, Theorems~\ref{theorem:erm:rate:data} and~\ref{theorem:erm:rate} demonstrate that the convergence rates in $L^{2}$-norm and $H^{1}$-semi-norm is sub-optimal.

\begin{table*}[!t]
\centering
\scriptsize
\caption{Convergence Rates for Sobolev Penalized Estimators}
\label{tab:rates}
\resizebox*{0.95\linewidth}{!}{
\begin{tabular}{llllcl}
\toprule
Estimator & Reg. Param. & & Convergence Rates & Minimax Opt. &  \\
\midrule
\textbf{DORE} & \multirow{2}*{$\lambda=\calO(n^{-\frac{2s}{d+2s}}\log^{3}n)$} & $\bbE\|\hat{f}_{\euD}^{\lambda}-f_{0}\|^{2}$ & $\calO(n^{-\frac{2s}{d+2s}}\log^{3}n)$ & \ding{52} &\multirow{2}*{Theorem~\ref{theorem:erm:rate:regression}} \\
Def.~\eqref{eq:empirical:risk:regularization} & & $\bbE\|\nabla\hat{f}_{\euD}^{\lambda}\|^{2}$ & $\calO(1)$ &  & \\
\midrule
\textbf{DORE} & \multirow{2}*{$\lambda=\calO(n^{-\frac{s}{d+4s}}\log^{2}n)$} & $\bbE\|\hat{f}_{\euD}^{\lambda}-f_{0}\|^{2}$ & $\calO(n^{-\frac{2s}{d+4s}}\log^{4}n)$ & \ding{56} & \multirow{2}*{Theorem~\ref{theorem:erm:rate}} \\
Def.~\eqref{eq:empirical:risk:convex} & & $\bbE\|\nabla(\hat{f}_{\euD}^{\lambda}-f_{0})\|^{2}$ & $\calO(n^{-\frac{s}{d+4s}}\log^{2}n)$ & \ding{56} & \\
\midrule
\textbf{SDORE} & \multirow{2}*{$\lambda=\calO(n^{-\frac{s}{d+4s}}\log^{2}n)$} & $\bbE\|\hat{f}_{\euD,\euS}^{\lambda}-f_{0}\|^{2}$ & $\calO(n^{-\frac{2s}{d+4s}}\log^{4}n)+\calO(n^{\frac{d}{d+4s}}\log^{4}nm^{-1})$ & \ding{56} & \multirow{2}*{Theorem~\ref{theorem:erm:rate:data}} \\
Def.~\eqref{eq:empirical:risk:convex} &  &  $\bbE\|\nabla(\hat{f}_{\euD,\euS}^{\lambda}-f_{0})\|^{2}$ & $\calO(n^{-\frac{s}{d+4s}}\log^{2}n)+\calO(n^{\frac{d+s}{d+4s}}\log^{2}nm^{-1})$ & \ding{56} &  \\
\bottomrule
\end{tabular}
}
\end{table*}

\par We utilize the deep Sobolev regressor to tackle an application scenarios: nonparametric variable selection. We present the theoretical findings related to nonparametric variable selection in Table~\ref{tab:app}, including the convergence rate and selection consistency. 

\begin{table*}[!t]
\centering
\scriptsize
\caption{Theoretical Results for Applications}
\label{tab:app}
\resizebox{0.9\linewidth}{!}{
\begin{tabular}{llll}
\toprule
\multicolumn{4}{c}{\textbf{Nonparametric Variable Selection}} \\
\midrule
Reg. Param. &  & Convergence Rates &  \\
\midrule
\multirow{2}*{$\lambda=\calO(n^{-\frac{s}{d^{*}+4s}}\log^{2}n)$} & $\bbE\|\hat{f}_{\euD}^{\lambda}-f_{0}\|^{2}$ & $\calO(n^{-\frac{2s}{d^{*}+4s}}\log^{4}n)$ & \multirow{2}*{Corollary~\ref{corollary:variable:selection}} \\
& $\bbE\|\nabla(\hat{f}_{\euD}^{\lambda}-f_{0})\|^{2}$ & $\calO(n^{-\frac{s}{d^{*}+4s}}\log^{2}n)$ & \\
\midrule
\multirow{2}*{Selection consistency} & Rel. Set $\calI(f_{0})$ & \multirow{2}*{$\displaystyle \lim_{n\rightarrow\infty}\pr\{\calI(f_{0})=\calI(\hat{f}_{\euD}^{\lambda})\}=1$} & \multirow{2}*{Corollary~\ref{corollary:selection:consistency}} \\
& Esti. Rel. Set $\calI(\hat{f}_{\euD}^{\lambda})$ & & \\
\bottomrule
\end{tabular}
}
\end{table*}

\subsection{Preliminaries and notations}

\par Before proceeding, we introduce some notation and definitions. Let $\Omega\subseteq\bbR^{d}$ be a bounded domain, and let $\mu_{X}$ and $\nu_{X}$ be two probability measures on $\Omega$ with densities $p(x)$ and $q(x)$, respectively. The $L^{2}(\mu_{X})$ inner-product and norm are given, respectively, by
\begin{align*}
&(u,v)_{L^{2}(\mu_{X})}=\int_{\Omega}uvd\mu_{X}, \\
&\|u\|_{L^{2}(\mu_{X})}^{2}=(u,u)_{L^{2}(\mu_{X})}.
\end{align*}
Similarly, one can define the $L^{2}(\nu_{X})$ inner-product and norm. Furthermore, define the density ratio between $\nu_{X}$ and $\mu_{X}$ by $r(x)=q(x)/p(x)$. Suppose the density ratio is uniformly upper- and lower-bounded, that is, $\kappa:=\sup_{x\in\Omega}|r(x)|<\infty$ and $\zeta:=\inf_{x\in\Omega}|r(x)|>0$. Then it is straightforward to verify that 
\begin{equation*}
\zeta\|u\|_{L^{2}(\mu_{X})}^{2}\leq\|u\|_{L^{2}(\nu_{X})}^{2}\leq\kappa\|u\|_{L^{2}(\mu_{X})}^{2}.
\end{equation*}
For two functions $u,v\in H^{1}(\nu_{X})$, the inner products between their gradients is defined as
\begin{equation*}
(\nabla u,\nabla v)_{L^{2}(\nu_{X})}=\int_{\Omega}\sum_{k=1}^{d}D_{k}uD_{k}vd\nu_{X}.
\end{equation*}

\begin{definition}[Continuous functions space]\label{def:continuous}
Let $\Omega$ be a bounded domain in $\bbR^{d}$ and $s\in\bbN$. Let $C^{s}(\Omega)$ denote the vector space consisting of all functions $f$ which, together with all their partial derivatives $D^{\alpha}f$ of orders $\|\alpha\|_{1}\leq s$, are continuous on $\Omega$. The Banach space $C^{s}(\Omega)$ is equipped with the norm
\begin{equation*}
\|f\|_{C^{s}(\Omega)}:=\max_{\|\alpha\|_{1}\leq s}\sup_{x\in\Omega}|D^{\alpha}f(x)|,
\end{equation*}
where $D^{\alpha}=D_{1}^{\alpha_{1}}\cdots D_{d}^{\alpha_{d}}$ with $\alpha=(\alpha_{1},\ldots,\alpha_{d})^{T}\in\bbN^{d}$.
\end{definition}

Next, we introduce the concept of a deep neural network. While deep ReLU neural networks have shown empirical success in nonparametric regression tasks, they are not suitable for scenarios where derivatives of the network are required in the objective function~\cite{E2018deep}. This limitation arises from the piecewise linear nature of the ReLU activation function, which results in a lack of continuous derivatives. In contrast, the Rectified Quadratic Unit (ReQU) activation function, defined as the square of the ReLU function, possesses a continuous first derivative. This characteristic allows us to incorporate the deep ReQU neural network in the SDORE framework, thereby expanding the possibilities for the simultaneous estimation of regression values and their derivatives.

\begin{definition}[Deep ReQU neural network]\label{def:dnn}
A neural network $\psi:\bbR^{N_{0}}\rightarrow\bbR^{N_{L+1}}$ is a function defined by
\begin{equation}\label{eq:dnn}
\psi(x)=T_{L}(\varrho(T_{L-1}(\cdots\varrho(T_{0}(x))\cdots))),
\end{equation}
where the ReQU activation function $\varrho(x)=(\max\{x,0\})^{2}$ is applied component-wisely and $T_{\ell}(x):=A_{\ell}x+b_{\ell}$ is an affine transformation with $A_{\ell}\in\bbR^{N_{\ell+1}\times N_{\ell}}$ and $b_{\ell}\in\bbR^{N_{\ell}}$ for $\ell=0,\ldots,L$. In this paper, we consider the case $N_{0}=d$ and $N_{L+1}=1$. The number $L$ is called the depth of the neural network, and the number $\max_{1\leq\ell\leq L}N_{\ell}$ is called the width of the neural network.Additionally, $\sum_{\ell=0}^{L}(\|A_{\ell}\|_{0}+\|b_{\ell}\|_{0})$ represents the total number of non-zero weights within the neural network. The space of deep ReQU neural networks with given network architecture is defined as 
\begin{multline*}
\calN(L,W,S):=\Big\{\psi~\text{is of the form~\eqref{eq:dnn}}: \\
\max_{1\leq\ell\leq L}N_{\ell}\leq W,~\sum_{\ell=0}^{L}(\|A_{\ell}\|_{0}+\|b_{\ell}\|_{0})\leq S\Big\}.
\end{multline*}
\end{definition}

To measure the complexity of a function class, we next introduce the empirical covering number.
\begin{definition}[Empirical covering number]
Let $\calF$ be a class of functions from $\Omega$ to $\bbR$ and $\euD=\{X_{i}\}_{i=1}^{n}\subseteq\Omega$. Define the $L^{p}(\euD)$-norm of the function $f\in\calF$ as
\begin{equation*}
\|f\|_{L^{p}(\euD)}=\Big(\frac{1}{n}\sum_{i=1}^{n}|f(X_{i})|^{p}\Big)^{1/p}, \quad 1\leq p<\infty.
\end{equation*}
For $p=\infty$, define $\|f\|_{L^{\infty}(\euD)}=\max_{1\leq i\leq n}|f(X_{i})|$. A function set $\calF_{\delta}$ is called an $L^{p}(\euD)$ $\delta$-cover of $\calF$ if for each $f\in\calF$, there exits $f_{\delta}\in\calF_{\delta}$ such that $\|f-f_{\delta}\|_{L^{p}(\euD)}\leq\delta$. Furthermore,
\begin{equation*}
N(\delta,\calF,L^{p}(\euD))=\inf\Big\{|\calF_{\delta}|:\calF_{\delta}\text{ is a $L^{p}(\euD)$ $\delta$-cover of $\calF$}\Big\}
\end{equation*}
is called the $L^{p}(\euD)$ $\delta$-covering number of $\calF$.
\end{definition}

\par We now introduce some basic notations. The set of positive integers is denoted by $\bbN_{+}=\{1,2,\ldots\}$. Denote $\bbN=\{0\}\cup\bbN_{+}$ for convenience. For a positive integer $m\in\bbN_{+}$, let $[m]$ denote the set $\{1,\ldots,m\}$. We employ the notations $A\lesssim B$ and $B\gtrsim A$ to signify that there exists an absolute constant $c>0$ such that $A\leq cB$.

\subsection{Organization}
The remainder of the article is organized as follows. We commence with a review of related work in Section~\ref{section:related}. Subsequently, we outline the deep Sobolev penalized regression and propose the semi-supervised estimator in Section~\ref{section:regularized:estimator}. We present the convergence rate analysis for the regression in Section~\ref{section:analysis:regression} and for the derivative estimation in Section~\ref{section:analysis:derivative}. In Section~\ref{section:applications}, we apply our method to nonparametric variable selection, and provide an abundance of numerical studies. The article concludes with a few summarizing remarks in Section~\ref{section:conclusion}. All technical proofs are relegated to the supplementary material.

\section{Related Work}\label{section:related}

\par In this section, we review the topics and literature related to this work, including derivative estimation, regression using deep neural network, nonparametric variable selection and semi-supervised learning.

\subsection{Nonparametric Derivative Estimation}

\par As previously indicated, the necessity to estimate derivatives arises in various application contexts. Among the simplest and most forthright methods for derivative estimation is the direct measurement of derivatives. For example, in the field of economics, estimating cost functions~\cite{Shephard1981Cost} frequently involves data on a function and its corresponding set of derivatives. A substantial volume of literature~\cite{Florens1996Sobolev,Hall2007nonparametric,Hall2010Nonparametric} considers this scenario by reverting to a corresponding regression model:
\begin{equation*}
Y^{\alpha}=D^{\alpha}f_{0}(X)+\xi^{\alpha},
\end{equation*}
where $\alpha\in\bbN^{d}$ is a multi-index, $D^{\alpha}$ is the $\alpha$-th derivative operator, and $\xi^{\alpha}$ are random noise. The theoretical framework underpinning this method can be seamlessly generalized from that of classical nonparametric regression. However, it may be worth noting that in some practical application settings, measurements of derivatives are often not readily available.

\par To estimate derivatives with noisy measurements only on function values, researchers have put forward nonparametric derivative estimators~\cite{Newell2007comparative}. Nonparametric derivative estimation encompasses four primary approaches: local polynomial regression~\cite{Brabanter2013Derivative}, smoothing splines~\cite{Nancy2000Penalized}, kernel ridge regression~\cite{Liu2023Estimation}, and difference quotients~\cite{Muller1987Bandwidth,Liu2018Derivative,Liu2020Smoothed}. Among these approaches, the first three are categorized as plug-in derivative estimators. In this article, we present a review of these plug-in approaches using the one-dimensional case as an illustrative example.

\subsubsection{Local Polynomial Regression}

\par In standard polynomial regression, a single polynomial function is used to fit the data. One of the main challenges with this method is the need to use high-order polynomials to achieve a more accurate approximation. However, high-order polynomials may be oscillative in some regions, which is known as Runge phenomenon~\cite{Epperson1987Runge}. To repair the drawbacks of the polynomial regression, a natural way is to employ the low-degree polynomial regression locally, which is called local polynomial regression~\cite{Fan1996local}. Derivative estimation using local polynomial regression was first proposed by~\cite{Brabanter2013Derivative}. Let $K$ be a kernel function and $h$ be the bandwidth controlling the smoothness. We assign a weight $K((X_{i}-x)/h)$ to the point $(X_{i},Y_{i})$, leading to the following weighted least-squaress problem:
\begin{equation}\label{eq:local:poly}
\min_{\{\beta_{\ell}(x)\}_{\ell=0}^{p}}\sum_{i=1}^{n}K\Big(\frac{X_{i}-x}{h}\Big)\Big(Y_{i}-\sum_{\ell=0}^{p}\beta_{\ell}(x)(X_{i}-x)^{\ell}\Big)^{2}.
\end{equation}
Herr the kernel $K$ should decay fast enough to eliminate the impact of a remote data point. Denote by $\{\what{\beta}_{\ell}\}_{\ell=0}^{p}$ the estimator obtained by~\eqref{eq:local:poly}. The estimated regression curve at point $x$ is given by $\what{f}(x)=\sum_{\ell=0}^{p}\what{\beta}_{\ell}(x)(X_{i}-x)^{\ell}$. Further, according to Taylor's theorem, the estimator of the first order derivative $f_{0}^{\prime}$ at point $x$ is given by $\what{f}^{\prime}(x)=\what{\beta}_{1}(x)$.~\cite{Masry1996Multivariate1,Masry1996Multivariate2} established the uniform strong consistency and the convergence rates for the regression  function and its partial derivatives. Derivative estimation using local polynomial regression in multivariate data has been discussed in~\cite{Amiri2018Regression}.

\subsubsection{Smoothing Splines}

\par Extensive research has been conducted on the use of smoothing splines in nonparametric regression~\cite{Wahba1990Spline,Green1993Nonparametric,Eubank1999Nonparametric,Wasserman2006all}. This method starts from the minimization of a penalized least-squaress risk
\begin{equation}\label{eq:spline}
\min_{f\in H^{2}([0,1])}\frac{1}{n}\sum_{i=1}^{n}(f(X_{i})-Y_{i})^{2}+\gamma\int_{0}^{1}(f^{\prime\prime}(x))^{2}dx,
\end{equation}
where the first term encourages the fitting of estimator to data, the second term penalizes the roughness of the estimator, and the smoothing parameter $\gamma>0$ controls the trade-off between the two conflicting goals. The minimizer $\what{f}$ of~\eqref{eq:spline} is an estimator of the regression function $f_{0}$, which is called cubic smoothing spline. The plug-in derivative estimator $\what{f}^{\prime}$ is a direct estimate of the derivative $f_{0}^{\prime}$ of the regression function. This idea has been pursued by~\cite{Stone1985Additive,Nancy2000Penalized}. In the perspective of theoretical analysis,~\cite{Stone1985Additive} shows that spline derivative estimators can achieve the optimal rate of convergence, and~\cite{Zhou2000Derivative} studies local asymptotic properties of derivative estimators.

\subsubsection{Kernel Ridge Regression}

\par Kernel ridge regression is a technique extensively employed in the domain of nonparametric regression.~\cite{Liu2023Estimation} introduced a plug-in kernel ridge regression estimator for derivatives of the regression function, establishing a nearly minimax convergence rate for univariate function classes within a random-design setting. Further expanding upon this method,~\cite{Liu2023EstimationHypothesis} applied it to multivariate regressions under the smoothing spline ANOVA model and established minimax optimal rates. Additionally,~\cite{Liu2023EstimationHypothesis} put forth a hypothesis testing procedure intended to determine whether a derivative is zero.

\subsection{Nonparametric Regression using Deep Neural Network}

\par In comparison to the nonparametric methods mentioned above, deep neural networks~\cite{Goodfellow2016Deep} also stand out as a formidable technique employed within machine learning and nonparametric statistics. Rigorous of the convergence rate analysis have been established for deep nonparametric regression~\cite{bauer2019deep,nakada2020Adaptive,schmidt2020nonparametric,kohler2021rate,Farrell2021Deep,Kohler2022Estimation,Jiao2023deep}, but derivative estimation using deep neural networks remained an open problem prior to this paper, even though the derivative of the regression estimate is of great importance as well.

\par Unfortunately, estimating derivatives is not always a by-product of function estimation. Indeed, the basic mathematical analysis~\cite[Section 3.7]{Tao2022analysis} shows that, even if estimators $\{f_{n}\}_{n\geq1}$ converge to the regression function $f_{0}$, the convergence of plug-in derivative estimators $\{\nabla f_{n}\}_{n\geq1}$ is typically not guaranteed. To give a counterexample, we consider the functions $f_{n}:I\rightarrow\bbR,x\mapsto n^{-1}\sin(nx)$, and let $f_{0}:I\rightarrow\bbR$ be the zero function $f_{0}(x)=0$, where $I:=[0,2\pi]$. Then $\|f_{n}-f_{0}\|_{L^{p}(I)}\rightarrow0$ as $n\rightarrow0$, but $\lim_{n\rightarrow\infty}\|f_{n}^{\prime}-f_{0}^{\prime}\|_{L^{p}(I)}\neq0$ for each $1\leq p \leq\infty$.

\par Roughly speaking, the success of classical approaches for derivative estimation can be attributed to their smoothing techniques, such as the kernel function incorporated in local polynomial regression, or the regularization in smoothing spline and kernel ridge regression. Thus, to guarantee the convergence of the plug-in derivative estimator, the incorporation of a Sobolev regularization term is imperative within the loss function, akin to the methodology applied in smoothing spline.

\subsection{Nonparametric Vairable Selection}

\par Data collected in real-world applications tend to be high-dimensional, although only a subset of the variables within the covariate vector may genuinely exert influence. Consequently, variable selection becomes critical in statistics and machine learning as it both mitigates computational complexity and enhances the interpretability of the model. However, traditional methods for variable selection have been primarily focused on linear or additive models and do not readily extend to nonlinear problems. One inclusive measure of the importance of each variable in a nonlinear model is its corresponding partial derivatives. Building on this concept, a series of works~\cite{Rosasco2010regularization,Mosci2012Is,Rosasco2013Nonparametric} introduced sparse regularization to kernel ridge regression for variable selection. They have consequently devised a feasible computational learning scheme and developed consistency properties of the estimator. However, the theoretical analysis is limited to reproducing kernel Hilbert space and cannot be generalized to deep neural network-based methods. 

\subsection{Semi-Supervised Learning}

\par Semi-supervised learning has recently gained significant attention in statistics and machine learning~\cite{Zhu2009Introduction,Van2020survey}. The basic setting of semi-supervised learning is common in many practical applications where the label is often more difficult or costly to collect than the covariate vector. Therefore, the fundamental question is how to design appropriate learning algorithms to fully exploit the value of unlabeled data. In the past years, significant effort has been devoted to studying the algorithms and theory of semi-supervised learning~\cite{Zhang2000Value,Belkin2006manifold,Wasserman2007Statistical,Azriel2022Semi,Livne2022Improved,Song2023General,Deng2023Optimal}. The most related work is~\cite{Belkin2006manifold}, whose main idea is to introduce an unlabeled-data-driven regularization term to the loss function. Specifically,~\cite{Belkin2006manifold} employ a manifold regularization to incorporate additional information about the geometric structure of the marginal distribution, where the regularization term is estimated on the basis of unlabeled data. In addition, our method does not require the distribution of the unlabeled data to be aligned with the marginal distribution of the labeled data exactly, which expands the applicability scenarios.

\section{Deep Sobolev Regression}\label{section:regularized:estimator}

\par In this section, we present an in-depth examination of Sobolev penalized least-squares regression as implemented through deep neural networks. Initially, we incorporate the $H^{1}$-semi-norm penalty into the least-squares risk. Subsequently, we delineate the deep Sobolev regressor as referenced in Section~\ref{sec:reg:erm}, followed by an introduction to the semi-supervised Sobolev regressor elaborated in Section~\ref{sec:reg:erm:semisup}.

\par We focus on the following $H^{1}(\nu_{X})$-semi-norm penalized least-squares risk:
\begin{equation}\label{eq:population:risk:regularization}
\min_{f\in\calA}L^{\lambda}(f)=\bbE_{(X,Y)\sim\mu}\big[(f(X)-Y)^{2}\big]+\lambda\|\nabla f\|_{L^{2}(\nu_{X})}^{2},
\end{equation}
where $\mu$ is a probability measure on $\Omega\times\bbR$ associated to the regression model~\eqref{eq:nonparametric:regression}, and $\nu_{X}$ is another probability measure on $\Omega$. The admissible set $\calA$ defined as
\begin{equation*}
\calA=\Big\{f\in L^{2}(\mu_{X}):D_{k}f\in L^{2}(\nu_{X}),~1\leq k\leq d\Big\}.
\end{equation*}
Here the regularization parameter $\lambda>0$ governs the delicate equilibrium between conflicting objectives: data fitting and smoothness. Specifically, when $\lambda$ nearly or entirely vanishes,~\eqref{eq:population:risk:regularization} aligns with the standard population least-squares risk. Conversely, as $\lambda$ approaches infinity, the minimizer of~\eqref{eq:population:risk:regularization} tends towards a constant estimator. For the joint distribution $\mu$ of $(X,Y)$, let $\mu_{X}$ denote the margin distribution of $X$. According to~\eqref{eq:nonparametric:regression}, one obtains easily
\begin{equation}\label{eq:population:risk:regularization:1}
L^{\lambda}(f)=\|f-f_{0}\|_{L^{2}(\mu_{X})}^{2}+\lambda\|\nabla f\|_{L^{2}(\nu_{X})}^{2}+\bbE[\xi^{2}],
\end{equation}
where the $L^{2}(\mu_{X})$-risk may be respect to a different measure $\mu_{X}$ than that $\nu_{X}$ associated with Sobolev penalty. Throughout this paper, we assume that the distributions $\mu_{X}$ and $\nu_{X}$ have density function $p$ and $q$, respectively. Furthermore, the density ratio $r(x):=q(x)/p(x)$ satisfies the following condition, which may encourage significant domain shift.

\begin{assumption}[Uniformly bounded density ratio]\label{assumption:bounded:density:ratio}
The density ratio between $\nu_{X}$ and $\mu_{X}$ has a uniform upper-bound and a positive lower-bound, that is,
\begin{equation*}
\kappa:=\sup_{x\in\Omega}|r(x)|<\infty \quad\text{and}\quad
\zeta:=\inf_{x\in\Omega}|r(x)|>0.
\end{equation*}
\end{assumption}

\par Sobolev penalized regression can be interpreted as a PDE-based smoother of the regression function $f_{0}$. Let $f^{\lambda}$ denote a solution to the quadratic optimization problem~\eqref{eq:population:risk:regularization}. Some standard calculus of variations~\cite{Brenner2008Mathematical,evans2010partial} show that, if the minimizer $f^{\lambda}$ has square integrable second derivatives, then $f^{\lambda}$ solves the following second-order linear elliptic equation with homogeneous Neumann boundary condition:
\begin{equation*}
\left\{
\begin{aligned}
-\lambda\Delta f^{\lambda}+f^{\lambda}&=f_{0}, &&\text{in}~\Omega, \\
\nabla f^{\lambda}\cdot\vn&=0, &&\text{on}~\partial\Omega.
\end{aligned}
\right.
\end{equation*}
In the context of partial differential equations (PDE), the variational problem~\eqref{eq:population:risk:regularization} is called Ritz method~\cite[Remark 2.5.11]{Brenner2008Mathematical}. The following lemma shows the uniqueness of solution to the above PDE.

\begin{lemma}[Existence and uniqueness of population risk minimizer]\label{lemma:existence:prm}
Suppose Assumption~\ref{assumption:bounded:density:ratio} holds and $f_{0}\in L^{2}(\mu_{X})$. Then~\eqref{eq:population:risk:regularization} has a unique minimizer in $H^{1}(\nu_{X})$. Furthermore, the minimizer $f^{\lambda}$ satisfies $f^{\lambda}\in H^{2}(\nu_{X})$.
\end{lemma}

\par In practical applications, the data distribution $\mu$ in~\eqref{eq:population:risk:regularization} remains unknown, making the minimization of population risk~\eqref{eq:population:risk:regularization} unattainable. The goal of regression is to estimate the function $f_{0}$ from a finite set of data pairs $\euD=\{(X_{i},Y_{i})\}_{i=1}^{n}$ which are independently and identically drawn from $\mu$, that is,
\begin{equation*}
Y_{i}=f_{0}(X_{i})+\xi_{i}, \quad i=1,\ldots,n.
\end{equation*}
We introduce two Sobolev regressor based on the random sample $\euD$ in the following two subsections, respectively.

\subsection{Deep Sobolev regressor}\label{sec:reg:erm}

\par Suppose that the probability measure $\nu_{X}$ is either provided or selected by the user. Then the regularization term can be estimated with an arbitrarily small error. Hence, without loss of generality, this error is omitted in this discussion. In this setting, the deep Sobolev regressor is derived from the regularized empirical risk minimization:
\begin{multline}\label{eq:empirical:risk:regularization}
\what{f}_{\euD}^{\lambda}\in\argmin_{f\in\calF}\what{L}_{\euD}^{\lambda}(f)=\frac{1}{n}\sum_{i=1}^{n}(f(X_{i})-Y_{i})^{2} \\
+\lambda\|\nabla f\|_{L^{2}(\nu_{X})}^{2},
\end{multline}
where $\calF\subseteq\calA$ is a class of deep neural networks.

\par The objective functional in~\eqref{eq:empirical:risk:regularization} has been investigated previously within the literature of splines, according to research by~\cite{Wahba1990Spline,Kohler2001Nonparametric,Kohler2002Application,gyorfi2002distribution}. However, in these studies, minimization was undertaken within the Sobolev space $H^{1}(\Omega)$ or the continuous function space $C^{1}(\Omega)$ as opposed to within a class of deep neural networks.

\subsection{Semi-Supervised Deep Sobolev regressor}\label{sec:reg:erm:semisup}

\par In numerous application scenarios, the probability measure $\nu_{X}$ remains unknown and cannot be provided by the user. Nevertheless, a substantial quantity of samples drawn from $\nu_{X}$ can be obtained at a very low cost. This is a semi-supervised setting that provides access to labeled data and a relatively large amount of unlabeled data.

\par Let $\euS=\{Z_{i}\}_{i=1}^{m}$ be a random sample with $\{Z_{i}\}_{i=1}^{m}$ independently and identically drawn from $\nu_{X}$. Then replacing the population regularization term in~\eqref{eq:empirical:risk:regularization} by its data-driven counterpart, we obtain the following semi-supervised empirical risk minimizer
\begin{multline}\label{eq:empirical:risk:regularization:data}
\what{f}_{\euD,\euS}^{\lambda}\in\argmin_{f\in\calF}\what{L}_{\euD,\euS}^{\lambda}(f)=\frac{1}{n}\sum_{i=1}^{n}(f(X_{i})-Y_{i})^{2} \\
+\frac{\lambda}{m}\sum_{i=1}^{m}\sum_{k=1}^{d}|D_{k}f(Z_{i})|^{2},
\end{multline}
where the deep neural network class $\calF$ satisfies $\calF\subseteq W^{1,\infty}(\Omega)$. A similar idea was mentioned by~\cite{Belkin2006manifold} in the context of manifold learning.

\par The estimator presented in~\eqref{eq:empirical:risk:regularization:data}, which incorporates unlabeled data into a supervised learning framework, is commonly referred to as a semi-supervised estimator. The availability of labeled data is often limited due to its high cost, but in many cases, there is an abundance of unlabeled data that remains underutilized. Given that there are no strict constraints on the measure $\nu_{X}$ in our method, it is possible to generate a substantial amount of unsupervised data from supervised data through data augmentation, even without a large quantity of unlabeled data. Hence, this semi-supervised learning framework exhibits a broad range of applicability across various scenarios.

\par It is worth highlighting that when the measure $\nu_{X}$ is equal to $\mu_{X}$, the formulation~\eqref{eq:empirical:risk:regularization:data} is reduced to
\begin{multline}\label{eq:empirical:risk:regularization:data:unlabeled}
\what{f}_{\euD,\euS}^{\lambda}\in\argmin_{f\in\calF}\what{L}_{\euD,\euS}^{\lambda}(f)=\frac{1}{n}\sum_{i=1}^{n}(f(X_{i})-Y_{i})^{2} \\
+\frac{\lambda}{n+m}\sum_{i=1}^{n+m}\sum_{k=1}^{d}|D_{k}f(X_{i})|^{2},
\end{multline}
where $X_{n+i}=Z_{i}$ for $1\leq i\leq m$. The semi-supervised Sobolev regressor, deployed in~\eqref{eq:empirical:risk:regularization:data} or~\eqref{eq:empirical:risk:regularization:data:unlabeled}, imparts meaningful insights on how to leverage unlabeled data to enhance the efficacy of original supervised learning approach.

\section{Deep Sobolev Regressor with Gradient-Norm Constraint}\label{section:analysis:regression}

\par In this section, we provide a theoretical analysis for the deep Sobolev regressor~\eqref{eq:empirical:risk:regularization}. The first result, given in Lemma~\ref{lemma:oracle:erm:regression}, is an oracle-type inequality, which provides an upper-bound for the $L^{2}(\mu_{X})$-error of the deep Sobolev regressor along with an upper-bound for the $L^{2}(\nu_{X})$-norm of its gradient. Further, we show that~\eqref{eq:empirical:risk:regularization} attains the minimax optimal convergence rate, given that the regularization parameter are chosen appropriately. We also confirm that the gradient norm of the deep Sobolev regressor can be uniformly bounded by a constant.

\begin{assumption}[Sub-Gaussian noise]\label{assumption:subGaussian}
The noise $\xi$ in~\eqref{eq:nonparametric:regression} is sub-Gaussian with mean 0 and finite variance proxy $\sigma^{2}$ conditioning on $X=x$ for each $x\in\Omega$, that is, its conditional moment generating function satisfies
\begin{equation*}
\mathbb{E}[\exp(t\xi)|X=x]\leq\exp\Big(\frac{\sigma^{2}t^{2}}{2}\Big),\quad\forall~t\in\mathbb{R},~x\in\Omega.
\end{equation*}
\end{assumption}

\begin{assumption}[Bounded hypothesis]\label{assumption:boundedness:f}
There exists an absolute positive constant $B_{0}$, such that $\sup_{x\in\Omega}|f_{0}(x)|\leq B_{0}$. Further, functions in hypothesis class $\calF$ are also bounded, that is, $\sup_{x\in\Omega}|f(x)|\leq B_{0}$.
\end{assumption}

\par Assumptions~\ref{assumption:subGaussian} and~\ref{assumption:boundedness:f} are standard and very mild conditions in nonparametric regression, as extensively discussed in the literature~\cite{gyorfi2002distribution,Tsybakov2009Introduction,nakada2020Adaptive,schmidt2020nonparametric,Farrell2021Deep,Jiao2023deep}. It is worth noting that the upper-bound $B_{0}$ of hypothesis may be arbitrarily large and does not vary with the sample size $n$. In fact, this assumption can be removed through the technique of truncation, without affecting the subsequent proof, which can be found in~\cite{bauer2019deep,kohler2021rate,Kohler2022Estimation} for details.

\par The convergence rate relies on an oracle-type inequality as follows.

\begin{lemma}[Oracle inequality]\label{lemma:oracle:erm:regression}
Suppose Assumptions~\ref{assumption:bounded:density:ratio} to~\ref{assumption:boundedness:f} hold. Let $\what{f}_{\euD}^{\lambda}$ be the deep Sobolev regressor defined as~\eqref{eq:empirical:risk:regularization} with regularization parameter $\lambda>0$. Then it follows that for each $n\geq\log N(B_{0}\delta,\calF,L^{2}(\euD))$,
\begin{align*}
&\bbE_{\euD\sim\mu^{n}}\Big[\|\what{f}_{\euD}^{\lambda}-f_{0}\|_{L^{2}(\mu_{X})}^{2}\Big] \\
&\lesssim\inf_{f\in\calF}\Big\{\|f-f_{0}\|_{L^{2}(\mu_{X})}^{2}+\lambda\|\nabla f\|_{L^{2}(\nu_{X})}^{2}\Big\} \\
&\quad+(B_{0}^{2}+\sigma^{2})\inf_{\delta>0}\Big\{\frac{\log N(B_{0}\delta,\calF,L^{2}(\euD))}{n}+\delta\Big\}, \\
&\bbE_{\euD\sim\mu^{n}}\Big[\|\nabla\what{f}_{\euD}^{\lambda}\|_{L^{2}(\nu_{X})}^{2}\Big] \\
&\lesssim\inf_{f\in\calF}\Big\{\frac{1}{\lambda}\|f-f_{0}\|_{L^{2}(\mu_{X})}^{2}+\|\nabla f\|_{L^{2}(\nu_{X})}^{2}\Big\} \\
&\quad+\frac{B_{0}^{2}+\sigma^{2}}{\lambda}\inf_{\delta>0}\Big\{\frac{\log N(B_{0}\delta,\calF,L^{2}(\euD))}{n}+\delta\Big\}.
\end{align*}
\end{lemma}

\par Roughly speaking, the first inequality of Lemma~\ref{lemma:oracle:erm:regression} decomposes the $L^{2}(\mu_{X})$-error of the deep Sobolev regressor into three terms, namely: the approximation error, the regularization term, and the generalization error. Intriguingly, from the perspective of the first two terms, we need to find a deep neural network in $\calF$ that not only has an sufficiently small $L^{2}(\mu_{X})$-distance from the regression function $f_{0}$, but also has an $H^{1}(\nu_{X})$-semi-norm as small as possible.

\par The literature on deep learning theory has extensively investigated the approximation properties of deep neural networks~\cite{yarotsky2018optimal,yarotsky2020phase,shen2019nonlinear,shen2020deep,lu2021deep,petersen2018optimal,Jiao2023deep,Li2019Better,Li2020PowerNet,Duan2022Convergence,shen2023differentiable}. However, there is limited research on the approximation error analysis for neural networks with gradient norm constraints~\cite{Huang2022Error,Jiao2023Approximation}. The following lemma illustrates the approximation power of deep ReQU neural networks with gradient norm constraints.

\begin{lemma}[Approximation with gradient constraints]\label{lemma:approximation:grad:bounded}
Let $\Omega\subseteq K\subseteq\bbR^{d}$ be two bounded domain. Set the hypothesis class as a deep ReQU neural network class $\calF=\calN(L,W,S)$ with $L=\calO(\log N)$ and $S=\calO(N^{d})$. Then for each $\phi\in C^{s}(K)$ with $s\in\bbN_{\geq1}$, there exists a neural network $f\in\calF$ such that
\begin{align*}
\|f-\phi\|_{L^{2}(\mu_{X})}
&\leq CN^{-s}\|\phi\|_{C^{s}(K)}, \\
\|\nabla f\|_{L^{2}(\nu_{X})}&\leq \|\nabla\phi\|_{L^{2}(\nu_{X})}+C\|\phi\|_{C^{s}(K)},
\end{align*}
where $C$ is a constant independent of $N$.
\end{lemma}

This lemma provides a novel approximation error bound of deep ReQU networks with gradient norm constraint. This highlights a fundamental difference between deep ReLU and ReQU neural networks. As presented by~\cite{Huang2022Error,Jiao2023Approximation}, the gradient norm of deep ReLU networks goes to infinity when the approximation error diminishes. In contrast, Lemma~\ref{lemma:approximation:grad:bounded} demonstrates that deep ReQU neural networks, under a gradient norm constraint, can approximate the target function with an arbitrarily small error.

\par With the aid of the preceding lemmas, we can now establish the following convergence rates for the regularized estimator.

\begin{theorem}[Convergence rates]\label{theorem:erm:rate:regression}
Suppose Assumptions~\ref{assumption:bounded:density:ratio} to~\ref{assumption:boundedness:f} hold. Let $\Omega\subseteq K\subseteq\bbR^{d}$ be two bounded domain. Assume that $f_{0}\in C^{s}(K)$ with $s\in\bbN_{\geq1}$. Set the hypothesis class as a deep ReQU neural network class $\calF=\calN(L,W,S)$ with $L=\calO(\log n)$ and $S=\calO(n^{\frac{d}{d+2s}})$. Let $\what{f}_{\euD}^{\lambda}$ be the deep Sobolev regressor defined as~\eqref{eq:empirical:risk:regularization} for each $\lambda>0$. Then it follows that
\begin{align*}
&\bbE_{\euD\sim\mu^{n}}\Big[\|\what{f}_{\euD}^{\lambda}-f_{0}\|_{L^{2}(\mu_{X})}^{2}\Big]\leq\calO(\lambda)+\calO\Big(n^{-\frac{2s}{d+2s}}\log^{3}n\Big), \\
&\bbE_{\euD\sim\mu^{n}}\Big[\|\nabla\what{f}_{\euD}^{\lambda}\|_{L^{2}(\nu_{X})}^{2}\Big]\leq\calO(1)+\calO\Big(\lambda^{-1}n^{-\frac{2s}{d+2s}}\log^{3}n\Big).
\end{align*}
Further, setting $\lambda=\calO(n^{-\frac{2s}{d+2s}}\log^{3}n)$ implies
\begin{align*}
&\bbE_{\euD\sim\mu^{n}}\Big[\|\what{f}_{\euD}^{\lambda}-f_{0}\|_{L^{2}(\mu_{X})}^{2}\Big]\leq\calO\Big(n^{-\frac{2s}{d+2s}}\log^{2}n\Big), \\
&\bbE_{\euD\sim\mu^{n}}\Big[\|\nabla\what{f}_{\euD}^{\lambda}\|_{L^{2}(\nu_{X})}^{2}\Big]\leq\calO(1).
\end{align*}
Here the constant behind the big $\calO$ notation is independent of $n$.
\end{theorem}

\par Theorem~\ref{theorem:erm:rate:regression} quantifies how the regularization parameter $\lambda$ balances two completing goals: data fitting and the gradient norm of the estimator, and thus provides an a priori guidance for the selection of the regularization term. When one chooses $\lambda=\calO(n^{-\frac{2s}{d+2s}}\log^{3}n)$, the rate of the deep Sobolev regressor $\mathcal{O}(n^{-\frac{2s}{d+2s}}\log^{3}n)$ aligns with the minimax optimal rate up to a log-factor, as established in~\cite{Stone1982optimal,Yang1999Information,gyorfi2002distribution,Tsybakov2009Introduction}. Additionally, our theoretical findings correspond to those in nonparametric regression using deep neural networks~\cite{bauer2019deep,nakada2020Adaptive,schmidt2020nonparametric,kohler2021rate,Farrell2021Deep,Kohler2022Estimation,Jiao2023deep}. In contrast to standard empirical risk minimizers, the deep Sobolev regressor imposes a constraint on the gradient norm while simultaneously ensuring the minimax optimal convergence rate. Consequently, Sobolev regularization improves the stability and enhances the generalization abilities of deep neural networks.

\par A similar problem has been explored by researchers within the context of splines~\cite{Kohler2001Nonparametric,Kohler2002Application,gyorfi2002distribution}, where the objective functional aligns with that of the deep Sobolev regressor~\eqref{eq:empirical:risk:regularization}. However, in these studies, minimization was token over the Sobolev space $H^{1}(\Omega)$ or the continuous function space $C^{1}(\Omega)$ instead of a deep neural network class. The consistency in this setting was studied by~\cite{Kohler2001Nonparametric}, and the convergence rate was proven to be minimax optimal by~\cite{Kohler2002Application} or~\cite[Theorem 21.2]{gyorfi2002distribution}. It is worth noting that the rate analysis in these studies relies heavily on the theoretical properties of the spline space and cannot be generalized to our setting.

\section{Simultaneous Estimation of Regression Function and its Derivative}\label{section:analysis:derivative}

\par In this section, we demonstrate that under certain mild conditions, deep Sobolev regressors converge to the regression function in both the $L^{2}(\mu_{X})$-norm and the $H^{1}(\nu_{X})$-semi-norm. We establish rigorous convergence rates for both the deep Sobolev regressor and its semi-supervised counterpart. Additionally, we provide a priori guidance for selecting the regularization parameter and determining the appropriate size of neural networks.

\par To begin with, we define the convex-hull of the neural network class $\calF$, denoted as  $\conv(\calF)$. Subsequently, we proceed to redefine both the deep Sobolev regressor~\eqref{eq:empirical:risk:regularization} and its semi-supervised counterpart~\eqref{eq:empirical:risk:regularization:data} as
\begin{equation}\label{eq:empirical:risk:convex}
\what{f}_{\euD}^{\lambda}\in\argmin_{f\in\conv(\calF)}\what{L}_{\euD}^{\lambda}(f), \quad
\what{f}_{\euD,\euS}^{\lambda}\in\argmin_{f\in\conv(\calF)}\what{L}_{\euD,\euS}^{\lambda}(f).
\end{equation}
Notice that the functions within the convex-hull $\conv(\calF)$ are also deep neural networks, which can be implemented by the parallelization of neural networks~\cite{Hashem1997Optimal,Guhring2020Error}. Therefore, in the algorithmic implementation, solving~\eqref{eq:empirical:risk:convex} will only result in mere changes compared to solving in the original problem~\eqref{eq:empirical:risk:regularization} or~\eqref{eq:empirical:risk:regularization:data}.

\par Throughout this section, suppose the following assumptions are fulfilled.

\begin{assumption}[Regularity of regression function]\label{assumption:regularity:f0}
The regression function in~\eqref{eq:nonparametric:regression} satisfies $\Delta f_{0}\in L^{2}(\nu_{X})$ and $\nabla f_{0}\cdot\vn=0$ a.e. on $\partial\Omega$, where $\vn$ is the unit normal to the boundary.
\end{assumption}

\par Since there are no measurements available on the boundary $\partial\Omega$ or out of the domain $\Omega$, it is not possible to estimate the derivatives on the boundary accurately. Hence, to simplify the problem without loss of generality, we assume that the underlying regression $f_{0}$ has zero normal derivative on the boundary, as stated in Assumption~\ref{assumption:regularity:f0}. This assumption corresponds to the homogeneous Neumann boundary condition in the context of partial differential equations~\cite{evans2010partial}.

\par We also make the following assumption regarding the regularity of the density function.

\begin{assumption}[Bounded score function]\label{assumption:bounded:score}
The score function of the probability measure $\nu_{X}$ is bounded in $L^{2}(\nu_{X})$-norm, that is, $\|\nabla(\log q)\|_{L^{2}(\nu_{X})}<\infty$.
\end{assumption}

A sufficient condition for Assumption~\ref{assumption:bounded:score} is that $\nabla q$ is uniformly upper bounded and $q$ has a uniform positive lower bound. In fact, this stronger assumption is mild and standard for a distribution $\nu_{X}$.

\par In the following lemma, we show that the population Sobolev penalized risk minimizer $f^{\lambda}$ converges to the regression function $f_{0}$ in $L^{2}(\mu_{X})$-norm with rate $\calO(\lambda^{2})$. Additionally, the $L^{2}(\nu_{X})$-rate of its derivatives is $\calO(\lambda)$.

\begin{lemma}\label{lemma:rate:population:minimizer}
Suppose Assumptions~\ref{assumption:bounded:density:ratio},~\ref{assumption:regularity:f0} and~\ref{assumption:bounded:score} hold. Let $f^{\lambda}$ be the unique minimizer of the population risk~\eqref{eq:population:risk:regularization}. Then it follows that for each $\lambda>0$
\begin{align*}
&\|f^{\lambda}-f_{0}\|_{L^{2}(\mu_{X})}^{2} \\
&\lesssim\lambda^{2}\kappa\Big\{\|\Delta f_{0}\|_{L^{2}(\nu_{X})}^{2}+\|\nabla f_{0}\cdot\nabla(\log q)\|_{L^{2}(\nu_{X})}^{2}\Big\},  \\
&\|\nabla(f^{\lambda}-f_{0})\|_{L^{2}(\nu_{X})}^{2} \\
&\lesssim\lambda\kappa\Big\{\|\Delta f_{0}\|_{L^{2}(\nu_{X})}^{2}+\|\nabla f_{0}\cdot\nabla(\log q)\|_{L^{2}(\nu_{X})}^{2}\Big\}.
\end{align*}
\end{lemma}

\par Up to now, we have shown the convergence of the population Sobolev penalized risk minimizer. However, researchers are primarily  concerned with convergence rates of the empirical estimators obtained via a finite number of labeled data pairs $\euD=\{(X_{i},Y_{i})\}_{i=1}^{n}$. In the remaining part of this section, we mainly focus on the convergence rate analysis for the deep Sobolev regressor and its semi-supervised counterpart in~\eqref{eq:empirical:risk:convex}.

\subsection{Analysis for Deep Sobolev Regressor}

\par The theoretical foundation for simultaneous estimation of the regression function and its gradient is the following oracle-type inequality.

\begin{lemma}[Oracle inequality]\label{lemma:oracle:erm}
Suppose Assumptions~\ref{assumption:bounded:density:ratio} to~\ref{assumption:bounded:score} hold. Let $\what{f}_{\euD}^{\lambda}$ be the deep Sobolev regressor defined as~\eqref{eq:empirical:risk:convex}. Then it follows that for each $\lambda>0$ and each $n\geq\log N(B_{0}\delta,\calF,L^{2}(\euD))$,
\begin{align*}
&\bbE_{\euD\sim\mu^{n}}\Big[\|\what{f}_{\euD}^{\lambda}-f_{0}\|_{L^{2}(\mu_{X})}^{2}\Big] \\
&\lesssim\beta\lambda^{2}+\varepsilon_{\app}(\calF,\lambda)+\varepsilon_{\gen}(\calF,n), \\
&\bbE_{\euD\sim\mu^{n}}\Big[\|\nabla(\what{f}_{\euD}^{\lambda}-f_{0})\|_{L^{2}(\nu_{X})}^{2}\Big] \\
&\lesssim\beta\lambda+\lambda^{-1}\varepsilon_{\app}(\calF,\lambda)+\lambda^{-1}\varepsilon_{\gen}(\calF,n),
\end{align*}
where $\beta$ is a positive constant defined as
\begin{equation*}
\beta=\kappa\Big\{\|\Delta f_{0}\|_{L^{2}(\nu_{X})}^{2}+\|\nabla f_{0}\cdot\nabla(\log q)\|_{L^{2}(\nu_{X})}^{2}\Big\},
\end{equation*}
the approximation error $\varepsilon_{\app}(\calF,\lambda)$ and the generalization error $\varepsilon_{\gen}(\calF,n)$ are defined, respectively, as
\begin{align*}
&\varepsilon_{\app}(\calF,\lambda) \\
&=\inf_{f\in\calF}\Big\{\|f-f_{0}\|_{L^{2}(\mu_{X})}^{2}+\lambda\|\nabla(f-f_{0})\|_{L^{2}(\nu_{X})}^{2}\Big\}, \\ 
&\varepsilon_{\gen}(\calF,n) \\
&=\frac{B_{0}^{2}+\sigma^{2}}{\log^{-1}n}\inf_{\delta>0}\Big\{\Big(\frac{2\log N(B_{0}\delta,\calF,L^{2}(\euD))}{n}\Big)^{\frac{1}{2}}+\delta\Big\}.
\end{align*}
\end{lemma}

\par As discussed in Section~\ref{section:analysis:regression},~\cite[Chapter 21]{gyorfi2002distribution} has investigated an optimization problem similar to the deep Sobolev regressor. However, to the best of our knowledge, we are the first to demonstrate the oracle inequality for the gradient of estimator. The proof employs a similar technique as that of Lemma~\ref{lemma:rate:population:minimizer}. Specifically, the deep Sobolev regressor acts as the minimizer of~\eqref{eq:empirical:risk:convex}, which implies that it satisfies a variational inequality derived from the first-order optimality condition~\cite{Hinze2009Optimization,troltzsch2010optimal}. By utilizing standard techniques from statistical learning theory, we are able to derive the desired oracle inequality.

\par In simple terms, if we select an appropriate neural network class and have a sufficiently large number of labeled data pairs, we can make the approximation error and generalization error arbitrarily small. Consequently, the overall error is primarily determined by the regularization parameter $\lambda$. At this point, the error bound aligns with rates in Lemma~\ref{lemma:rate:population:minimizer}.

\par Recall the oracle inequality derived in Lemma~\ref{lemma:oracle:erm:regression}, which requires the neural network to approximate the regression function while restricting its gradient norm. In contrast, the approximation term in Lemma~\ref{lemma:oracle:erm} necessitates the neural network to approximate both the regression function and its derivatives simultaneously. Thus, we now introduce the following approximation error bound in $H^1$-norm.

\begin{lemma}[Approximation in $H^{1}$-norm]\label{lemma:approximation:H1}
Let $\Omega\subseteq K\subseteq\bbR^{d}$ be two bounded domain. Set the hypothesis class as a deep ReQU neural network $\calF=\calN(L,W,S)$ with $L=\calO(\log N)$ and $S=\calO(N^{d})$. Then for each $\phi\in C^{s}(K)$ with $s\in\bbN_{\geq2}$, there exists $f\in\calF$ such that
\begin{align*}
\|f-\phi\|_{L^{2}(\mu_{X})}&\leq CN^{-s}\|\phi\|_{C^{s}(K)}, \\
\|\nabla(f-\phi)\|_{L^{2}(\nu_{X})}&\leq CN^{-(s-1)}\|\phi\|_{C^{s}(K)},
\end{align*}
where $C$ is a constant independent of $N$.
\end{lemma}

\par With the aid of previously prepared lemmas, we have following convergence rates for the deep Sobolev regressor.

\begin{theorem}[Convergence rates]\label{theorem:erm:rate}
Suppose Assumptions~\ref{assumption:bounded:density:ratio} to~\ref{assumption:bounded:score} hold. Let $\Omega\subseteq K\subseteq\bbR^{d}$ be two bounded domain. Assume that $f_{0}\in C^{s}(K)$ with $s\in\bbN_{\geq2}$. Set the hypothesis class as a deep ReQU neural network class $\calF=\calN(L,W,S)$ with $L=\calO(\log n)$ and $S=\calO(n^{\frac{d}{d+4s}})$. Let $\what{f}_{\euD}^{\lambda}$ be the deep Sobolev regressor defined in~\eqref{eq:empirical:risk:convex} with regularization parameter $\lambda>0$. Then it follows that
\begin{align*}
&\bbE_{\euD\sim\mu^{n}}\Big[\|\what{f}_{\euD}^{\lambda}-f_{0}\|_{L^{2}(\mu_{X})}^{2}\Big] \\
&\leq\calO(\lambda^{2})+\calO\Big(n^{-\frac{2s}{d+4s}}\log^{4}n\Big), \\
&\bbE_{\euD\sim\mu^{n}}\Big[\|\nabla(\what{f}_{\euD}^{\lambda}-f_{0})\|_{L^{2}(\nu_{X})}^{2}\Big] \\
&\leq\calO(\lambda)+\calO\Big(\lambda^{-1}n^{-\frac{2s}{d+4s}}\log^{4}n\Big).
\end{align*}
Further, setting $\lambda=\calO(n^{-\frac{s}{d+4s}}\log^{2}n)$ implies
\begin{align*}
\bbE_{\euD\sim\mu^{n}}\Big[\|\what{f}_{\euD}^{\lambda}-f_{0}\|_{L^{2}(\mu_{X})}^{2}\Big]&\leq\calO\Big(n^{-\frac{2s}{d+4s}}\log^{4}n\Big), \\
\bbE_{\euD\sim\mu^{n}}\Big[\|\nabla(\what{f}_{\euD}^{\lambda}-f_{0})\|_{L^{2}(\nu_{X})}^{2}\Big]
&\leq\calO\Big(n^{-\frac{s}{d+4s}}\log^{2}n\Big).
\end{align*}
Here the constant behind the big $\calO$ notation is independent of $n$.
\end{theorem}

\par  Theorem~\ref{theorem:erm:rate} provides theoretical guidance for the selection of the size of neural networks and the choice of regularization parameters. In comparison to the regularization parameter $\lambda=\calO(n^{-\frac{2s}{d+2s}}\log^{3}n)$ employed in Theorem~\ref{theorem:erm:rate:regression}, $\lambda=\calO(n^{-\frac{s}{d+4s}}\log^{2}n)$ utilized in Theorem~\ref{theorem:erm:rate} is much larger.
The $L^{2}(\mu_{X})$-rate $\calO(n^{-\frac{2s}{d+4s}})$ of the deep Sobolev regressor does not attain the minimax optimality. Furthermore, the convergence rate $\calO(n^{-\frac{s}{d+4s}}\log^{4}n)$ for the derivatives is also slower than the minimax optimal rate $\calO(n^{-\frac{2(s-1)}{d+2s}})$ derived in~\cite{Stone1982optimal}.
\subsection{Analysis for Semi-Supervised Deep Sobolev Regressor}

\par In scenarios where the distribution $\nu_{X}$ is unknown, estimating the Sobolev penalty using the unlabeled data becomes crucial. In qualitative terms, having a sufficiently large number of unlabeled data points allows us to estimate the regularization term with an arbitrarily small error. However, the following questions are not answered quantitatively:
\begin{quote}
\emph{\normalsize How does the error of the semi-supervised estimator depend on the number of unlabeled data? How does the unlabeled data in semi-supervised learning improve the standard supervised estimators?}
\end{quote}
In this section, we provide a comprehensive and rigorous analysis for the semi-supervised deep Sobolev regressor. To begin with, we present the following oracle inequality.

\begin{assumption}[Bounded derivatives of hypothesis]\label{assumption:boundedness:f:derivative}
There exists positive constants $\{B_{1,k}\}_{k=1}^{d}$, such that $\sup_{x\in\Omega}|D_{k}f_{0}(x)|\leq B_{1,k}$ for $1\leq k\leq d$. Further, the first-order partial derivatives of functions in hypothesis class $\calF$ are also bounded, i.e., $\sup_{x\in\Omega}|D_{k}f(x)|\leq B_{1,k}$ for each $1\leq k\leq d$ and $f\in\calF$. Denote by $B_{1}^{2}:=\sum_{k=1}^{d}B_{1,k}^{2}$
\end{assumption}

\par The inclusion of Assumption~\ref{assumption:boundedness:f:derivative} is essential in the analysis of generalization error that involves derivatives, as it plays a similar role to Assumption~\ref{assumption:boundedness:f} in the previous analysis.

\begin{lemma}[Oracle inequality]\label{lemma:oracle:erm:data}
Suppose Assumptions~\ref{assumption:bounded:density:ratio} to~\ref{assumption:boundedness:f:derivative} hold. Let $\what{f}_{\euD,\euS}^{\lambda}$ be the semi-supervised deep Sobolev regressor defined in~\eqref{eq:empirical:risk:convex}. For each $\lambda>0$, $n\geq\log N(B_{0}\delta,\calF,L^{2}(\euD))$ and $m\geq\max_{1\leq k\leq d}\log N(B_{1,k}\delta,D_{k}\calF,L^{2}(\euS))$, 
\begin{align*}
&\bbE_{(\euD,\euS)\sim\mu^{n}\times\nu_{X}^{m}}\Big[\|\what{f}_{\euD,\euS}^{\lambda}-f_{0}\|_{L^{2}(\mu_{X})}^{2}\Big] \\
&\lesssim\tilde{\beta}\lambda^{2}+\varepsilon_{\app}(\calF,\lambda) \\
&\quad+\varepsilon_{\gen}(\calF,n)+\varepsilon_{\gen}^{\reg}(\nabla\calF,m), \\
&\bbE_{(\euD,\euS)\sim\mu^{n}\times\nu_{X}^{m}}\Big[\|\nabla(\what{f}_{\euD,\euS}^{\lambda}-f_{0})\|_{L^{2}(\nu_{X})}^{2}\Big] \\
&\lesssim\tilde{\beta}\lambda+\lambda^{-1}\varepsilon_{\app}(\calF,\lambda) \\
&\quad+\lambda^{-1}\varepsilon_{\gen}(\calF,n)+\lambda^{-1}\varepsilon_{\gen}^{\reg}(\nabla\calF,m),
\end{align*}
where $\tilde{\beta}$ is a positive constant defined as $\tilde{\beta}=\beta+B_{1}^{2}$, the approximation error $\varepsilon_{\app}(\calF,\lambda)$ and the generalization error $\varepsilon_{\gen}(\calF,n)$ are defined as those in Lemma~\ref{lemma:oracle:erm}. The generalization error $\varepsilon_{\gen}^{\reg}(\nabla\calF,m)$ corresponding to the regularization term are defined as
\begin{equation*}
\varepsilon_{\gen}^{\reg}(\nabla\calF,m)=B_{1}^{2}\inf_{\delta>0}\Big\{\max_{1\leq k\leq d}\frac{N(B_{1,k}\delta,D_{k}\calF,L^{2}(\euS))}{m}+\delta\Big\}.
\end{equation*}
\end{lemma}

\par In comparison to Lemma~\ref{lemma:oracle:erm}, the error bound has not undergone significant changes, and it has only been augmented by one additional generalization error associated with the regularization term. Further, this term vanishes as the number of unlabeled data increases.

\par In particular, we focus on the scenario where the distributions of covariates in both labeled and unlabeled data are identical, i.e., $\nu_{X}=\mu_{X}$. When only the labeled data pairs (e.g.,~\eqref{eq:gp}) are used, the generalization error corresponding to the regularization term is denoted as $\varepsilon_{\gen}^{\reg}(\nabla\calF,n)$. In contrast, for the semi-supervised Sobolev regressor, the corresponding generalization term becomes:
\begin{multline*}
\varepsilon_{\gen}^{\reg}(\nabla\calF,m+n) \\
=B_{1}^{2}\inf_{\delta>0}\Big\{\max_{1\leq k\leq d}\frac{\log N(B_{1,k}\delta,D_{k}\calF,L^{2}(\euS))}{m+n}+\delta\Big\}.
\end{multline*}
It is worth noting that for every $m\in\bbN_{\geq1}$, the inequality $\varepsilon_{\gen}^{\reg}(\nabla\calF,m+n)<\varepsilon_{\gen}^{\reg}(\nabla\calF,n)$ holds. This demonstrates the provable advantages of incorporation of unlabeled data in the semi-supervised learning framework.

\par Finally, we derive convergence rates of the semi-supervised deep Sobolev regressor.

\begin{theorem}[Convergence rates]\label{theorem:erm:rate:data}
Suppose Assumptions~\ref{assumption:bounded:density:ratio} to~\ref{assumption:boundedness:f:derivative} hold. Let $\Omega\subseteq K\subseteq\bbR^{d}$ be two bounded domain. Assume that $f_{0}\in C^{s}(K)$ with $s\in\bbN_{\geq2}$. Set the hypothesis class as a deep ReQU neural network class $\calF=\calN(L,W,S)$ with $L=\calO(\log n)$ and $S=\calO(n^{\frac{d}{d+4s}})$. Let $\what{f}_{\euD,\euS}^{\lambda}$ be the regularized empirical risk minimizer defined as~\eqref{eq:empirical:risk:regularization:data} with regularization parameter $\lambda=\calO(n^{-\frac{s}{d+4s}}\log^{2}n)$. Then it follows that
\begin{align*}
&\bbE_{(\euD,\euS)\sim\mu^{n}\times\nu_{X}^{m}}\Big[\|\what{f}_{\euD,\euS}^{\lambda}-f_{0}\|_{L^{2}(\mu_{X})}^{2}\Big] \\
&\leq\calO\Big(n^{-\frac{2s}{d+4s}}\log^{4}n\Big)+\calO\Big(n^{\frac{d}{d+4s}}\log^{4}nm^{-1}\Big), \\
&\bbE_{(\euD,\euS)\sim\mu^{n}\times\nu_{X}^{m}}\Big[\|\nabla(\what{f}_{\euD,\euS}^{\lambda}-f_{0})\|_{L^{2}(\nu_{X})}^{2}\Big] \\
&\leq\calO\Big(n^{-\frac{s}{d+4s}}\log^{2}n\Big)+\calO\Big(n^{\frac{d+s}{d+4s}}\log^{2}nm^{-1}\Big).
\end{align*}
Here the constant behind the big $\calO$ notation is independent of $n$.
\end{theorem}

\par For the number of unlabeled data $m$ sufficiently large, the convergence rate of the semi-supervised deep Sobolev regressor tends to the rate derived in Theorem~\ref{theorem:erm:rate}.

\section{Applications and Numerical Experiments}\label{section:applications}
In this section, we demonstrate the effectiveness of our proposed SDORE in the context of derivative estimation, and nonparametric variable selection.

\subsection{Derivative Estimation}

\par In this section we give a one-dimensional example, and a detailed example in two dimensions is shown in Appendix~\ref{appendix:experiments:derivative}.

\begin{example}\label{example:onedim}
Let the regression function be $f_0(x) = 1 + 36x^2 - 59x^3 + 21x^5 + 0.5\cos(\pi x)$. The labeled data pairs are generated from a regression model $Y=f_0(X)+\xi$, where $X$ is sampled from the uniform distribution on $[0,1]$, and $\xi$ is sampled from a Gaussian distribution $N(0,\sigma^{2})$. Here the variance $\sigma^{2}$ is determined by a given signal-to-noise ratio $\frac{\bbE[f_{0}^{2}(X)]}{\sigma^{2}}=30$. The unlabeled data are also drawn from the uniform distribution on $[0,1]$. The regularization parameter is set as $\lambda=0.005$.
\end{example}

\par To demonstrate the effectiveness of SDORE in scenario where only few labeled sample is available, we conducted SDORE using 40 labeled data pairs and an additional 1000 unlabeled samples. The comparisons with the least-squares regression are presented in Figure~\ref{fig:toy}, which includes point-wise comparisons of function values and derivatives. In the upper panel, the least-squares estimator generally matches the target function. However, the least-squares estimator fits the noise in the data rather than the underlying patterns near the left and right endpoints. Also, the lower panel shows that its estimated derivatives is inaccurate and unstable near the left and right endpoints. In comparison, our SDORE method successfully estimates the regression function and its derivatives simultaneously, and the regularization avoids the overfitting on the primitive function.

\par The errors in derivative estimates by SDORE are more pronounced near the interval boundary. This is primarily due to the lack of observations of function values outside the intervals, preventing accurate estimation of the boundary derivatives. 
From a theoretical perspective, the convergence of the derivative in $L^{2}$-norm is guaranteed by Theorem~\ref{theorem:erm:rate:regression}. However, this theorem does not provide guarantees for accuracy on the boundary. Estimating the boundary error requires the interior estimation of second-order derivatives, as outlined in the trace theorem~\cite[Theorem 1 in Section 5.5]{evans2010partial}.

\begin{figure}[htbp] 
\centering
\includegraphics[width=\linewidth]{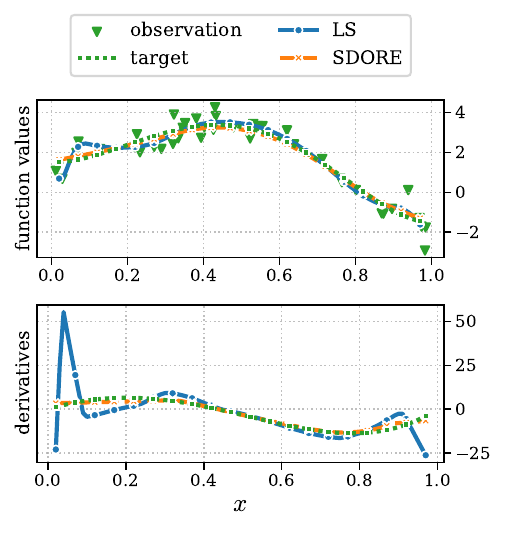}
\caption{Numerical results of Example~\ref{example:onedim}. (left) Scatter plot of noisy observations (paired data used for supervised learning), line plot of the ground-truth regression function and its values predicted by least-squares (LS) regression and SDORE. (right) The ground truth derivative function and its estimated values by LS and SDORE.}
\label{fig:toy}
\end{figure}

\subsection{Nonparametric Variable Selection}

\par Deep neural network is a widely utilized tool in nonparametric statistics and machine learning. It effectively captures the nonlinear relationship between the covariate vector and the corresponding label. However, the interpretability of neural network estimators has faced significant criticism. This is primarily due to the inability to determine the relevance of variables in the covariate vector and quantify their impact on the neural network's output. 

\par In this section, we propose a novel approach to address this issue by measuring the importance of a variable through its corresponding partial derivatives. Leveraging the deep Sobolev regressor, we introduce a nonparametric variable selection technique with deep neural networks. Remarkably, our method incorporates variable selection as a natural outcome of the regression process, eliminating the need for the design of a separate algorithm for this purpose.

\par Before proceeding, we impose additional sparsity structure on the underlying regression function, that is, there exists $f_{0}^{*}:\bbR^{d^{*}}\rightarrow\bbR$ ($1\leq d^{*}\leq d$) such that
\begin{multline}\label{eq:sparsity}
f_{0}(x_{1},\ldots,x_{d})=f_{0}^{*}(x_{j_{1}},\ldots,x_{j_{d^{*}}}), \\
 \{j_{1},\ldots,j_{d^{*}}\}\subseteq[d].
\end{multline}
This sparsity setting has garnered significant attention in the study of linear models and additive models, as extensively discussed in~\cite{Wainwright2019High}. In the context of reproducing kernel Hilbert space,~\cite{Rosasco2010regularization,Mosci2012Is,Rosasco2013Nonparametric} introduced a nonparametric variable selection algorithm. Nevertheless, their approach and analysis heavily depend on the finite dimensional explicit representation of the estimator, making it unsuitable for generalizing to deep neural network estimators.

\par We introduce the definition of relevant set, which was proposed by~\cite[Definition 10]{Rosasco2013Nonparametric}. The goal of the variable selection is to estimate the relevant set.

\begin{definition}[Relevant set]
Let $f:\bbR^{d}\rightarrow\bbR$ be a differentiable function. A variable $k\in[d]$ is irrelevant for the function $f$ with respect to the probability measure $\nu_{X}$, if $D_{k}f(X)=0$ $\nu_{X}$-almost surely, and relevant otherwise. The set of relevant variables is defined as 
\begin{equation*}
\calI(f)=\{k\in[d]:\|D_{k}f\|_{L^{2}(\nu_{X})}>0\}.
\end{equation*}
\end{definition}

\subsubsection{Convergence Rates and Selection Consistency}

\begin{assumption}[Sparsity of the regression function]\label{assumption:sparse}
The number of relevant variables is less than the dimension $d$, that is, there exists a positive integer $d^{*}\leq d$, such that $|\calI(f_{0})|=d^{*}$.
\end{assumption}

\par Under Assumption~\ref{assumption:sparse}, our focus is solely on estimating the low-dimensional function $f_{0}^{*}$ in~\eqref{eq:sparsity} using deep neural networks. Consequently, the approximation and generalization error in Lemma~\ref{lemma:oracle:erm} are reliant solely on the intrinsic dimension $d^{*}$. This implies an immediate result as follows.

\begin{corollary}\label{corollary:variable:selection}
Suppose Assumptions~\ref{assumption:bounded:density:ratio} to~\ref{assumption:sparse} hold. Let $\Omega\subseteq K\subseteq\bbR^{d}$ be two bounded domain. Assume that $f_{0}\in C^{s}(K)$ with $s\in\bbN_{\geq2}$. Set the hypothesis class $\calF$ as a ReQU neural network class $\calF=\calN(L,W,S)$ with $L=\calO(\log n)$ and $S=\calO(n^{\frac{d^{*}}{d^{*}+4s}})$. Let $\what{f}_{\euD}^{\lambda}$ be the regularized empirical risk minimizer defined as~\eqref{eq:empirical:risk:regularization:data} with regularization parameter $\lambda=\calO(n^{-\frac{s}{d^{*}+4s}}\log^{2}n)$. Then the following inequality holds
\begin{align*}
\bbE_{\euD\sim\mu^{n}}\Big[\|\what{f}_{\euD}^{\lambda}-f_{0}\|_{L^{2}(\mu_{X})}\Big]&\leq\calO\Big(n^{-\frac{s}{d^{*}+4s}}\log^{4}n\Big), \\
\bbE_{\euD\sim\mu^{n}}\Big[\|\nabla(\what{f}_{\euD}^{\lambda}-f_{0})\|_{L^{2}(\nu_{X})}\Big]
&\leq\calO\Big(n^{-\frac{s}{2(d^{*}+4s)}}\log^{2}n\Big).
\end{align*}
\end{corollary}

\par The convergence rate presented in Corollary~\ref{corollary:variable:selection} is solely determined by the intrinsic dimension $d^{*}$ and remains unaffected by the data dimension $d$, which effectively mitigates the curse of dimensionality when $d^{*}$ is significantly smaller than $d$.

\par Furthermore, we establish the selection properties of the deep Sobolev regressor, which directly follow from the convergence of derivatives.

\begin{corollary}[Selection consistency]\label{corollary:selection:consistency}
Under the same conditions as Corollary~\ref{corollary:variable:selection}. It follows that
\begin{equation*}
\lim_{n\rightarrow\infty}\pr\Big\{\calI(f_{0})=\calI(\what{f}_{\euD}^{\lambda})\Big\}=1,
\end{equation*}
where $\lambda=\calO(n^{-\frac{s}{d^{*}+4s}}\log^{2}n)$.
\end{corollary}

\par Corollary~\ref{corollary:selection:consistency} demonstrates that, given a sufficiently large number of data pairs, the estimated relevant set $\calI(\what{f}_{\euD}^{\lambda})$ is equal to the ground truth relevant set $\calI(f_{0})$ with high probability. In comparison,~\cite[Theorem 11]{Rosasco2013Nonparametric} only provided a one-side consistency
\begin{equation*}
\lim_{n\rightarrow\infty}\pr\Big\{\calI(f_{0})\subseteq\calI(\what{f}_{\euD}^{\lambda})\Big\}=1,
\end{equation*}
were unable to establish the converse inclusion.

\subsubsection{Numerical Experiments}

\par In this section, we present a high-dimensional example which has sparsity structure to verify the performance of SDORE in variable selection. The additional experiments for variable selection are shown in Appendix~\ref{appendix:experiments:selection}.

\begin{example}\label{example:selection}
Let the regression function be 
\begin{equation*}
f_0(x)=\sum_{i=1}^{3}\sum_{j=i+1}^{4}x_{i}x_{j}.
\end{equation*}
Suppose the covariate in both labeled and unlabeled data are sampled from the uniform distribution on $[0,1]^{20}$. The label $Y$ is generated from the regression model $Y=f_{0}(X)+\xi$, where $\xi$ is the noise term sampled from a Gaussian distribution with the signal-to-noise ratio to be 25, in the same way as Example~\ref{example:onedim}. The regularization parameter is set as $\lambda=1.0\times10^{-2}$.
\end{example}

\par In real-world applications, the process of labeling data can be prohibitively costly, resulting in a limited availability of labeled data. Conversely, there is an abundance of unlabeled data that is readily accessible. Hence, it becomes crucial to leverage few labeled data alongside a substantial amount of unlabeled data for the purpose of variable selection. Nevertheless, the task of variable selection with few labeled samples presents significant challenges. Due to the scarcity of data points, there is a restricted range of variability within the dataset, posing difficulties in accurately determining the variables that hold true significance in predicting the desired outcome.

\par To demonstrate the effectiveness of SDORE in this challenging scenario, we employ SDORE for the variable selection in this example, utilizing 50 labeled data pairs and an additional sample containing 100 unlabeled covariate vectors. Additionally, we use least-squares regression on the same data as a comparison. Figure~\ref{fig:selection} visually presents the empirical mean square (EMS) of estimated partial derivatives with respect to each variable on the test set, that is, for each $1\leq k\leq d$,
\begin{equation*}
\mathrm{EMS}_{k}=\frac{1}{n}\sum_{i=1}^{n}|D_{k}\what{f}(X_{i,k})|^{2}.
\end{equation*}
Here $\what{f}$ is an estimator, $\{X_{i}\}_{i=1}^{n}$ is a set of test data, and $X_{i,k}$ represents the $k$-th element of $X_{i}$.
The results by SDORE reveals that the derivatives with respect to relevant variables $x_1$ to $x_4$ are significantly larger than those of the other variables, while least-squares regression wrongly regards $x_9$ as relevant variables, possibly due to the lack of paired training samples.
This shows that our proposed method can estimate the derivatives accurately, which facilitates the variable selection. 
We select the variables by setting a 75\% quantile threshold of the estimated partial derivatives. The partial derivatives greater than the threshold is considered relevant.
Additionally, Figure~\ref{fig:selection} displays the mean selection error (SE), calculated as the mean of the false positive rate and false negative rate as defined by~\cite{Rosasco2013Nonparametric}, as well as the root mean squared prediction error (PE) on the regression function.
Notably, the results consistently demonstrate the superior performance of SDORE over least-squares regression, underscoring the advantages of incorporating unlabeled data.

\begin{remark}
Since in Figure~\ref{fig:selection}, for SDORE, the estimated partial derivatives with respect to the first four features is significantly larger than others, we can choose the threshold directly. In other application scenarios, if we can not observe such a clear difference, we can employ the strategy such as cross-validation (CV) to determine the number of features. Specifically, the cross-validation process involves dividing the dataset into a training set and a validation set independently. The training set is used for Sobolev regression, and the model's performance is evaluated on the validation set. The mean square of partial derivatives, also known as the important score, is sorted from largest to smallest. The cross-validation process begins by selecting the feature with the largest important score, and adds the remaining most important features incrementally until the accuracy in the validation set no longer shows improvement.
\end{remark}

\begin{figure}[htbp]  
\centering
\includegraphics[width=\linewidth]{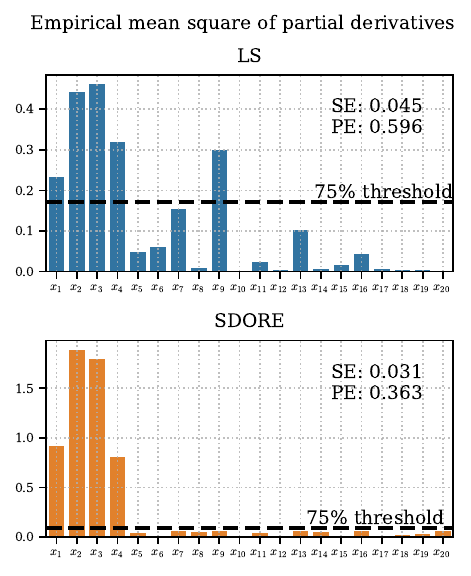}
\caption{Numerical results of Example~\ref{example:selection}. The empirical mean square of the partial derivatives of the regression function $f_{0}$ (which depends only on the $x_1$ to $x_4$), estimated by least-squares fitting (LS, left) and SDORE (right). 
The dashed line is the 75\% quantile threshold for variable selection.
We also report the mean selection error (SE) for the estimated derivative function and the root mean squared prediction error (PE) for the primitive function by each method.}
\label{fig:selection}
\end{figure}

\section{Conclusion}\label{section:conclusion}

\par In this paper, we present a novel semi-supervised deep Sobolev regressor that allows for the simultaneous estimation of the underlying regression function and its gradient. We provide a thorough convergence rate analysis for this estimator, demonstrating the provable benefits of incorporating unlabeled data into the semi-supervised learning framework. To the best of our knowledge, these results are original contributions to the literature in the field of deep learning, thereby enhancing the theoretical understanding of semi-supervised learning and gradient penalty strategy. From an application standpoint, our approach introduces powerful new tools for nonparametric variable selection. Moreover, our method has demonstrated exceptional performance in various numerical examples, further validating its efficacy.

\par We would like to highlight the generality of our method and analysis, as it can be extended to various loss functions. In our upcoming research, we have extended the Sobolev penalized strategy to encompass a wide range of statistical and machine learning tasks, such as density estimation, deconvolution, classification, and quantile regression. Furthermore, we have discovered the significant role that the semi-supervised deep Sobolev regressor plays in addressing inverse problems related to partial differential equations. There still remains some challenges that need to be addressed. For example, in Theorem~\ref{theorem:erm:rate}, the $L^{2}(\mu_{X})$-rate $\calO(n^{-\frac{2s}{d+4s}}\log^{4}n)$ of the deep Sobolev regressor does not attain the minimax optimality. Additionally, the convergence rate $\calO(n^{-\frac{s}{d+4s}}\log^{2}n)$ for the derivatives is also slower than the minimax optimal rate $\calO(n^{-\frac{2(s-1)}{d+2s}})$ derived in~\cite{Stone1982optimal}. Moreover, while Corollary~\ref{corollary:selection:consistency} establishes selection consistency, it does not provide the rate of convergence. Furthermore, an interesting avenue for future research would be to investigate deep nonparametric regression with a sparse/group sparse penalty.

\appendices

\section{Supplemental Definitions and Lemmas}

\par In this section, we present some definitions and lemmas for preparation. We first extend Green's formula in Lebesgue measure to general measures.

\begin{lemma}[Green's formula in general measure]\label{lemma:Green}
Let $\nu_{X}$ be a probability measure on $\Omega$ with density function $q(x)\in W^{1,\infty}(\Omega)$. Let $u\in H^{1}(\nu_{X})$ and let $v\in H^{2}(\nu_{X})$ satisfying $\nabla v\cdot\vn=0$ a.e. on $\partial\Omega$, where $\vn$ is the unit normal to the boundary. Then it follows that
\begin{equation*}
-(\nabla u,\nabla v)_{L^{2}(\nu_{X})}=(\Delta v+\nabla v\cdot\nabla(\log q),u)_{L^{2}(\nu_{X})}
\end{equation*}
\end{lemma}

\begin{proof}[Proof of Lemma~\ref{lemma:Green}]
It is straightforward that
\begin{align*}
&-(\nabla u,\nabla v)_{L^{2}(\nu_{X})} \\
&=-\int_{\Omega}\nabla u\cdot\nabla vqdx \\
&=-\int_{\Omega}\nabla\cdot(\nabla vqu)dx+\int_{\Omega}\nabla\cdot(\nabla vq)udx \\
&=-\int_{\partial\Omega}(\nabla v\cdot\vn)uqds+\int_{\Omega}\nabla\cdot(\nabla vq)udx \\
&=\int_{\Omega}\Delta vuqdx+\int_{\Omega}\nabla v\cdot\nabla(\log q)uqdx \\
&=(\Delta v,u)_{L^{2}(\nu_{X})}+(\nabla v\cdot\nabla(\log q),u)_{L^{2}(\nu_{X})},
\end{align*}
where the second equality holds from integration by parts, the third equality follows from the divergence theorem~\cite[Theorem 1 in Section C.2]{evans2010partial}, and the forth one used the assumption $\nabla v\cdot\vn=0$ and the equality $\nabla(\log q)=\nabla q/q$.
\end{proof}

\par We next present the maximal inequality for sub-Gaussian variables.

\begin{lemma}\label{lemma:subgaussian:max}
Let $\xi_{j}$ be $\sigma^{2}$-sub-Gaussian for each $1\leq j\leq N$. Then
\begin{equation*}
\bbE\Big[\max_{1\leq j\leq N}\xi_{j}^{2}\Big]\leq 4\sigma^{2}(\log N+1).
\end{equation*}
\end{lemma}

\begin{proof}[Proof of Lemma~\ref{lemma:subgaussian:max}]
By Jensen's inequality, it is straightforward that
\begin{align*}
\exp\Big(\frac{\lambda}{2\sigma^{2}}\bbE\Big[\max_{1\leq j\leq N}\xi_{j}^{2}\Big]\Big)
&\leq\bbE\Big[\max_{1\leq j\leq N}\exp\Big(\frac{\lambda\xi_{j}^{2}}{2\sigma^{2}}\Big)\Big] \\
&\leq N\bbE\Big[\exp\Big(\frac{\lambda\xi_{1}^{2}}{2\sigma^{2}}\Big)\Big]
\leq\frac{N}{\sqrt{1-\lambda}},
\end{align*}
where the last inequality holds from~\cite[Theorem 2.6]{Wainwright2019High} for each $\lambda\in[0,1)$. Letting $\lambda=1/2$ yields the desired inequality.
\end{proof}

\par To measure the complexity of a function class, we next introduce the Vapnik-Chervonenkis (VC) dimension and some associated lemmas.

\begin{definition}[VC-dimension]\label{def:vcdim}
Let $\calF$ be a class of functions from $\Omega$ to $\{\pm1\}$. For any non-negative integer $m$, we define the growth function of $\calF$ as
\begin{equation*}
\Pi_{\calF}(m)=\max_{\{x_{i}\}_{i=1}^{m}\subseteq\Omega}\big|\{(f(x_{1}),\ldots,f(x_{m})):f\in\calF\}\big|.
\end{equation*}
A set $\{x_{i}\}_{i=1}^{m}$ is said to be shattered by $\calF$ when 
\begin{equation*}
|\{(f(x_{1}),\ldots,f(x_{m})):f\in\calF\}|=2^{m}.
\end{equation*}
The VC-dimension of $\calF$, denoted $\vcdim(\calF)$, is the size of the largest set that can be shattered by $\calF$, that is, $\vcdim(\calF)=\max\{m:\Pi_{\calF}(m)=2^{m}\}$. For a class $\calF$ of real-valued functions, we define $\vcdim(\calF)=\vcdim(\sign(\calF))$.
\end{definition}

\par The following lemma provides a VC-dimension bound for the empirical covering number.

\begin{lemma}[{\cite[Theorem 12.2]{Anthony1999neural}}]\label{lemma:cn:vcdim}
Let $\calF$ be a set of real functions from $\Omega$ to the bounded interval $[-B,B]$. Let $\delta\in(0,1)$ and $\euD=\{X_{i}\}_{i=1}^{n}\subseteq\Omega$. Then for each $1\leq p\leq\infty$ and $n\geq\vcdim(\calF)$, the following inequality holds
\begin{equation*}
\log N(\delta,\calF,L^{p}(\euD))\leq c\vcdim(\calF)\log(nB\delta^{-1}),
\end{equation*}
where $c>0$ is an absolute constant.
\end{lemma}

\par Lemma~\ref{lemma:cn:vcdim} demonstrates that the metric entropy of a function class is bounded by its VC-dimension. The following lemma provides a VC-dimension bound for a deep neural network classes with a piecewise-polynomial activation function, and with a fixed architecture, i.e., the positions of the nonzero parameters are fixed.

\begin{lemma}[{\cite[Theorem 7]{Bartlett2019nearly}}]\label{lemma:vcdim}
Let $\calN$ be a deep neural network architecture with $L$ layers and $S$ non-zero parameters. The activation function is piecewise-polynomial. Then $\vcdim(\calN)\leq cLS\log(S)$, where $c>0$ is an absolute constant.
\end{lemma}

\par With the help of Lemmas~\ref{lemma:cn:vcdim} and~\ref{lemma:vcdim}, we can bound the metric entropy of the deep neural networks by its depth and number of nonzero parameters as the following lemma. The proof of this lemma is inspired by~\cite[Lemma 5]{schmidt2020nonparametric}.

\begin{lemma}\label{lemma:covering:number:sparse}
Let $\calN\subseteq\calN(L,W,S)$ be a set of deep neural networks from $\Omega$ to the bounded interval $[-B,B]$. The activation function is piecewise-polynomial. Let $\delta\in(0,1)$ and $\euD=\{X_{i}\}_{i=1}^{n}\subseteq\Omega$. Then 
\begin{equation*}
\log N(\delta,\calN,L^{2}(\euD))\leq cLS\log(S)\log\Big(\frac{nB}{\delta}\Big),
\end{equation*}
where $c>0$ is an absolute constant.
\end{lemma}

\begin{proof}[Proof of Lemma~\ref{lemma:covering:number:sparse}]
Before proceeding, it follows from the technique of removal of inactive nodes~\cite[eq. (9)]{schmidt2020nonparametric} that 
\begin{equation}\label{eq:lemma:covering:number:sparse:1}
\calN\subseteq\calN(L,W,S)=\calN(L,W\wedge S,S).
\end{equation}
For each deep neural network in $\calN(L,W,S)$, the number of parameters $T$ satisfies  
\begin{equation}\label{eq:lemma:covering:number:sparse:2}
\begin{aligned}
T&:=\sum_{\ell=0}^{L}(N_{\ell}+1)N_{\ell+1}\leq(L+1)2^{-L}\prod_{\ell=0}^{L}(N_{\ell}+1) \\
&\leq\prod_{\ell=0}^{L}(N_{\ell}+1)\leq(W+1)^{L+1}\leq(S+1)^{L+1},
\end{aligned}
\end{equation}
where the last inequality is due to~\eqref{eq:lemma:covering:number:sparse:1}. Then there exist $\binom{T}{s}$ combinations to pick $s$ non-zero parameters from all $T$ parameters, which yields a partition
\begin{align*}
\calN^{s}
&:=\Big\{\phi\in\calN:\sum_{\ell=0}^{L}(\|A_{\ell}\|_{0}+\|b_{\ell}\|_{0})=s\Big\} \\
&=\big\{\calN_{1}^{s},\ldots,\calN_{m}^{s}\big\}, \quad m_{s}=\binom{T}{s},
\end{align*}
where the deep neural networks in the same subset have the same positions of the non-zeros parameters. Consequently, 
\begin{align*}
&N(\delta,\calN,L^{2}(\euD)) \\
&=\sum_{s=1}^{S}N(\delta,\calN^{s},L^{2}(\euD))
=\sum_{s=1}^{S}\sum_{i=1}^{m_{s}}N(\delta,\calN_{i}^{s},L^{2}(\euD)) \\
&\leq\sum_{s=1}^{S}\sum_{i=1}^{m_{s}}\Big(\frac{nB}{\delta}\Big)^{\vcdim(\calN_{i}^{s})}
\leq\sum_{s=1}^{S}\binom{T}{s}\Big(\frac{nB}{\delta}\Big)^{cLs\log(s)} \\
&\leq\sum_{s=1}^{S}(S+1)^{(L+1)s}\Big(\frac{nB}{\delta}\Big)^{cLs\log(s)} \\
&\leq(S+1)^{(L+1)(S+1)}\Big(\frac{nB}{\delta}\Big)^{cL(S+1)\log(S)},
\end{align*}
where the first inequality holds from Lemma~\ref{lemma:cn:vcdim}, and the second inequality follows from Lemma~\ref{lemma:vcdim}. The third inequality used the inequality $\binom{T}{s}\leq T^{s}$ and~\eqref{eq:lemma:covering:number:sparse:2}. Taking logarithm on both sides of the inequality yields the desired result.
\end{proof}

\par By an argument similar to~\cite[Lemma 5.7]{Duan2022Convergence}, we derive the following lemma, which shows that the first-order derivative of a ReQU neural network can be represented by a ReQU-ReLU network. With the help of this lemma and Lemma~\ref{lemma:covering:number:sparse}, we can bound the metric entropy of the class of derivatives of ReQU networks.

\begin{lemma}\label{lemma:vcdim:derivative}
Let $f:\bbR^{d}\rightarrow\bbR$ be a ReQU neural network with depth no more that $L$ and the number of non-zero weights no more than $S$. Then $D_{k}f$ can be implemented by a ReQU-ReLU neural network with depth no more that $cL$ and the number of non-zero weights no more than $c^{\prime}LS$, where $c$ and $c^{\prime}$ are two positive absolute constants.
\end{lemma}

\begin{proof}[Proof of Lemma~\ref{lemma:vcdim:derivative}]
We prove this lemma by induction. For simplicity of presentation, we omit the intercept terms in this proof. Denote by $\varrho_{1}=\max\{0,x\}$ and $\varrho_{2}=(\max\{0,x\})^{2}$. It is straightforward to verify that
\begin{equation}\label{eq:proof:lemma:vcdim:derivative:1}
\varrho_{2}^{\prime}(z)=2\varrho_{1}(z),
\end{equation}
and 
\begin{multline}\label{eq:proof:lemma:vcdim:derivative:2}
yz=\frac{1}{4}\Big(\varrho_{2}(y+z)+\varrho_{2}(-y-z) \\
-\varrho_{2}(y-z)-\varrho_{2}(z-y)\Big). 
\end{multline}

\par For the two-layers ReQU sub-network, the $p$-th element can be defined as
\begin{equation*}
f_{p}^{(2)}(x):=\sum_{j\in[N_{2}]}a_{pj}^{(2)}\varrho_{2}\Big(\sum_{i\in[N_{1}]}a_{ji}^{(1)}x_{i}\Big).
\end{equation*}
The number of non-zero weights of $f_{p}^{(2)}$ is given by
\begin{equation*}
S_{2,p}:=\sum_{j\in[N_{2}]:a_{pj}^{(2)}\neq0}\|(a_{ji}^{(1)})_{i=1}^{N_{1}}\|_{0}.
\end{equation*}
By some simple calculation, we have that for each $1\leq k\leq d$,
\begin{align*}
D_{k}f_{p}^{(2)}(x)
&=\sum_{j\in[N_{2}]}a_{pj}^{(2)}D_{k}\varrho_{2}\Big(\sum_{i\in[N_{1}]}a_{ji}^{(1)}x_{i}\Big) \\
&=2\sum_{j\in[N_{2}]}a_{pj}^{(2)}\varrho_{1}\Big(\sum_{i\in[N_{1}]}a_{ji}^{(1)}x_{i}\Big)a_{jk}^{(1)}, 
\end{align*}
where the last equality holds from~\eqref{eq:proof:lemma:vcdim:derivative:1}. Thus $D_{k}f_{p}^{(2)}$ can be implemented by a ReLU network with 2 layers and the number of non-zero weights is same to $f_{p}^{(2)}$, that is, $S_{2,p}^{\prime,k}=S_{2,p}$.

\par For the three layers ReQU sub-network, by a same argument, we have
\begin{equation*}
f_{p}^{(3)}(x):=\sum_{j\in[N_{3}]}a_{pj}^{(3)}\varrho_{2}(f_{j}^{(2)}(x)),
\end{equation*}
the number of non-zeros weights of which is
\begin{equation*}
S_{3,p}:=\sum_{j\in[N_{3}]:a_{pj}^{(3)}\neq0}S_{2,j}.
\end{equation*}
Then its derivatives are given by
\begin{align*}
D_{k}f_{p}^{(3)}(x)&=\sum_{j\in[N_{3}]}a_{pj}^{(3)}D_{k}\varrho_{2}(f_{j}^{(2)}(x)) \\
&=2\sum_{j\in[N_{3}]}a_{pj}^{(3)}\varrho_{1}(f_{j}^{(2)}(x))D_{k}f_{j}^{(2)}(x) \\
&=\frac{1}{2}\sum_{j\in[N_{3}]}a_{pj}^{(3)}\Big\{\varrho_{2}\Big(\varrho_{1}(f_{j}^{(2)}(x))+D_{k}f_{j}^{(2)}(x)\Big) \\
&\quad+\varrho_{2}\Big(-\varrho_{1}(f_{j}^{(2)}(x))-D_{k}f_{j}^{(2)}(x)\Big) \\
&\quad-\varrho_{2}\Big(\varrho_{1}(f_{j}^{(2)}(x))-D_{k}f_{j}^{(2)}(x)\Big) \\
&\quad-\varrho_{2}\Big(-\varrho_{1}(f_{j}^{(2)}(x))+D_{k}f_{j}^{(2)}(x)\Big)\Big\},
\end{align*}
where the second equality holds from~\eqref{eq:proof:lemma:vcdim:derivative:1} and the last one is due to~\eqref{eq:proof:lemma:vcdim:derivative:2}. This implies that $D_{k}f_{p}^{(3)}$ can be implemented by a ReQU-ReLU mixed network with 4 layers. Furthermore, the number of non-zero weights of $D_{k}f_{p}^{(3)}$ is given by
\begin{align}
S_{3,p}^{\prime,k}
&:=\sum_{j\in[N_{3}]:a_{pj}^{(3)}\neq0}\Big(S_{2,j}+S_{2,j}^{\prime,k}+12\Big) \nonumber \\
&\leq\sum_{j\in[N_{3}]:a_{pj}^{(3)}\neq0}\Big(S_{2,j}+13S_{2,j}^{\prime,k}\Big) \nonumber \\
&=14\sum_{j\in[N_{3}]:a_{pj}^{(3)}\neq0}S_{2,j}=14S_{3,p}. \label{eq:proof:lemma:vcdim:derivative:3}
\end{align}

\par We claim that the depth of $D_{k}f_{p}^{(\ell-1)}$ is no more than $2\ell-2$ and the number of non-zero weights satisfies
\begin{equation}\label{eq:proof:lemma:vcdim:derivative:4}
S_{\ell,p}^{\prime,k}\leq13\ell S_{\ell,p}, \quad 3\leq \ell\leq L.
\end{equation}
The case of $\ell=3$ has be shown in~\eqref{eq:proof:lemma:vcdim:derivative:3}, and it remains to verify that this inequality also holds for $\ell$, provided that~\eqref{eq:proof:lemma:vcdim:derivative:4} holds for $\ell-1$.

\par According to~\eqref{eq:proof:lemma:vcdim:derivative:4}, suppose that $D_{k}f_{j}^{(\ell-1)}$ has $2(\ell-1)-2$ layers and no more than $13(\ell-1)S_{\ell-1,p}^{\prime,k}$ non-zero weights for $j\in[N_{\ell-1}]$. Notice the $p$-th element of the $\ell$-th layer are given by
\begin{equation*}
f_{p}^{(\ell)}(x):=\sum_{j\in[N_{\ell}]}a_{pj}^{(\ell)}\varrho_{2}(f_{j}^{(\ell-1)}(x)),
\end{equation*}
the number of non-zeros weights of which is
\begin{equation*}
S_{\ell,p}:=\sum_{j\in[N_{\ell}]:a_{pj}^{(\ell)}\neq0}S_{\ell-1,j}.
\end{equation*}
Then its derivatives are defined as
\begin{align*}
&D_{k}f_{p}^{(\ell)}(x) \\
&=\sum_{j\in[N_{\ell}]}a_{pj}^{(\ell)}D_{k}\varrho_{2}(f_{j}^{(\ell-1)}(x)) \\
&=2\sum_{j\in[N_{\ell}]}a_{pj}^{(\ell)}\varrho_{1}(f_{j}^{(\ell-1)}(x))D_{k}f_{j}^{(\ell-1)}(x) \\
&=\frac{1}{2}\sum_{j\in[N_{\ell}]}a_{pj}^{(\ell)}\Big\{\varrho_{2}\Big(\varrho_{1}(f_{j}^{(\ell-1)}(x))+D_{k}f_{j}^{(\ell-1)}(x)\Big) \\
&\quad+\varrho_{2}\Big(-\varrho_{1}(f_{j}^{(\ell-1)}(x))-D_{k}f_{j}^{(\ell-1)}(x)\Big) \\
&\quad-\varrho_{2}\Big(\varrho_{1}(f_{j}^{(\ell-1)}(x))-D_{k}f_{j}^{(\ell-1)}(x)\Big) \\
&\quad-\varrho_{2}\Big(-\varrho_{1}(f_{j}^{(\ell-1)}(x))+D_{k}f_{j}^{(\ell-1)}(x)\Big)\Big\}.
\end{align*}
Hence $D_{k}f_{p}^{(\ell)}$ has $2(\ell-1)-2+2$ layers and the number of non-zero weights of $D_{k}f_{p}^{(\ell)}$ is given by
\begin{align*}
S_{\ell,p}^{\prime,k}
&:=\sum_{j\in[N_{\ell}]:a_{pj}^{(\ell)}\neq0}\Big(S_{\ell-1,j}+S_{\ell-1,j}^{\prime,k}+12\Big) \\
&\leq\sum_{j\in[N_{\ell}]:a_{pj}^{(\ell)}\neq0}\Big(S_{\ell-1,j}+13(\ell-1)S_{\ell-1,j}+12S_{\ell-1,j}\Big) \\
&=13S_{\ell,p},
\end{align*}
which deduces~\eqref{eq:proof:lemma:vcdim:derivative:4} for $\ell$. Therefore, we complete the proof.
\end{proof}

\begin{remark}
Notice that both ReLU and ReQU are piecewise-polynomial activation functions. By the proof of Lemma~\ref{lemma:vcdim} in~\cite[Theorem 7]{Bartlett2019nearly}, it is apparent that the VC-dimension bounds also hold for ReQU-ReLU neural networks, which are constructed in Lemma~\ref{lemma:vcdim:derivative}. In addition, see~\cite[Theorem 5.1]{Duan2022Convergence} for a complete proof of the VC-dimension bound of ReQU-ReLU networks.
\end{remark}

\par Combining Lemmas~\ref{lemma:covering:number:sparse} and~\ref{lemma:vcdim:derivative} yields the following results.
\begin{lemma}\label{lemma:covering:number:sparse:grad}
Let $\calN\subseteq\calN(L,W,S)$ be a set of deep neural networks from $\Omega$ to the bounded interval $[-B,B]$. The activation function is piecewise-polynomial. Let $\delta\in(0,1)$ and $\euD=\{X_{i}\}_{i=1}^{n}\subseteq\Omega$. Then 
\begin{equation*}
\log N(\delta,D_{k}\calN,L^{2}(\euD))\leq cL^{2}S\log(S)\log\Big(\frac{nB}{\delta}\Big),
\end{equation*}
where $c>0$ is an absolute constant.
\end{lemma}

\par We conclude this section by introducing an approximation error bound for deep ReQU neural networks.

\begin{lemma}[Approximation error]\label{lemma:approximation:requ}
Let $\Omega\subseteq K\subseteq\bbR^{d}$ be two bounded domain. For each $\phi\in C^{s}(K)$ with $s\in\bbN_{\geq1}$, there exists a ReQU neural network $f$ with the depth and the number of nonzero weights no more than $d\lfloor\log_{2}N\rfloor+d$ and $C^{\prime}N^{d}$, respectively, such that $0\leq k\leq\min\{s,N\}$,
\begin{equation*}
\inf_{f\in\calF}\|f-\phi\|_{C^{k}(\Omega)}\leq CN^{-(s-k)}\|\phi\|_{C^{s}(K)},
\end{equation*}
where $C$ and $C^{\prime}$ are constants independent of $N$.
\end{lemma}

\begin{proof}[Proof of Lemma~\ref{lemma:approximation:requ}]
We first approximate the target function $\phi\in C^{s}(K)$ by polynomials. According to~\cite[Theorem 2]{Bagby2002Multivariate}, for each $N\in\bbN$, these exists a polynomial $p_{N}$ of degree at most $N$ on $\bbR^{d}$ such that for $0\leq|\gamma|\leq\min\{s,N\}$,
\begin{multline}\label{eq:proof:lemma:approximation:requ:1}
\sup_{x\in K}|D^{\gamma}(\phi(x)-p_{N}(x))| \\
\leq\frac{C}{N^{s-|\gamma|}}\sum_{|\alpha|\leq s}\sup_{x\in K}|D^{\alpha}\phi(x)|,
\end{multline}
where $C$ is a positive constant depending only on $d$, $s$ and $K$. Applying~\cite[Theorem 3.1]{Li2019Better}, one obtains that there exists a ReQU neural network $f$ with the depth $d\lfloor\log_{2}N\rfloor+d$ and nonzero weights no more than $C^{\prime}N^{d}$, such that
\begin{equation}\label{eq:proof:lemma:approximation:requ:2}
f=p_{N},
\end{equation}
where $C^{\prime}$ is a constant independent of $N$. Combining~\eqref{eq:proof:lemma:approximation:requ:1} and~\ref{eq:proof:lemma:approximation:requ:2} yields
\begin{align*}
\|f-\phi\|_{C^{k}(\Omega)}
&\leq\sup_{x\in K}|D^{\gamma}(f(x)-p_{N}(x))| \\
&\leq CN^{-(s-k)}\|\phi\|_{C^{s}(K)},
\end{align*}
for each $0\leq k\leq\min\{s,N\}$. This completes the proof.
\end{proof}

\section{Proofs of Results in Section~\ref{section:regularized:estimator}}

\par Proofs of theoretical results in Section~\ref{section:regularized:estimator} are shown in this section.

\begin{proof}[Proof of Lemma~\ref{lemma:existence:prm}]
By~\eqref{eq:population:risk:regularization} and the standard variational theory~\cite{evans2010partial}, it is sufficient to focus on the variational problem
\begin{equation}\label{eq:proof:lemma:existence:prm:1}
\scrB(f^{\lambda},g)=(f_{0},g)_{L^{2}(\mu_{X})}, \quad \forall g\in H^{1}(\nu_{X}),
\end{equation}
where the bilinear form $\scrB:H^{1}(\nu_{X})\times H^{1}(\nu_{X})\rightarrow\bbR$ is defined as
\begin{equation*}
\scrB(f,g):=\lambda(\nabla f,\nabla g)_{L^{2}(\nu_{X})}+(f,g)_{L^{2}(\mu_{X})}.
\end{equation*}
It is straightforward to verify the boundedness and coercivity of the bilinear form from Assumption~\ref{assumption:bounded:density:ratio}, that is,
\begin{align*}
|\scrB(f,g)|&\leq(\lambda\vee\zeta^{-1/2})\|f\|_{H^{1}(\nu_{X})}\|g\|_{H^{1}(\nu_{X})}, \\
\scrB(f,f)&\geq(\lambda\wedge\kappa^{-1/2})\|f\|_{H^{1}(\nu_{X})},
\end{align*}
for each $f,g\in H^{1}(\nu_{X})$. Further, since that $f_{0}\in L^{2}(\mu_{X})$, the functional $F:H\rightarrow\bbR,~g\mapsto(f_{0},g)_{L^{2}(\mu_{X})}$ is bounded and linear. Then according to Lax-Milgram theorem~\cite[Theorem 1 in Chapter 6.2]{evans2010partial}, there exists a unique solution $f^{\lambda}\in H^{1}(\nu_{X})$ to the varitional problem~\eqref{eq:proof:lemma:existence:prm:1}. This completes the proof of the uniqueness. See~\cite[Theorem 2.4.2.7]{Grisvard2011Elliptic} for the proof of the higher regularity of the solution.
\end{proof}

\section{Proofs of Results in Section~\ref{section:analysis:regression}}

\par In this section, we demonstrate proofs of theoretical results in Section~\ref{section:analysis:regression}, including Lemma~\ref{lemma:oracle:erm:regression}, Lemma~\ref{lemma:approximation:grad:bounded} and Theorem~\ref{theorem:erm:rate:regression}. The proof of Lemma~\ref{lemma:oracle:erm:regression} uses the technique of offset Rademacher complexity, which has been investigated by~\cite{{Liang2015Learning}}.

\begin{proof}[Proof of Lemma~\ref{lemma:oracle:erm:regression}]
Recall the population excess risk $R(f)$ and the empirical excess risk $\what{R}_{\euD}(f)$ defined in the proof of Lemma~\ref{lemma:oracle:erm}. We further define the regularized excess risk and regularized empirical risk as
\begin{align*}
R^{\lambda}(f)&:=R(f)+\lambda\|\nabla f\|_{L^{2}(\nu_{X})}^{2}, \\
\what{R}_{\euD}^{\lambda}(f)&:=\what{R}_{\euD}(f)+\lambda\|\nabla f\|_{L^{2}(\nu_{X})}^{2}.
\end{align*}
It suffices to shown that
\begin{multline}\label{eq:proof:theorem:empirical:minimizer:regression:0}
\bbE_{\euD}\Big[R^{\lambda}(\what{f}_{\euD}^{\lambda})\Big]\lesssim\inf_{f\in\calF}R^{\lambda}(f) \\
+\frac{B_{0}^{2}+\sigma^{2}}{\log^{-1}n}\inf_{\delta>0}\Big\{\frac{2\log N(B_{0}\delta,\calF,L^{2}(\euD))}{n}+\delta\Big\}.
\end{multline}
Before proceeding, we provide the proof sketch. Firstly, in \emph{Step (I)}, we show that 
\begin{equation}\label{eq:proof:theorem:empirical:minimizer:regression:1}
\begin{aligned}
&\bbE_{\euD}\Big[R^{\lambda}(\what{f}_{\euD}^{\lambda})-2\what{R}_{\euD}^{\lambda}(\what{f}_{\euD}^{\lambda})\Big] \\
&=\bbE_{\euD}\Big[R(\what{f}_{\euD}^{\lambda})-2\what{R}_{\euD}(\what{f}_{\euD}^{\lambda})\Big] \\
&\leq cB_{0}^{2}\inf_{\delta>0}\Big\{\frac{2\log N(B_{0}\delta,\calF,L^{2}(\euD))}{n}+\delta\Big\},
\end{aligned}
\end{equation}
where $c$ is an absolute positive constant. It remains to consider the regularized empirical risk. According to~\eqref{eq:nonparametric:regression}, we have
\begin{align*}
&\what{L}_{\euD}^{\lambda}(\what{f}_{\euD}^{\lambda}) \\
&=\what{R}_{\euD}^{\lambda}(\what{f}_{\euD}^{\lambda})-\frac{2}{n}\sum_{i=1}^{n}\xi_{i}(\what{f}_{\euD}^{\lambda}(X_{i})-f_{0}(X_{i}))+\bbE\Big[\frac{1}{n}\sum_{i=1}^{n}\xi_{i}^{2}\Big].
\end{align*}
Taking expectation with respect to $\euD\sim\mu^{n}$ on both sides of the equality yields that for each $f\in\calF$,
\begin{align*}
&\bbE_{\euD}\Big[\what{R}_{\euD}^{\lambda}(\what{f}_{\euD}^{\lambda})\Big] \\
&=\bbE_{\euD}\Big[\what{L}_{\euD}^{\lambda}(f)\Big]+2\bbE_{\euD}\Big[\frac{1}{n}\sum_{i=1}^{n}\xi_{i}\what{f}_{\euD}^{\lambda}(X_{i})\Big]-\bbE\Big[\frac{1}{n}\sum_{i=1}^{n}\xi_{i}^{2}\Big] \\
&\leq R^{\lambda}(f)+\frac{1}{2}\bbE_{\euD}\Big[\what{R}_{\euD}(\what{f}_{\euD}^{\lambda})\Big] \\
&\quad+c(B_{0}^{2}+\sigma^{2})\inf_{\delta>0}\Big\{\frac{2\log N(B_{0}\delta,\calF,L^{2}(\euD))}{n}+\delta\Big\},
\end{align*}
which implies
\begin{multline}\label{eq:proof:theorem:empirical:minimizer:regression:2}
\what{R}_{\euD}^{\lambda}(\what{f}_{\euD}^{\lambda})\leq2R^{\lambda}(f) \\
+2c(B_{0}^{2}+\sigma^{2})\inf_{\delta>0}\Big\{\frac{2\log N(B_{0}\delta,\calF,L^{2}(\euD))}{n}+\delta\Big\},
\end{multline}
where $c$ is an absolute positive constant. Here the inequality invokes
\begin{multline}\label{eq:proof:theorem:empirical:minimizer:regression:5}
\bbE_{\euD}\Big[\frac{1}{n}\sum_{i=1}^{n}\xi_{i}\what{f}_{\euD}^{\lambda}(X_{i})\Big]
\leq\frac{1}{4}\bbE_{\euD}\Big[\what{R}_{\euD}(\what{f}_{\euD}^{\lambda})\Big] \\
+\frac{8\sigma^{2}}{n}\inf_{\delta>0}\Big\{\frac{\log N(B_{0}\delta,\calF,L^{2}(\euD))}{n}+\delta\Big\} \\
+2(B_{0}^{2}+\sigma^{2})\delta,
\end{multline}
which is obtained in \emph{Step (II)}. Combining~\eqref{eq:proof:theorem:empirical:minimizer:regression:1} and~\eqref{eq:proof:theorem:empirical:minimizer:regression:2} obtains~\eqref{eq:proof:theorem:empirical:minimizer:regression:0}.

\par\noindent\emph{Step (I).}
Given a ghost sample $\euD^{\prime}=\{(X_{i}^{\prime},Y_{i}^{\prime})\}_{i=1}^{n}$, where $\{X_{i}^{\prime}\}_{i=1}^{n}$ are independently drawn from $\mu_{X}$. Further, the ghost sample $\euD^{\prime}$ is independent of $\euD=\{(X_{i},Y_{i})\}_{i=1}^{n}$. Let $\varepsilon=\{\varepsilon_{i}\}_{i=1}^{n}$ be a set of Rademacher variables and independent of $\euD$ and $\euD^{\prime}$. Since that $\what{f}_{\euD}^{\lambda}\in\calF$, by the technique of symmetrization, we have
\begin{align}
&\bbE_{\euD}\Big[R(\what{f}_{\euD}^{\lambda})-2\what{R}_{\euD}(\what{f}_{\euD}^{\lambda})\Big]\leq\bbE_{\euD}\Big[\sup_{f\in\calF}R(f)-2\what{R}_{\euD}(f)\Big] \nonumber \\
&=\bbE_{\euD}\Big[\sup_{f\in\calF}\bbE_{\euD^{\prime}}\Big[\frac{1}{n}\sum_{i=1}^{n}(f(X^{\prime}_{i})-f_{0}(X^{\prime}_{i}))^{2}\Big] \nonumber \\
&\quad-\frac{2}{n}\sum_{i=1}^{n}(f(X_{i})-f_{0}(X_{i}))^{2}\Big] \nonumber \\
&\leq\bbE_{\euD}\bbE_{\euD^{\prime}}\Big[\sup_{f\in\calF}\frac{1}{n}\sum_{i=1}^{n}(f(X^{\prime}_{i})-f_{0}(X^{\prime}_{i}))^{2} \nonumber \\
&\quad-\frac{2}{n}\sum_{i=1}^{n}(f(X_{i})-f_{0}(X_{i}))^{2}\Big] \nonumber \\
&=\bbE_{\euD}\bbE_{\euD^{\prime}}\Big[\sup_{f\in\calF}\frac{3}{2n}\sum_{i=1}^{n}(f(X^{\prime}_{i})-f_{0}(X^{\prime}_{i}))^{2} \nonumber \\
&\quad-\frac{1}{2n}\sum_{i=1}^{n}(f(X^{\prime}_{i})-f_{0}(X^{\prime}_{i}))^{2} \nonumber \\
&\quad-\frac{3}{2n}\sum_{i=1}^{n}(f(X_{i})-f_{0}(X_{i}))^{2} \nonumber \\
&\quad-\frac{1}{2n}\sum_{i=1}^{n}(f(X_{i})-f_{0}(X_{i}))^{2}\Big] \nonumber \\
&\leq\bbE_{\euD}\bbE_{\euD^{\prime}}\Big[\sup_{f\in\calF}\frac{3}{2n}\sum_{i=1}^{n}(f(X^{\prime}_{i})-f_{0}(X^{\prime}_{i}))^{2} \nonumber \\
&\quad-\frac{1}{8B_{0}^{2}n}\sum_{i=1}^{n}(f(X^{\prime}_{i})-f_{0}(X^{\prime}_{i}))^{4} \nonumber \\
&\quad-\frac{3}{2n}\sum_{i=1}^{n}(f(X_{i})-f_{0}(X_{i}))^{2} \nonumber \\
&\quad-\frac{1}{8B_{0}^{2}n}\sum_{i=1}^{n}(f(X_{i})-f_{0}(X_{i}))^{4}\Big] \nonumber \\
&=\bbE_{\euD}\bbE_{\varepsilon}\Big[\sup_{f\in\calF}\frac{3}{n}\sum_{i=1}^{n}\varepsilon_{i}(f(X_{i})-f_{0}(X_{i}))^{2} \nonumber \\
&\quad-\frac{1}{4B_{0}^{2}n}\sum_{i=1}^{n}(f(X_{i})-f_{0}(X_{i}))^{4}\Big], \label{eq:proof:lemma:oracle:erm:regression:1:1}
\end{align}
where the second inequality follows from the convexity of supremum and Jensen's inequality, and the third inequality is owing to the fact that $0\leq(f(X_{i})-f_{0}(X_{i}))^{2}\leq4B_{0}^{2}$ for each $f\in\calF$.

\par Let $\delta>0$ and let $\calF_{\delta}$ be an $L^{2}(\euD)$ $(B_{0}\delta)$-cover of $\calF$ satisfying $|\calF_{\delta}|=N(B_{0}\delta,\calF,L^{2}(\euD))$. Then it follows from Cauchy-Schwarz inequality that for each $f\in\calF$, there exists $f_{\delta}\in\calF_{\delta}$ such that
\begin{align*}
&\frac{1}{n}\sum_{i=1}^{n}\varepsilon_{i}(f(X_{i})-f_{0}(X_{i}))^{2}-\frac{1}{n}\sum_{i=1}^{n}\varepsilon_{i}(f_{\delta}(X_{i})-f_{0}(X_{i}))^{2}, \\
&\leq\Big(\frac{1}{n}\sum_{i=1}^{n}(f(X_{i})+f_{\delta}(X_{i})-2f_{0}(X_{i}))^{2} \\
&\quad\times(f(X_{i})-f_{\delta}(X_{i}))^{2}\Big)^{1/2}\Big(\frac{1}{n}\sum_{i=1}^{n}\varepsilon_{i}^{2}\Big)^{1/2} \\
&\quad\leq 4B_{0}^{2}\delta.
\end{align*}
By a same argument, we obtain
\begin{align*}
&\frac{1}{B_{0}^{2}n}\Big(-\sum_{i=1}^{n}(f(X_{i})-f_{0}(X_{i}))^{4}+\sum_{i=1}^{n}(f_{\delta}(X_{i})-f_{0}(X_{i}))^{4}\Big) \\
&\leq 32B_{0}^{2}\delta.
\end{align*}
Combining \eqref{eq:proof:theorem:empirical:minimizer:1:1} with above two inequalities yields
\begin{equation}\label{eq:proof:lemma:oracle:erm:regression:1:2}
\begin{aligned}
&\bbE_{\euD}\Big[R(\what{f}_{\euD}^{\lambda})-2\what{R}_{\euD}(\what{f}_{\euD}^{\lambda})\Big]-20B_{0}^{2}\delta\\
&\leq\bbE_{\euD}\bbE_{\varepsilon}\Big[\max_{f\in\calF_{\delta}}\frac{3}{n}\sum_{i=1}^{n}\varepsilon_{i}(f(X_{i})-f_{0}(X_{i}))^{2} \\
&\quad-\frac{1}{4B_{0}^{2}n}\sum_{i=1}^{n}(f(X_{i})-f_{0}(X_{i}))^{4}\Big].
\end{aligned}
\end{equation}

\par In order to estimate the expectation in \eqref{eq:proof:lemma:oracle:erm:regression:1:2}, we consider the following probability conditioning on $\euD=\{(X_{i},Y_{i})\}_{i=1}^{n}$
\begin{multline*}
\pr_{\varepsilon}\Big\{\frac{3}{n}\sum_{i=1}^{n}\varepsilon_{i}(f(X_{i})-f_{0}(X_{i}))^{2} \\
>t+\frac{1}{4B_{0}^{2}n}\sum_{i=1}^{n}(f(X_{i})-f_{0}(X_{i}))^{4}\Big\}.
\end{multline*}
For a fixed sample $\euD=\{(X_{i},Y_{i})\}_{i=1}^{n}$, the random variables $\{\varepsilon_{i}(f(X_{i})-f_{0}(X_{i}))^{2}\}_{i=1}^{n}$ are independent and satisfy
\begin{equation*}
\bbE_{\varepsilon}\big[\varepsilon_{i}(f(X_{i})-f_{0}(X_{i}))^{2}\big]=0,
\end{equation*}
and for each $1\leq i\leq n$,
\begin{align*}
-(f(X_{i})-f_{0}(X_{i}))^{2}
&\leq\varepsilon_{i}(f(X_{i})-f_{0}(X_{i}))^{2} \\
&\leq(f(X_{i})-f_{0}(X_{i}))^{2}.
\end{align*}
Consequently, it follows from Hoeffding's inequality \cite[Lemma D.2]{mohri2018foundations} that
\begin{align*}
&\pr_{\varepsilon}\Big\{\frac{3}{n}\sum_{i=1}^{n}\varepsilon_{i}(f(X_{i})-f_{0}(X_{i}))^{2} \\
&\quad>t+\frac{1}{4B_{0}^{2}n}\sum_{i=1}^{n}(f(X_{i})-f_{0}(X_{i}))^{4}\Big\} \\
&\leq\exp\Bigg(-\frac{(\frac{nt}{3}+\frac{1}{12B_{0}^{2}}\sum_{i=1}^{n}(f(X_{i})-f_{0}(X_{i}))^{4})^{2}}{2\sum_{i=1}^{n}(f(X_{i})-f_{0}(X_{i}))^{4}}\Bigg) \\
&\leq\exp\Big(-\frac{nt}{18B_{0}^{2}}\Big),
\end{align*}
where we used the numeric inequality that $(a+y)^{2}/y\geq4a$ for each $a>0$. Then with the aid of the above estimate of the tail probability, it follows that
\begin{align*}
&\bbE_{\varepsilon}\Big[\max_{f\in\calF_{\delta}}\frac{3}{n}\sum_{i=1}^{n}\varepsilon_{i}(f(X_{i})-f_{0}(X_{i}))^{2} \\
&\quad-\frac{1}{4B_{0}^{2}n}\sum_{i=1}^{n}(f(X_{i})-f_{0}(X_{i}))^{4}\Big] \\
&\leq T+N(B_{0}\delta,\calF,L^{2}(\euD))\int_{T}^{\infty}\exp\Big(-\frac{nt}{18B_{0}^{2}}\Big)dt \\
&=T+\frac{18B_{0}^{2}}{n}N(B_{0}\delta,\calF,L^{2}(\euD))\exp\Big(-\frac{nT}{18B_{0}^{2}}\Big).
\end{align*}
By setting $T=18B_{0}^{2}\log N(B_{0}\delta,\calF,L^{2}(\euD))n^{-1}$, we deduces
\begin{equation}\label{eq:proof:lemma:oracle:erm:regression:1:3}
\begin{aligned}
&\bbE_{\varepsilon}\Big[\max_{f\in\calF_{\delta}}\frac{3}{n}\sum_{i=1}^{n}\varepsilon_{i}(f(X_{i})-f_{0}(X_{i}))^{2} \\
&\quad-\frac{1}{4B_{0}^{2}n}\sum_{i=1}^{n}(f(X_{i})-f_{0}(X_{i}))^{4}\Big] \\
&\leq\frac{18B_{0}^{2}}{n}(1+\log N(B_{0}\delta,\calF,L^{2}(\euD))).
\end{aligned}
\end{equation}
Combining \eqref{eq:proof:lemma:oracle:erm:regression:1:2} and \eqref{eq:proof:lemma:oracle:erm:regression:1:3} implies that
\begin{equation}\label{eq:proof:lemma:oracle:erm:regression:1:4}
\begin{aligned}
&\bbE_{\euD}\Big[R(\what{f}_{\euD}^{\lambda})-2\what{R}_{\euD}(\what{f}_{\euD}^{\lambda})\Big] \\
&\leq \frac{18B_{0}^{2}}{n}(1+\log N(B_{0}\delta,\calF,L^{2}(\euD)))+20B_{0}^{2}\delta.
\end{aligned}
\end{equation}
This completes the proof of~\eqref{eq:proof:theorem:empirical:minimizer:regression:1}.

\par\noindent\emph{Step (II).} 
Recall the $L^{2}(\euD)$ $(B_{0}\delta)$-cover $\calF_{\delta}$ of the hypothesis class $\calF$. There exists $f_{\delta}\in\calF_{\delta}$ such that 
\begin{equation*}
\frac{1}{n}\sum_{i=1}^{n}|f_{\delta}(X_{i})-\what{f}_{\euD}^{\lambda}(X_{i})|^{2}\leq(B_{0}\delta)^{2}, 
\end{equation*}
which implies
\begin{equation}\label{eq:proof:lemma:oracle:erm:regression:2:1}
\begin{aligned}
&\bbE_{\euD}\Big[\frac{1}{n}\sum_{i=1}^{n}\xi_{i}(\what{f}_{\euD}^{\lambda}(X_{i})-f_{\delta}(X_{i}))\Big] \\
&\leq\bbE_{\euD}^{1/2}\Big[\frac{1}{n}\sum_{i=1}^{n}\xi_{i}^{2}\Big]\bbE_{\euD}^{1/2}\Big[\frac{1}{n}\sum_{i=1}^{n}(\what{f}_{\euD}^{\lambda}(X_{i})-f_{\delta}(X_{i}))^{2}\Big] \\
&\leq B_{0}\sigma\delta, 
\end{aligned}
\end{equation}
and
\begin{equation}\label{eq:proof:lemma:oracle:erm:regression:2:2}
\begin{aligned}
&\what{R}_{\euD}^{1/2}(f_{\delta}) \\
&\leq\Big(\frac{1}{n}\sum_{i=1}^{n}(f_{\delta}(X_{i})-\what{f}_{\euD}^{\lambda}(X_{i}))^{2}\Big)^{1/2}+\what{R}_{\euD}^{1/2}(\what{f}_{\euD}^{\lambda}) \\
&\leq B_{0}\delta+\what{R}_{\euD}^{1/2}(\what{f}_{\euD}^{\lambda}),
\end{aligned}
\end{equation}
where we used Cauchy-Schwarz inequality and Assumption \ref{assumption:subGaussian}. Consequently, we have
\begin{align}
&\bbE_{\euD}\Big[\frac{1}{n}\sum_{i=1}^{n}\xi_{i}\what{f}_{\euD}^{\lambda}(X_{i})\Big] \nonumber \\
&=\bbE_{\euD}\Big[\frac{1}{n}\sum_{i=1}^{n}\xi_{i}(\what{f}_{\euD}^{\lambda}(X_{i})-f_{0}(X_{i}))\Big] \nonumber \\
&\leq\bbE_{\euD}\Big[\frac{1}{n}\sum_{i=1}^{n}\xi_{i}(f_{\delta}(X_{i})-f_{0}(X_{i}))\Big]+B_{0}\sigma\delta \nonumber \\
&\leq\bbE_{\euD}\Big[\frac{\what{R}_{\euD}^{1/2}(\what{f}_{\euD}^{\lambda})+B_{0}\delta}{\sqrt{n}}\psi(f_{\delta})\Big]+B_{0}\sigma\delta \nonumber \\
&\leq\Big(\bbE_{\euD}^{1/2}\Big[\what{R}_{\euD}(\what{f}_{\euD}^{\lambda})\Big]+B_{0}\delta\Big)\frac{1}{\sqrt{n}}\bbE_{\euD}^{1/2}\Big[\psi^{2}(f_{\delta})\Big]+B_{0}\sigma\delta \nonumber \\
&\leq\frac{1}{4}\bbE_{\euD}\Big[\what{R}_{\euD}(\what{f}_{\euD}^{\lambda})\Big]+\frac{2}{n}\bbE_{\euD}\Big[\psi^{2}(f_{\delta})\Big] \nonumber \\
&\quad+\frac{1}{4}B_{0}^{2}\delta^{2}+B_{0}\sigma\delta. \label{eq:proof:theorem:empirical:minimizer:2:3}
\end{align}
Here, the first inequality holds from \eqref{eq:proof:lemma:oracle:erm:regression:2:1}, the second inequality is from \eqref{eq:proof:lemma:oracle:erm:regression:2:2}, where
\begin{equation*}
\psi(f_{\delta}):=\frac{\sum_{i=1}^{n}\xi_{i}(f_{\delta}(X_{i})-f_{0}(X_{i}))}{\sqrt{n}\what{R}_{\euD}^{1/2}(f_{\delta})}.
\end{equation*}
The third inequality follows from Cauchy-Schwarz inequality, while the last one is owing to the inequality $ab\leq a^{2}/4+b^{2}$ for $a,b>0$. Observe that for each fixed $f$ independent of $\xi$, the random variable $\psi(f)$ is sub-Gaussian with variance proxy $\sigma^{2}$.
Then using Lemma \ref{lemma:subgaussian:max} gives that
\begin{equation}\label{eq:proof:theorem:empirical:minimizer:2:4}
\bbE_{\xi}\Big[\psi^{2}(f_{\delta})\Big]\leq\bbE_{\xi}\Big[\max_{f\in\calF_{\delta}}\psi^{2}(f)\Big]\leq4\sigma^{2}(\log|\calF_{\delta}|+1).
\end{equation}
Combining~\eqref{eq:proof:theorem:empirical:minimizer:2:3} and~\eqref{eq:proof:theorem:empirical:minimizer:2:4} yields~\eqref{eq:proof:theorem:empirical:minimizer:regression:5}.
\end{proof}

\begin{proof}[Proof of Lemma~\ref{lemma:approximation:grad:bounded}]
Using Lemma~\ref{lemma:approximation:requ}, by setting $k=0$, we obtain the estimate in $L^{2}(\mu_{X})$-norm. Further, setting $k=1$ yields the estimate for the first-order derivative. This completes the proof.
\end{proof}

\begin{proof}[Proof of Theorem~\ref{theorem:erm:rate:regression}]
According to Lemma 4.2, we set the hypothesis class $\calF$ as ReQU neural networks $\calF=\calN(L,S)$ with $L=\calO(\log N)$ and $S=\calO(N^{d})$. Then there exists $f\in\calF$ such that $\|f-\phi\|_{L^{2}(\mu_{X})}\leq CN^{-s}$ and $\|\nabla f\|_{L^{2}(\nu_{X})}\leq C$. By using Lemma~\ref{lemma:covering:number:sparse} and set $\delta=1/n$, we find 
\begin{multline*}
\log N(B_{0}n^{-1},\calF,L^{2}(\euD)) \\
\lesssim LS\log S\log n\lesssim N^{d}\log^{2}N\log n.
\end{multline*}
Substituting these estimates into Lemma~\ref{lemma:oracle:erm:regression} yields
\begin{multline*}
\bbE_{\euD}\Big[\|\what{f}_{\euD}^{\lambda}-f_{0}\|_{L^{2}(\mu_{X})}^{2}\Big] \\
\lesssim CN^{-2s}+C\lambda+C\log n\frac{N^{d}\log^{2}N}{n\log^{-1}(n)}.
\end{multline*}
Letting $N=\calO(n^{\frac{1}{d+2s}})$ and $\lambda=\calO(n^{-\frac{2s}{d+2s}}\log^{3}n)$ deduces the desired result.
\end{proof}

\section{Proofs of Results in Section~\ref{section:analysis:derivative}}

\par In this section, we show proofs of theoretical results in Section~\ref{section:analysis:derivative}. The proofs for the deep Sobolev regressor are shown in Section~\ref{section:appendix:DORE}, and proofs for semi-supervised deep Sobolev regressor are shown in Section~\ref{section:appendix:SDORE}.

\begin{proof}[Proof of Lemma~\ref{lemma:rate:population:minimizer}]
It follows from~\eqref{eq:proof:lemma:existence:prm:1} that
\begin{align*}
&\lambda(\nabla(f^{\lambda}-f_{0}),\nabla h)_{L^{2}(\nu_{X})}+(f^{\lambda}-f_{0},h)_{L^{2}(\mu_{X})} \\
&=\lambda(\Delta f_{0}+\nabla f_{0}\cdot\nabla(\log q),h)_{L^{2}(\nu_{X})} , \quad \forall h\in H^{1}(\nu_{X}),
\end{align*}
where we used Lemma~\ref{lemma:Green} and Assumption~\ref{assumption:regularity:f0}. By setting $h=f^{\lambda}-f_{0}\in H^{1}(\nu_{X})$ and using Cauchy-Schwarz inequality, we derive
\begin{align}
&\lambda\|\nabla(f^{\lambda}-f_{0})\|_{L^{2}(\nu_{X})}^{2}+\|f^{\lambda}-f_{0}\|_{L^{2}(\mu_{X})}^{2} \nonumber \\
&\leq\lambda\|\Delta f_{0}+\nabla f_{0}\cdot\nabla(\log q)\|_{L^{2}(\nu_{X})}\|f^{\lambda}-f_{0}\|_{L^{2}(\nu_{X})} \nonumber \\
&\leq\lambda\kappa^{1/2}\|\Delta f_{0}+\nabla f_{0}\cdot\nabla(\log q)\|_{L^{2}(\nu_{X})} \nonumber \\
&\quad\times\|f^{\lambda}-f_{0}\|_{L^{2}(\mu_{X})} \label{eq:proof:lemma:rate:population:minimizer:1}
\end{align}
which implies immediately
\begin{equation}\label{eq:proof:lemma:rate:population:minimizer:2}
\|f^{\lambda}-f_{0}\|_{L^{2}(\mu_{X})}\leq\lambda\kappa^{1/2}\|\Delta f_{0}+\nabla f_{0}\cdot\nabla(\log q)\|_{L^{2}(\nu_{X})}.
\end{equation}
Substituting~\eqref{eq:proof:lemma:rate:population:minimizer:2} into~\eqref{eq:proof:lemma:rate:population:minimizer:1} deduces the estimate for the derivative, which completes the proof.
\end{proof}

\subsection{Deep Sobolev Regressor}\label{section:appendix:DORE}

\begin{proof}[Proof of Lemma~\ref{lemma:oracle:erm}]
For simplicity of notation, we define the empirical inner-product and norm based on the sample $\euD=\{(X_{i},Y_{i})\}_{i=1}^{n}$, respectively, as
\begin{align*}
(u,v)_{L^{2}(\euD)}&=\frac{1}{n}\sum_{i=1}^{n}u(X_{i})v(X_{i}), \\
\|u\|_{L^{2}(\euD)}^{2}&=\frac{1}{n}\sum_{i=1}^{n}u^{2}(X_{i}),
\end{align*}
for each $u,v\in L^{\infty}(\mu_{X})$. Then we define the excess risk and its empirical counterpart, respectively, as
\begin{equation*}
R(f)=\|f-f_{0}\|_{L^{2}(\mu_{X})}^{2}
\quad\text{and}\quad
\what{R}_{\euD}(f)=\|f-f_{0}\|_{L^{2}(\euD)}^{2}.
\end{equation*}
The proof is divided into five parts which are denoted by (I) to (V):
\begin{enumerate}
\item[(I)] We first relate the excess risk with its empirical counterpart:
\begin{equation}\label{eq:proof:theorem:empirical:minimizer:1}
\begin{aligned}
&\bbE_{\euD}\Big[R(\what{f}_{\euD}^{\lambda})-\what{R}_{\euD}(\what{f}_{\euD}^{\lambda})\Big] \\
&\leq 4B_{0}^{2}\Big\{\Big(\frac{2\log N(B_{0}\delta,\calF,L^{2}(\euD))}{n}\Big)^{\frac{1}{2}}+\delta\Big\}.
\end{aligned}
\end{equation}
\item[(II)] We next derive the following inequality for preparation:
\begin{equation}\label{eq:proof:theorem:empirical:minimizer:2}
\begin{aligned}
&\bbE_{\euD}\Big[\frac{1}{n}\sum_{i=1}^{n}\xi_{i}\what{f}_{\euD}^{\lambda}(X_{i})\Big] \\
&\leq \frac{B_{0}^{2}+\sigma^{2}}{\log^{-1}n}\Big\{\Big(\frac{2\log N(B_{0}\delta,\calF,L^{2}(\euD))}{n}\Big)^{\frac{1}{2}}+\delta\Big\}.
\end{aligned}
\end{equation}
\item[(III)] With the help of variational inequality, we obtain the following inequality:
\begin{equation}\label{eq:proof:theorem:empirical:minimizer:3}
\begin{aligned}
&\lambda\bbE_{\euD}\Big[\|\nabla(\what{f}_{\euD}^{\lambda}-f_{0})\|_{L^{2}(\nu_{X})}^{2}\Big]+\bbE_{\euD}\Big[\what{R}_{\euD}(\what{f}_{\euD}^{\lambda})\Big] \\
&\leq\frac{1}{8}\bbE_{\euD}\Big[R(\what{f}_{\euD}^{\lambda})\Big]+2\bbE_{\euD}\Big[\frac{1}{n}\sum_{i=1}^{n}\xi_{i}\what{f}_{\euD}^{\lambda}(X_{i})\Big] \\
&\quad+c\Big\{\beta\lambda^{2}+\varepsilon_{\app}(\calF,\lambda)\Big\},
\end{aligned}
\end{equation}
where $c$ is an absolute positive constant. Here the constant $\beta$ and the approximation error $\varepsilon_{\app}(\calF,\lambda)$ is defined as
\begin{align*}
&\beta=\kappa\Big\{\|\Delta f_{0}\|_{L^{2}(\nu_{X})}^{2}+\|\nabla f_{0}\cdot\nabla(\log q)\|_{L^{2}(\nu_{X})}^{2}\Big\}, \\
&\varepsilon_{\app}(\calF,\lambda) \\
&=\inf_{f\in\calF}\Big\{\|f-f_{0}\|_{L^{2}(\mu_{X})}^{2}+\lambda\|\nabla(f-f_{0})\|_{L^{2}(\nu_{X})}^{2}\Big\}.
\end{align*}

\item[(IV)] Combining~\eqref{eq:proof:theorem:empirical:minimizer:1},~\eqref{eq:proof:theorem:empirical:minimizer:2} and~\eqref{eq:proof:theorem:empirical:minimizer:3}, we obtain an estimate for $L^{2}(\mu_{X})$-error:
\begin{equation}\label{eq:proof:theorem:empirical:minimizer:4}
\begin{aligned}
&\bbE_{\euD}\Big[\|\what{f}_{\euD}^{\lambda}-f_{0}\|_{L^{2}(\mu_{X})}^{2}\Big] \\
&\lesssim\beta\lambda^{2}+\varepsilon_{\app}(\calF,\lambda)+\varepsilon_{\gen}(\calF,n),
\end{aligned}
\end{equation}
and an estimate for $L^{2}(\nu_{X})$-error of the gradient:
\begin{equation}\label{eq:proof:theorem:empirical:minimizer:5}
\begin{aligned}
&\bbE_{\euD}\Big[\|\nabla(\what{f}_{\euD}^{\lambda}-f_{0})\|_{L^{2}(\nu_{X})}^{2}\Big] \\
&\lesssim\beta\lambda+\frac{\varepsilon_{\app}(\calF,\lambda)}{\lambda}+\frac{\varepsilon_{\gen}(\calF,n)}{\lambda}.
\end{aligned}
\end{equation}
Here the generalization error $\varepsilon_{\gen}(\calF,n)$ is defined as
\begin{align*}
&\varepsilon_{\gen}(\calF,n) \\
&=\frac{B_{0}^{2}+\sigma^{2}}{\log^{-1}n}\inf_{\delta>0}\Big\{\Big(\frac{2\log N(B_{0}\delta,\calF,L^{2}(\euD))}{n}\Big)^{\frac{1}{2}}+\delta\Big\}.
\end{align*}
\end{enumerate}

\par\noindent\emph{Step (I).}
Given a ghost sample $\euD^{\prime}=\{(X_{i}^{\prime},Y_{i}^{\prime})\}_{i=1}^{n}$, where $\{X_{i}^{\prime}\}_{i=1}^{n}$ are independently and identically drawn from $\mu_{X}$. Further, the ghost sample $\euD^{\prime}$ is independent of $\euD=\{(X_{i},Y_{i})\}_{i=1}^{n}$. Let $\varepsilon=\{\varepsilon_{i}\}_{i=1}^{n}$ be a set of Rademacher variables and independent of $\euD$ and $\euD^{\prime}$. Since that $\what{f}_{\euD}^{\lambda}\in\conv(\calF)$, by the technique of symmetrization, we have
\begin{align}
&\bbE_{\euD}\Big[R(\what{f}_{\euD}^{\lambda})-\what{R}_{\euD}(\what{f}_{\euD}^{\lambda})\Big]\leq\bbE_{\euD}\Big[\sup_{f\in\conv(\calF)}R(f)-\what{R}_{\euD}(f)\Big] \nonumber \\
&=\bbE_{\euD}\Big[\sup_{f\in\conv(\calF)}\bbE_{\euD^{\prime}}\Big[\frac{1}{n}\sum_{i=1}^{n}(f(X^{\prime}_{i})-f_{0}(X^{\prime}_{i}))^{2}\Big] \nonumber \\
&\quad-\frac{1}{n}\sum_{i=1}^{n}(f(X_{i})-f_{0}(X_{i}))^{2}\Big] \nonumber \\
&\leq\bbE_{\euD}\bbE_{\euD^{\prime}}\Big[\sup_{f\in\conv(\calF)}\frac{1}{n}\sum_{i=1}^{n}(f(X^{\prime}_{i})-f_{0}(X^{\prime}_{i}))^{2} \nonumber \\
&\quad-\frac{1}{n}\sum_{i=1}^{n}(f(X_{i})-f_{0}(X_{i}))^{2}\Big] \nonumber \\
&=\bbE_{\euD}\bbE_{\euD^{\prime}}\bbE_{\varepsilon}\Big[\sup_{f\in\conv(\calF)}\frac{1}{n}\sum_{i=1}^{n}\varepsilon_{i}\big((f(X^{\prime}_{i})-f_{0}(X^{\prime}_{i}))^{2} \nonumber \\
&\quad-(f(X_{i})-f_{0}(X_{i}))^{2}\big)\Big] \nonumber \\
&=2\bbE_{\euD}\bbE_{\varepsilon}\Big[\sup_{f\in\conv(\calF)}\frac{1}{n}\sum_{i=1}^{n}\varepsilon_{i}(f(X_{i})-f_{0}(X_{i}))^{2}\Big] \nonumber \\
&\leq 4B_{0}\bbE_{\euD}\bbE_{\varepsilon}\Big[\sup_{f\in\conv(\calF)}\frac{1}{n}\sum_{i=1}^{n}\varepsilon_{i}(f(X_{i})-f_{0}(X_{i}))\Big] \nonumber \\
&=4B_{0}\bbE_{\euD}\bbE_{\varepsilon}\Big[\sup_{f\in\conv(\calF)}\frac{1}{n}\sum_{i=1}^{n}\varepsilon_{i}f(X_{i})\Big] \nonumber \\
&=4B_{0}\bbE_{\euD}\bbE_{\varepsilon}\Big[\sup_{f\in\calF}\frac{1}{n}\sum_{i=1}^{n}\varepsilon_{i}f(X_{i})\Big], \label{eq:proof:theorem:empirical:minimizer:1:1}
\end{align}
where the second inequality follows from the convexity of supremum and Jensen's inequality, and the third inequality holds from Ledoux-Talagrand contraction inequality~\cite[Lemma 5.7]{mohri2018foundations} and the fact that $0\leq|f(X_{i})-f_{0}(X_{i})|\leq 2B_{0}$ for each $f\in\conv(\calF)$ and each $1\leq i\leq n$. The last equality invokes the fact that the Rademacher complexity of the convex hull of $\calF$ is equal to that of $\calF$.

\par Let $\delta>0$ and let $\calF_{\delta}$ be an $L^{2}(\euD)$ $(B_{0}\delta)$-cover of $\calF$ satisfying $|\calF_{\delta}|=N(B_{0}\delta,\calF,L^{2}(\euD))$. Then it follows from Cauchy-Schwarz inequality that for each $f\in\calF$, there exists $f_{\delta}\in\calF_{\delta}$ such that
\begin{align*}
&\frac{1}{n}\sum_{i=1}^{n}\varepsilon_{i}f(X_{i})-\frac{1}{n}\sum_{i=1}^{n}\varepsilon_{i}f_{\delta}(X_{i}) \\
&\leq\Big(\frac{1}{n}\sum_{i=1}^{n}\varepsilon_{i}^{2}\Big)^{1/2}\Big(\frac{1}{n}\sum_{i=1}^{n}(f(X_{i})-f_{\delta}(X_{i}))^{2}\Big)^{1/2}\leq B_{0}\delta.
\end{align*}
Combining~\eqref{eq:proof:theorem:empirical:minimizer:1:1} with the above inequality yields
\begin{align}
&\bbE_{\euD}\Big[R(\what{f}_{\euD}^{\lambda})-\what{R}_{\euD}(\what{f}_{\euD}^{\lambda})\Big] \nonumber \\
&\leq 4B_{0}\bbE_{\euD}\bbE_{\varepsilon}\Big[\sup_{f\in\calF_{\delta}}\frac{1}{n}\sum_{i=1}^{n}\varepsilon_{i}f(X_{i})\Big]+4B_{0}^{2}\delta \nonumber \\
&\leq 4B_{0}^{2}\Big(\frac{2\log|\calF_{\delta}|}{n}\Big)^{\frac{1}{2}}+4B_{0}^{2}\delta \nonumber \\
&=4B_{0}^{2}\Big(\frac{2\log N(B_{0}\delta,\calF,L^{2}(\euD))}{n}\Big)^{\frac{1}{2}}+4B_{0}^{2}\delta, \label{eq:proof:theorem:empirical:minimizer:1:2}
\end{align}
where the last inequality holds from Massart's lemma~\cite[Theroem 3.7]{mohri2018foundations}. This completes the proof of~\eqref{eq:proof:theorem:empirical:minimizer:1}.

\par\noindent\emph{Step (II).} 
According to~\cite[Lemma 4]{Bartlett2002Rademacher}, the Gaussian complexity can be bounded by the Rademacher complexity, that is,
\begin{align}
&\bbE_{\euD}\Big[\frac{1}{n}\sum_{i=1}^{n}\xi_{i}\what{f}_{\euD}^{\lambda}(X_{i})\Big] \nonumber \\
&\leq\bbE_{\euD}\Big[\sup_{f\in\conv(\calF)}\frac{1}{n}\sum_{i=1}^{n}\xi_{i}f(X_{i})\Big] \nonumber \\
&\leq\sigma(\log n)\bbE_{\euD}\bbE_{\varepsilon}\Big[\sup_{f\in\conv(\calF)}\frac{1}{n}\sum_{i=1}^{n}\varepsilon_{i}f(X_{i})\Big], \label{eq:proof:theorem:empirical:minimizer:2:1}
\end{align}
where $\varepsilon=\{\varepsilon_{i}\}_{i=1}^{n}$ is a set of Rademacher variables and independent of $\euD$. By the same argument as~\eqref{eq:proof:theorem:empirical:minimizer:1:1} and~\eqref{eq:proof:theorem:empirical:minimizer:1:2}, we have 
\begin{align}
&\bbE_{\euD}\bbE_{\varepsilon}\Big[\sup_{f\in\conv(\calF)}\frac{1}{n}\sum_{i=1}^{n}\varepsilon_{i}f(X_{i})\Big] \nonumber \\
&=\bbE_{\euD}\bbE_{\varepsilon}\Big[\sup_{f\in\calF}\frac{1}{n}\sum_{i=1}^{n}\varepsilon_{i}f(X_{i})\Big] \nonumber \\
&\leq B_{0}\Big(\frac{2\log N(B_{0}\delta,\calF,L^{2}(\euD))}{n}\Big)^{\frac{1}{2}}+\delta. \label{eq:proof:theorem:empirical:minimizer:2:2}
\end{align}
Combining~\eqref{eq:proof:theorem:empirical:minimizer:2:1} and~\eqref{eq:proof:theorem:empirical:minimizer:2:2} completes the proof of~\eqref{eq:proof:theorem:empirical:minimizer:2}.

\par\noindent\emph{Step (III).} For each element $f\in\conv(\calF)$, by the convexity of $\conv(\calF)$ we have $\what{f}_{\euD}^{\lambda}+t(f-\what{f}_{\euD}^{\lambda})\in\conv(\calF)$ for each $t\in[0,1]$. Now the optimality of $\what{f}_{\euD}^{\lambda}$ yields that for each $t\in[0,1]$
\begin{equation*}
\what{L}_{\euD}^{\lambda}(\what{f}_{\euD}^{\lambda})-\what{L}_{\euD}^{\lambda}(\what{f}_{\euD}^{\lambda}+t(f-\what{f}_{\euD}^{\lambda}))\leq 0,
\end{equation*}
which implies
\begin{align*}
&\lim_{t\rightarrow0^{+}}\frac{1}{t}\Big(\what{L}_{\euD}^{\lambda}(\what{f}_{\euD}^{\lambda})-\what{L}_{\euD}^{\lambda}(\what{f}_{\euD}^{\lambda}+t(f-\what{f}_{\euD}^{\lambda}))\Big) \\
&=\lambda(\nabla \what{f}_{\euD}^{\lambda},\nabla(\what{f}_{\euD}^{\lambda}-f))_{L^{2}(\nu_{X})} \nonumber \\
&\quad+\frac{1}{n}\sum_{i=1}^{n}(\what{f}_{\euD}^{\lambda}(X_{i})-Y_{i})(\what{f}_{\euD}^{\lambda}(X_{i})-f(X_{i}))\leq 0.
\end{align*}
Therefore, it follows from (1) that for each $f\in\conv(\calF)$,
\begin{equation}\label{eq:proof:theorem:empirical:minimizer:3:1}
\begin{aligned}
&\lambda(\nabla \what{f}_{\euD}^{\lambda},\nabla(\what{f}_{\euD}^{\lambda}-f))_{L^{2}(\nu_{X})}+(\what{f}_{\euD}^{\lambda}-f_{0},\what{f}_{\euD}^{\lambda}-f)_{L^{2}(\euD)} \\
&\leq\frac{1}{n}\sum_{i=1}^{n}\xi_{i}(\what{f}_{\euD}^{\lambda}(X_{i})-f(X_{i})).
\end{aligned}
\end{equation}
For the first term in the left-hand side of~\eqref{eq:proof:theorem:empirical:minimizer:3:1}, it follows from the linearity of inner-product that
\begin{align}
&\lambda(\nabla \what{f}_{\euD}^{\lambda},\nabla(\what{f}_{\euD}^{\lambda}-f))_{L^{2}(\nu_{X})} \nonumber \\
&=\lambda(\nabla(\what{f}_{\euD}^{\lambda}-f_{0})+\nabla f_{0},\nabla(\what{f}_{\euD}^{\lambda}-f_{0})-\nabla(f-f_{0}))_{L^{2}(\nu_{X})} \nonumber \\
&=\lambda\|\nabla(\what{f}_{\euD}^{\lambda}-f_{0})\|_{L^{2}(\nu_{X})}^{2}+\lambda(\nabla f_{0},\nabla(\what{f}_{\euD}^{\lambda}-f))_{L^{2}(\nu_{X})} \nonumber \\
&\quad-\lambda(\nabla(\what{f}_{\euD}^{\lambda}-f_{0}),\nabla(f-f_{0}))_{L^{2}(\nu_{X})}. \label{eq:proof:theorem:empirical:minimizer:3:2}
\end{align}
Then using Lemma~\ref{lemma:Green} and Assumption~\ref{assumption:regularity:f0}, one obtains easily
\begin{align}
&-\lambda(\nabla f_{0},\nabla(\what{f}_{\euD}^{\lambda}-f))_{L^{2}(\nu_{X})}  \nonumber \\
&=\lambda(\Delta f_{0},\what{f}_{\euD}^{\lambda}-f)_{L^{2}(\nu_{X})} \nonumber \\
&\quad+\lambda(\nabla f_{0}\cdot\nabla(\log q),\what{f}_{\euD}^{\lambda}-f)_{L^{2}(\nu_{X})} \nonumber \\
&\leq\lambda\kappa^{1/2}\Big\{\|\Delta f_{0}\|_{L^{2}(\nu_{X})}+\|\nabla f_{0}\cdot\nabla(\log q)\|_{L^{2}(\nu_{X})}\Big\} \nonumber \\
&\quad\times\Big\{R^{1/2}(\what{f}_{\euD}^{\lambda})+\|f-f_{0}\|_{L^{2}(\mu_{X})}\Big\} \nonumber \\
&\leq9\lambda^{2}\kappa\Big\{\|\Delta f_{0}\|_{L^{2}(\nu_{X})}^{2}+\|\nabla f_{0}\cdot\nabla(\log q)\|_{L^{2}(\nu_{X})}^{2}\Big\} \nonumber \\
&\quad+\frac{1}{16}R(\what{f}_{\euD}^{\lambda})+\frac{1}{2}\|f-f_{0}\|_{L^{2}(\mu_{X})}^{2}, \label{eq:proof:theorem:empirical:minimizer:3:3}
\end{align}
where the first inequality holds from Cauchy-Schwarz inequality and the triangular inequality, and the last inequality is due to $ab\leq\epsilon a^{2}+b^{2}/(4\epsilon)$ for $a,b,\epsilon>0$. Similarly, we also find that
\begin{equation}\label{eq:proof:theorem:empirical:minimizer:3:4}
\begin{aligned}
&\lambda(\nabla(\what{f}_{\euD}^{\lambda}-f_{0}),\nabla(f-f_{0}))_{L^{2}(\nu_{X})} \\
&\leq\frac{\lambda}{2}\|\nabla(\what{f}_{\euD}^{\lambda}-f_{0})\|_{L^{2}(\nu_{X})}^{2}+\frac{\lambda}{2}\|\nabla(f-f_{0})\|_{L^{2}(\nu_{X})}^{2}.
\end{aligned}
\end{equation}
Using~\eqref{eq:proof:theorem:empirical:minimizer:3:2},~\eqref{eq:proof:theorem:empirical:minimizer:3:3} and~\eqref{eq:proof:theorem:empirical:minimizer:3:4} yields
\begin{equation}\label{eq:proof:theorem:empirical:minimizer:3:5}
\begin{aligned}
&\frac{\lambda}{2}\|\nabla(\what{f}_{\euD}^{\lambda}-f_{0})\|_{L^{2}(\nu_{X})}^{2} \\
&\leq\lambda(\nabla \what{f}_{\euD}^{\lambda},\nabla(\what{f}_{\euD}^{\lambda}-f))_{L^{2}(\nu_{X})}+\frac{1}{16}R(\what{f}_{\euD}^{\lambda}) \\
&\quad+\Big\{\frac{1}{2}\|f-f_{0}\|_{L^{2}(\mu_{X})}^{2}+\frac{\lambda}{2}\|\nabla(f-f_{0})\|_{L^{2}(\nu_{X})}^{2}\Big\} \\
&\quad+9\lambda^{2}\kappa\Big\{\|\Delta f_{0}\|_{L^{2}(\nu_{X})}^{2}+\|\nabla f_{0}\cdot\nabla(\log q)\|_{L^{2}(\nu_{X})}^{2}\Big\}.
\end{aligned}
\end{equation}
We next turn to consider the second term in the left-hand side of~\eqref{eq:proof:theorem:empirical:minimizer:3:1}. By Cauchy-Schwarz inequality and AM-GM inequality we have
\begin{align}
&(\what{f}_{\euD}^{\lambda}-f_{0},\what{f}_{\euD}^{\lambda}-f)_{L^{2}(\euD)} \nonumber \\
&=\what{R}_{\euD}(\what{f}_{\euD}^{\lambda})-(\what{f}_{\euD}^{\lambda}-f_{0},f-f_{0})_{L^{2}(\euD)} \nonumber \\
&\geq\frac{1}{2}\what{R}_{\euD}(\what{f}_{\euD}^{\lambda})-\frac{1}{2}\what{R}_{\euD}(f). \label{eq:proof:theorem:empirical:minimizer:3:6}
\end{align}
Combining~\eqref{eq:proof:theorem:empirical:minimizer:3:1},~\eqref{eq:proof:theorem:empirical:minimizer:3:5} and~\eqref{eq:proof:theorem:empirical:minimizer:3:6} and taking expectation with respect to $\euD\sim\mu^{n}$ implies the following inequality for each $f\in\conv(\calF)$
\begin{align*}
&\frac{\lambda}{2}\bbE_{\euD}\Big[\|\nabla(\what{f}_{\euD}^{\lambda}-f_{0})\|_{L^{2}(\nu_{X})}^{2}\Big]+\frac{1}{2}\bbE_{\euD}\Big[\what{R}_{\euD}(\what{f}_{\euD}^{\lambda})\Big] \\
&\leq\frac{1}{16}\bbE_{\euD}\Big[R(\what{f}_{\euD}^{\lambda})\Big]+\bbE_{\euD}\Big[\frac{1}{n}\sum_{i=1}^{n}\xi_{i}\what{f}_{\euD}^{\lambda}(X_{i})\Big] \\
&\quad+\Big\{\|f-f_{0}\|_{L^{2}(\mu_{X})}^{2}+\frac{\lambda}{2}\|\nabla(f-f_{0})\|_{L^{2}(\nu_{X})}^{2}\Big\} \\
&\quad+9\lambda^{2}\kappa\Big\{\|\Delta f_{0}\|_{L^{2}(\nu_{X})}^{2}+\|\nabla f_{0}\cdot\nabla(\log q)\|_{L^{2}(\nu_{X})}^{2}\Big\},
\end{align*}
where we used the fact that $\bbE[\what{R}_{\euD}(f)]=R(f)$ and $\bbE[\sum_{i=1}^{n}\xi_{i}f(X_{i})]=0$ for each fixed function $f\in L^{\infty}(\Omega)$. Since that $\calF\subseteq\conv(\calF)$, it is apparent that this inequality also holds for each element in $\calF$. Taking infimum with respect to $f\in\calF$ obtains the inequality~\eqref{eq:proof:theorem:empirical:minimizer:3}.

\par\noindent\emph{Step (IV).} Using~\eqref{eq:proof:theorem:empirical:minimizer:2} and~\eqref{eq:proof:theorem:empirical:minimizer:3}, we have
\begin{align*}
&\bbE_{\euD}\Big[\what{R}_{\euD}(\what{f}_{\euD}^{\lambda})\Big] \\
&\leq\frac{1}{4}\bbE_{\euD}\Big[R(\what{f}_{\euD}^{\lambda})\Big]+c\Big\{\beta\lambda^{2}+\varepsilon_{\app}(\calF,\lambda)+\varepsilon_{\gen}(\calF,n)\Big\},
\end{align*}
where $c$ is an absolute positive constant and $\beta$ is a positive constant defined as
\begin{equation*}
\beta=\kappa\Big\{\|\Delta f_{0}\|_{L^{2}(\nu_{X})}^{2}+\|\nabla f_{0}\cdot\nabla(\log q)\|_{L^{2}(\nu_{X})}^{2}\Big\}.
\end{equation*}
Consequently, by the estimate in~\eqref{eq:proof:theorem:empirical:minimizer:1}, we have
\begin{equation}\label{eq:proof:theorem:empirical:minimizer:4:1}
\begin{aligned}
&\bbE_{\euD}\Big[R(\what{f}_{\euD}^{\lambda})\Big]\leq\bbE_{\euD}\Big[\what{R}_{\euD}(\what{f}_{\euD}^{\lambda})\Big]+c^{\prime}\varepsilon_{\gen}(\calF,n) \\
&\leq\frac{1}{4}\bbE_{\euD}\Big[R(\what{f}_{\euD}^{\lambda})\Big]+c\Big\{\beta\lambda^{2}+\varepsilon_{\app}(\calF,\lambda)\Big\} \nonumber \\
&\quad+(c+c^{\prime})\varepsilon_{\gen}(\calF,n).
\end{aligned}
\end{equation}
This completes the proof of~\eqref{eq:proof:theorem:empirical:minimizer:4}. Finally, combining~\eqref{eq:proof:theorem:empirical:minimizer:3} and~\eqref{eq:proof:theorem:empirical:minimizer:4} achieves~\eqref{eq:proof:theorem:empirical:minimizer:5}.
\end{proof}

\begin{proof}[Proof of Lemma~\ref{lemma:approximation:H1}]
A direct conclusion of Lemma~\ref{lemma:approximation:requ}.
\end{proof}

\begin{proof}[Proof of Theorem~\ref{theorem:erm:rate}]
According to Lemma 5.3, we set the hypothesis class $\calF$ as ReQU neural networks $\calF=\calN(L,S)$ with $L=\calO(\log N)$ and $S=\calO(N^{d})$. Then there exists $f\in\calF$ such that
\begin{align*}
\|f-\phi\|_{L^{2}(\mu_{X})}\leq CN^{-s}, \\
\|\nabla(f-\phi)\|_{L^{2}(\nu_{X})}\leq CN^{-(s-1)} .
\end{align*}
By using Lemma~\ref{lemma:covering:number:sparse} and set $\delta=1/n$, we find 
\begin{multline}\label{eq:proof:theorem:erm:rate:1}
\log N(B_{0}n^{-1},\calF,L^{2}(\euD)) \\
\lesssim LS\log(S)(\log n)\lesssim N^{d}\log^{2}N\log n.
\end{multline}
Substituting these estimates into Lemma~\ref{lemma:oracle:erm} yields
\begin{align*}
&\bbE_{\euD}\Big[\|\what{f}_{\euD}^{\lambda}-f_{0}\|_{L^{2}(\mu_{X})}^{2}\Big] \\
&\lesssim\beta\lambda^{2}+CN^{-2s}+C\lambda N^{-2(s-1)} \\
&\quad+C\log n\Big(\frac{N^{d}\log^{2}N\log n}{n}\Big)^{\frac{1}{2}}.
\end{align*}
Setting $N=\calO(n^{\frac{1}{d+4s}})$, and letting the regularization parameter be $\lambda=\calO(n^{-\frac{s}{d+4s}}\log^{2}n)$ deduce the desired result.
\end{proof}

\subsection{Semi-Supervised Deep Sobolev Regressor}\label{section:appendix:SDORE}

\begin{proof}[Proof of Lemma~\ref{lemma:oracle:erm:data}]
Before proceeding, we first define the empirical inner-product and norm based on the sample $\euS=\{Z_{i}\}_{i=1}^{m}$ as
\begin{equation*}
(u,v)_{L^{2}(\euS)}=\frac{1}{m}\sum_{i=1}^{m}u(Z_{i})v(Z_{i}), ~ \|u\|_{L^{2}(\euS)}^{2}=\frac{1}{m}\sum_{i=1}^{m}u^{2}(Z_{i}), 
\end{equation*}
for each $u,v\in L^{\infty}(\nu_{X})$. The proof is divided into four parts which are denoted by (I) to (IV):
\begin{enumerate}
\item[(I)] By a same argument as (I) in the proof of Lemma~\ref{lemma:oracle:erm}, we deduces
\begin{multline}\label{eq:proof:theorem:empirical:minimizer:data:1}
\bbE_{\euS}\Big[\|\nabla(\what{f}_{\euD,\euS}^{\lambda}-f_{0})\|_{L^{2}(\nu_{X})}^{2}-\|\nabla(\what{f}_{\euD,\euS}^{\lambda}-f_{0})\|_{L^{2}(\euS)}^{2}\Big] \\
\leq\varepsilon_{\gen}^{\reg}(\nabla\calF,m).
\end{multline}
Here the generalization error $\varepsilon_{\gen}^{\reg}(\nabla\calF,m)$ associated to the regularization term are defined as
\begin{multline*}
\varepsilon_{\gen}^{\reg}(\nabla\calF,m) \\
=B_{1}^{2}\inf_{\delta>0}\Big\{\max_{1\leq k\leq d}\frac{N(B_{1,k}\delta,D_{k}\calF,L^{2}(\euS))}{m}+\delta\Big\}.
\end{multline*}
\item[(II)] By the technique of symmetrization and Green's formula, it holds that
\begin{multline}\label{eq:proof:theorem:empirical:minimizer:data:2}
-\lambda\bbE_{\euS}\Big[(\nabla f_{0},\nabla(\what{f}_{\euD,\euS}^{\lambda}-f))_{L^{2}(\euS)}\Big] \\
\leq\frac{1}{16}\bbE_{\euS}\Big[R(\what{f}_{\euD,\euS}^{\lambda})\Big]+c\Big\{\tilde{\beta}\lambda^{2}+\varepsilon_{\gen}^{\reg}(\nabla\calF,m)\Big\},
\end{multline}
where $c$ is an absolute positive constant. Here the constant $\tilde{\beta}$ is defined as
\begin{equation*}
\tilde{\beta}=\kappa\Big\{\|\Delta f_{0}\|_{L^{2}(\nu_{X})}^{2}+\|\nabla f_{0}\cdot\nabla(\log q)\|_{L^{2}(\nu_{X})}^{2}+B_{1}^{2}\Big\}.
\end{equation*}
\item[(III)] With the aid of the variational inequality and~\eqref{eq:proof:theorem:empirical:minimizer:data:2}, we have
\begin{align}
&\lambda\bbE_{\euD,\euS}\Big[\|\nabla(\what{f}_{\euD,\euS}^{\lambda}-f_{0})\|_{L^{2}(\euS)}^{2}\Big]+\bbE_{\euD,\euS}\Big[\what{R}_{\euD}(\what{f}_{\euD,\euS}^{\lambda})\Big] \nonumber \\
&\leq\frac{1}{8}\bbE_{\euD,\euS}\Big[R(\what{f}_{\euD,\euS}^{\lambda})\Big]+2\bbE_{\euD,\euS}\Big[\frac{1}{n}\sum_{i=1}^{n}\xi_{i}\what{f}_{\euD,\euS}^{\lambda}(X_{i})\Big] \nonumber \\
&\quad+c\Big\{\tilde{\beta}\lambda^{2}+\varepsilon_{\app}(\calF,\lambda)+\varepsilon_{\gen}^{\reg}(\nabla\calF,m)\Big\}, \label{eq:proof:theorem:empirical:minimizer:data:3}
\end{align}
where $c$ is an absolute positive constant
\item[(IV)] Applying~\eqref{eq:proof:theorem:empirical:minimizer:1},~\eqref{eq:proof:theorem:empirical:minimizer:2},~\eqref{eq:proof:theorem:empirical:minimizer:data:1} and~\eqref{eq:proof:theorem:empirical:minimizer:data:3}, we conclude the final results.
\end{enumerate}

\par\noindent\emph{Step (I).}
By a same argument as \emph{Step (I)} in the proof of Lemma~\ref{lemma:oracle:erm}, we deduce the following inequality
\begin{align*}
&\bbE_{\euS}\Big[\|D_{k}(\what{f}_{\euD,\euS}^{\lambda}-f_{0})\|_{L^{2}(\nu_{X})}^{2}-\|\what{f}_{\euD,\euS}^{\lambda}-f_{0}\|_{L^{2}(\euS)}^{2}\Big] \\
&\leq 4B_{1,k}^{2}\inf_{\delta>0}\Big\{\frac{N(B_{1,k}\delta,D_{k}\calF,L^{2}(\euS))}{m}+\delta\Big\},
\end{align*}
for each $1\leq k\leq d$. Summing over these equalities obtains~\eqref{eq:proof:theorem:empirical:minimizer:data:1} immediately.

\par\noindent\emph{Step (II).}
Given a ghost sample $\euS^{\prime}=\{Z_{i}^{\prime}\}_{i=1}^{m}$, where $\{Z_{i}^{\prime}\}_{i=1}^{m}$ are independently and identically distributed random variables from $\nu_{X}$. Further, the ghost sample $\euS^{\prime}$ is independent of $\euS=\{Z_{i}\}_{i=1}^{m}$. Let $\varepsilon=\{\varepsilon_{i}\}_{i=1}^{m}$ be a set of Rademacher variables and independent of $\euS$ and $\euS^{\prime}$. Then by the technique of symmetrization, we have
\begin{align}
&\bbE_{\euS}\Big[(D_{k}f_{0},D_{k}\what{f}_{\euD,\euS}^{\lambda})_{L^{2}(\nu_{X})}-(D_{k}f_{0},D_{k}\what{f}_{\euD,\euS}^{\lambda})_{L^{2}(\euS)}\Big] \nonumber \\
&\leq\bbE_{\euS}\Big[\sup_{f\in\conv(\calF)}(D_{k}f_{0},D_{k}f)_{L^{2}(\nu_{X})}-(D_{k}f_{0},D_{k}f)_{L^{2}(\euS)}\Big] \nonumber \\
&=\bbE_{\euS}\Big[\sup_{f\in\conv(\calF)}\bbE_{\euS^{\prime}}\Big[\frac{1}{m}\sum_{i=1}^{m}D_{k}f_{0}(Z_{i}^{\prime})D_{k}f(Z_{i}^{\prime})\Big] \nonumber \\
&\quad-\frac{1}{m}\sum_{i=1}^{m}D_{k}f_{0}(Z_{i})D_{k}f(Z_{i})\Big] \nonumber \\
&\leq\bbE_{\euS}\bbE_{\euS^{\prime}}\Big[\sup_{f\in\conv(\calF)}\frac{1}{m}\sum_{i=1}^{m}D_{k}f_{0}(Z_{i}^{\prime})D_{k}f(Z_{i}^{\prime}) \nonumber \\
&\quad-\frac{1}{m}\sum_{i=1}^{m}D_{k}f_{0}(Z_{i})D_{k}f(Z_{i})\Big] \nonumber \\
&=\bbE_{\euS}\bbE_{\euS^{\prime}}\bbE_{\varepsilon}\Big[\sup_{f\in\conv(\calF)}\frac{1}{m}\sum_{i=1}^{m}\varepsilon_{i}\Big(D_{k}f_{0}(Z_{i}^{\prime})D_{k}f(Z_{i}^{\prime}) \nonumber \\
&\quad-D_{k}f_{0}(Z_{i})D_{k}f(Z_{i})\Big)\Big] \nonumber \\
&=\bbE_{\euS}\bbE_{\varepsilon}\Big[\sup_{f\in\conv(\calF)}\frac{1}{m}\sum_{i=1}^{m}\varepsilon_{i}D_{k}f_{0}(Z_{i})D_{k}f(Z_{i})\Big] \nonumber \\
&=\bbE_{\euS}\bbE_{\varepsilon}\Big[\sup_{f\in\calF}\frac{1}{m}\sum_{i=1}^{m}\varepsilon_{i}D_{k}f_{0}(Z_{i})D_{k}f(Z_{i})\Big], \label{eq:proof:theorem:empirical:minimizer:data:2:1}
\end{align}
where the second inequality follows from the Jensen's inequality, and the last equality invokes the fact that the Rademacher complexity of the convex hull is equal to that of the original set. According to Ledoux-Talagrand contraction inequality~\cite[Lemma 5.7]{mohri2018foundations}, we have
\begin{equation}\label{eq:proof:theorem:empirical:minimizer:data:2:2}
\begin{aligned}
&\bbE_{\varepsilon}\Big[\sup_{f\in\calF}\frac{1}{m}\sum_{i=1}^{m}\varepsilon_{i}D_{k}f_{0}(Z_{i})D_{k}f(Z_{i})\Big] \\
&\leq B_{1,k}\bbE_{\varepsilon}\Big[\sup_{f\in\calF}\frac{1}{m}\sum_{i=1}^{m}\varepsilon_{i}D_{k}f(Z_{i})\Big].
\end{aligned}
\end{equation}
Let $\delta>0$ and let $(D_{k}\calF)_{\delta}$ be an $L^{2}(\euS)$ $(B_{1,k}\delta)$-cover of $D_{k}\calF$. Suppose $|(D_{k}\calF)_{\delta}|=N(B_{1,k}\delta,D_{k}\calF,L^{2}(\euS))$. Then it follows from Cauchy-Schwarz inequality that for each $D_{k}f\in D_{k}\calF$, there exists $(D_{k}f)_{\delta}\in(D_{k}\calF)_{\delta}$ such that
\begin{align*}
&\frac{1}{m}\sum_{i=1}^{m}\varepsilon_{i}D_{k}f(Z_{i})-\frac{1}{m}\sum_{i=1}^{m}\varepsilon_{i}(D_{k}f)_{\delta}(Z_{i}) \\
&\leq\Big(\frac{1}{m}\sum_{i=1}^{m}\varepsilon_{i}^{2}\Big)^{1/2}\Big(\frac{1}{m}\sum_{i=1}^{m}(D_{k}f(Z_{i})-(D_{k}f)_{\delta}(Z_{i}))^{2}\Big)^{1/2} \\
&\leq B_{1,k}\delta,
\end{align*}
which implies
\begin{align}
&\bbE_{\euS}\bbE_{\varepsilon}\Big[\sup_{D_{k}f\in\calF}\frac{1}{m}\sum_{i=1}^{m}\varepsilon_{i}D_{k}f(Z_{i})\Big] \nonumber \\
&\leq\bbE_{\euS}\bbE_{\varepsilon}\Big[\sup_{D_{k}f\in(D_{k}\calF)_{\delta}}\frac{1}{m}\sum_{i=1}^{m}\varepsilon_{i}D_{k}f(Z_{i})\Big]+B_{1,k}\delta \nonumber \\
&\leq B_{1,k}\Big(\frac{2\log N(B_{1,k}\delta,D_{k}\calF,L^{2}(\euS))}{m}\Big)^{1/2}+B_{1,k}\delta, \label{eq:proof:theorem:empirical:minimizer:data:2:3}
\end{align}
where the last inequality holds from Massart's lemma~\cite[Theroem 3.7]{mohri2018foundations}. Combining~\eqref{eq:proof:theorem:empirical:minimizer:data:2:1},~\eqref{eq:proof:theorem:empirical:minimizer:data:2:2} and~\eqref{eq:proof:theorem:empirical:minimizer:data:2:3} deduces
\begin{align*}
&\bbE_{\euS}\Big[(D_{k}f_{0},D_{k}\what{f}_{\euD,\euS}^{\lambda})_{L^{2}(\nu_{X})}-(D_{k}f_{0},D_{k}\what{f}_{\euD,\euS}^{\lambda})_{L^{2}(\euS)}\Big] \\
&\leq B_{1,k}^{2}\Big(\frac{2\log N(B_{1,k}\delta,D_{k}\calF,L^{2}(\euS))}{m}\Big)^{1/2}+B_{1,k}^{2}\delta.
\end{align*}
Summing over this equation for $1\leq k\leq d$ yields
\begin{align*}
&\bbE_{\euS}\Big[(\nabla f_{0},\nabla(\what{f}_{\euD,\euS}^{\lambda}-f_{0}))_{L^{2}(\nu_{X})} \\
&\quad-(\nabla f_{0},\nabla(\what{f}_{\euD,\euS}^{\lambda}-f_{0}))_{L^{2}(\euS)}\Big] \\
&\leq\sum_{k=1}^{d}B_{1,k}^{2}\inf_{\delta>0}\Big\{\Big(\frac{2\log N(B_{1,k}\delta,D_{k}\calF,L^{2}(\euS))}{m}\Big)^{1/2}+\delta\Big\}.
\end{align*}
Combining this with Lemma~\ref{lemma:Green} and Assumption~\ref{assumption:regularity:f0}, we find that for each $\delta>0$,
\begin{align*}
&-\lambda\bbE_{\euS}\Big[(\nabla f_{0},\nabla(\what{f}_{\euD,\euS}^{\lambda}-f))_{L^{2}(\euS)}\Big] \\
&\leq\lambda\bbE_{\euS}\Big[-(\nabla f_{0},\nabla(\what{f}_{\euD,\euS}^{\lambda}-f))_{L^{2}(\nu_{X})}\Big] \nonumber \\
&\quad+\lambda\sum_{k=1}^{d}B_{1,k}^{2}\Big\{\Big(\frac{2\log N(B_{1,k}\delta,D_{k}\calF,L^{2}(\euS))}{m}\Big)^{1/2}+\delta\Big\} \\
&=\lambda\bbE_{\euS}\Big[(\Delta f_{0},\what{f}_{\euD,\euS}^{\lambda}-f)_{L^{2}(\nu_{X})} \\
&\quad+(\nabla f_{0}\cdot\nabla(\log q),\what{f}_{\euD,\euS}^{\lambda}-f)_{L^{2}(\nu_{X})}\Big] \\
&\quad+\lambda\sum_{k=1}^{d}B_{1,k}^{2}\Big\{\Big(\frac{2\log N(B_{1,k}\delta,D_{k}\calF,L^{2}(\euS))}{m}\Big)^{1/2}+\delta\Big\} \\
&\leq\lambda\kappa^{1/2}\bbE_{\euS}\Big[\Big\{\|\Delta f_{0}+\nabla f_{0}\cdot\nabla(\log q)\|_{L^{2}(\nu_{X})}\Big\} \\
&\quad\times\Big\{R(\what{f}_{\euD,\euS}^{\lambda})^{1/2}+\|f-f_{0}\|_{L^{2}(\mu_{X})}\Big\}\Big] \\
&\quad+\lambda\sum_{k=1}^{d}B_{1,k}^{2}\Big\{\Big(\frac{2\log N(B_{1,k}\delta,D_{k}\calF,L^{2}(\euS))}{m}\Big)^{1/2}+\delta\Big\} \\
&\leq9\lambda^{2}\kappa\Big\{\|\Delta f_{0}\|_{L^{2}(\nu_{X})}^{2}+\|\nabla f_{0}\cdot\nabla(\log q)\|_{L^{2}(\nu_{X})}^{2}\Big\} \\
&\quad+\frac{1}{16}\bbE_{\euS}\Big[R(\what{f}_{\euD,\euS}^{\lambda})\Big]+\frac{1}{2}\|f-f_{0}\|_{L^{2}(\mu_{X})}^{2}+\lambda^{2}\Big(\sum_{k=1}^{d}B_{1,k}^{2}\Big) \\
&\quad+\frac{1}{4}\sum_{k=1}^{d}B_{1,k}^{2}\Big\{\max_{1\leq k\leq d}\frac{2\log N(B_{1,k}\delta,D_{k}\calF,L^{2}(\euS))}{m}+\delta^{2}\Big\},
\end{align*}
where the second inequality holds from Cauchy-Schwarz inequality and Assumption~\ref{assumption:bounded:density:ratio}, and the last inequality is due to the inequality $ab\leq\epsilon a^{2}+b^{2}/(4\epsilon)$ for $a,b,\epsilon>0$. This completes the proof of~\eqref{eq:proof:theorem:empirical:minimizer:data:2}.

\par\noindent\emph{Step (III).} For each element $f\in\conv(\calF)$, by the convexity of $\conv(\calF)$ we have $\what{f}_{\euD,\euS}^{\lambda}+t(f-\what{f}_{\euD,\euS}^{\lambda})\in\conv(\calF)$ for each $t\in[0,1]$. Now the optimality of $\what{f}_{\euD,\euS}^{\lambda}$ yields that for each $t\in[0,1]$
\begin{equation*}
\what{L}_{\euD,\euS}^{\lambda}(\what{f}_{\euD}^{\lambda})-\what{L}_{\euD,\euS}^{\lambda}(\what{f}_{\euD,\euS}^{\lambda}+t(f-\what{f}_{\euD,\euS}^{\lambda}))\leq 0,
\end{equation*}
which implies
\begin{align*}
&\lim_{t\rightarrow0^{+}}\frac{1}{t}\Big(\what{L}_{\euD,\euS}^{\lambda}(\what{f}_{\euD,\euS}^{\lambda})-\what{L}_{\euD,\euS}^{\lambda}(\what{f}_{\euD,\euS}^{\lambda}+t(f-\what{f}_{\euD,\euS}^{\lambda}))\Big) \\
&=\lambda(\nabla \what{f}_{\euD,\euS}^{\lambda},\nabla(\what{f}_{\euD,\euS}^{\lambda}-f))_{L^{2}(\euS)} \\
&\quad+\frac{1}{n}\sum_{i=1}^{n}(\what{f}_{\euD,\euS}^{\lambda}(X_{i})-Y_{i})(\what{f}_{\euD,\euS}^{\lambda}(X_{i})-f(X_{i}))\leq 0.
\end{align*}
Therefore, it follows from (1) that for each $f\in\conv(\calF)$,
\begin{equation}\label{eq:proof:theorem:empirical:minimizer:data:3:1}
\begin{aligned}
&\lambda(\nabla \what{f}_{\euD,\euS}^{\lambda},\nabla(\what{f}_{\euD,\euS}^{\lambda}-f))_{L^{2}(\euS)} \\
&\quad+(\what{f}_{\euD,\euS}^{\lambda}-f_{0},\what{f}_{\euD,\euS}^{\lambda}-f)_{L^{2}(\euD)} \\
&\leq\frac{1}{n}\sum_{i=1}^{n}\xi_{i}(\what{f}_{\euD,\euS}^{\lambda}(X_{i})-f(X_{i})).
\end{aligned}
\end{equation}
For the first term in the left-hand side of~\eqref{eq:proof:theorem:empirical:minimizer:data:3:1}, we have
\begin{align}
&\lambda(\nabla \what{f}_{\euD,\euS}^{\lambda},\nabla(\what{f}_{\euD,\euS}^{\lambda}-f))_{L^{2}(\euS)} \nonumber \\
&=\lambda(\nabla(\what{f}_{\euD,\euS}^{\lambda}-f_{0})+\nabla f_{0},\nabla(\what{f}_{\euD,\euS}^{\lambda}-f_{0})-\nabla(f-f_{0}))_{L^{2}(\euS)} \nonumber \\
&=\lambda\|\nabla(\what{f}_{\euD,\euS}^{\lambda}-f_{0})\|_{L^{2}(\euS)}^{2}+\lambda(\nabla f_{0},\nabla(\what{f}_{\euD,\euS}^{\lambda}-f))_{L^{2}(\euS)} \nonumber \\
&\quad-\lambda(\nabla(\what{f}_{\euD,\euS}^{\lambda}-f_{0}),\nabla(f-f_{0}))_{L^{2}(\euS)}, \label{eq:proof:theorem:empirical:minimizer:data:3:2}
\end{align}
According to Cauchy-Schwarz inequality and AM-GM inequality, one obtains easily
\begin{multline}\label{eq:proof:theorem:empirical:minimizer:data:3:3}
\lambda(\nabla(\what{f}_{\euD,\euS}^{\lambda}-f_{0}),\nabla(f-f_{0}))_{L^{2}(\euS)} \\
\leq\frac{\lambda}{2}\|\nabla(\what{f}_{\euD,\euS}^{\lambda}-f_{0})\|_{L^{2}(\euS)}^{2}+\frac{\lambda}{2}\|\nabla(f-f_{0})\|_{L^{2}(\euS)}^{2}.
\end{multline}
Using~\eqref{eq:proof:theorem:empirical:minimizer:data:3:2} and~\eqref{eq:proof:theorem:empirical:minimizer:data:3:3}, and taking expectation with respect to $\euS\sim\nu_{X}^{m}$ yield
\begin{align*}
&\frac{\lambda}{2}\bbE_{\euS}\Big[\|\nabla(\what{f}_{\euD,\euS}^{\lambda}-f_{0})\|_{L^{2}(\euS)}^{2}\Big] \\
&\leq\lambda\bbE_{\euS}\Big[(\nabla \what{f}_{\euD,\euS}^{\lambda},\nabla(\what{f}_{\euD,\euS}^{\lambda}-f))_{L^{2}(\euS)}\Big] \\
&\quad+\frac{\lambda}{2}\|\nabla(f-f_{0})\|_{L^{2}(\nu_{X})}^{2} \\
&\quad-\lambda\bbE_{\euS}\Big[(\nabla f_{0},\nabla(\what{f}_{\euD,\euS}^{\lambda}-f))_{L^{2}(\euS)}\Big].
\end{align*}
Combining this estimate with~\eqref{eq:proof:theorem:empirical:minimizer:data:2} implies
\begin{equation}\label{eq:proof:theorem:empirical:minimizer:data:3:4}
\begin{aligned}
&\frac{\lambda}{2}\bbE_{\euS}\Big[\|\nabla(\what{f}_{\euD,\euS}^{\lambda}-f_{0})\|_{L^{2}(\euS)}^{2}\Big] \\
&\leq\lambda\bbE_{\euS}\Big[(\nabla \what{f}_{\euD,\euS}^{\lambda},\nabla(\what{f}_{\euD,\euS}^{\lambda}-f))_{L^{2}(\euS)}\Big] \\
&\quad+\frac{\lambda}{2}\|\nabla(f-f_{0})\|_{L^{2}(\nu_{X})}^{2}+\frac{1}{16}\bbE_{\euS}\Big[R(\what{f}_{\euD,\euS}^{\lambda})\Big] \\
&\quad+c\Big\{\tilde{\beta}\lambda^{2}+\varepsilon_{\gen}^{\reg}(\nabla\calF,m)\Big\}.
\end{aligned}
\end{equation}
We next turn to consider the second term in the left-hand side of~\eqref{eq:proof:theorem:empirical:minimizer:data:3:1}. By Cauchy-Schwarz inequality and AM-GM inequality we have
\begin{align*}
&(\what{f}_{\euD,\euS}^{\lambda}-f_{0},\what{f}_{\euD,\euS}^{\lambda}-f)_{L^{2}(\euD)} \\
&=\what{R}_{\euD}(\what{f}_{\euD,\euS}^{\lambda})-(\what{f}_{\euD,\euS}^{\lambda}-f_{0},f-f_{0})_{L^{2}(\euD)} \\
&\geq\frac{1}{2}\what{R}_{\euD}(\what{f}_{\euD,\euS}^{\lambda})-\frac{1}{2}\what{R}_{\euD}(f),
\end{align*}
which implies by taking expectation with respect to $\euD\sim\mu^{n}$ that for each $f\in\calF$,
\begin{multline}\label{eq:proof:theorem:empirical:minimizer:data:3:5}
\frac{1}{2}\bbE_{\euD}\Big[\what{R}_{\euD}(\what{f}_{\euD,\euS}^{\lambda})\Big] \\
\leq\frac{1}{2}R(f)+\bbE_{\euD}\Big[(\what{f}_{\euD,\euS}^{\lambda}-f_{0},\what{f}_{\euD,\euS}^{\lambda}-f)_{L^{2}(\euD)}\Big].
\end{multline}
Combining~\eqref{eq:proof:theorem:empirical:minimizer:data:3:1},~\eqref{eq:proof:theorem:empirical:minimizer:data:3:4} and~\eqref{eq:proof:theorem:empirical:minimizer:data:3:5} yields~\eqref{eq:proof:theorem:empirical:minimizer:data:3}.

\par\noindent\emph{Step (IV).} Using~\eqref{eq:proof:theorem:empirical:minimizer:2} and~\eqref{eq:proof:theorem:empirical:minimizer:data:3}, we have
\begin{multline*}
\bbE_{\euD,\euS}\Big[\what{R}_{\euD}(\what{f}_{\euD,\euS}^{\lambda})\Big]\leq\frac{1}{4}\bbE_{\euD,\euS}\Big[R(\what{f}_{\euD,\euS}^{\lambda})\Big] \\
+c\Big\{\tilde{\beta}\lambda^{2}+\varepsilon_{\app}(\calF,\lambda)+\varepsilon_{\gen}(\calF,n)+\varepsilon_{\gen}^{\reg}(\nabla\calF,m)\Big\}.
\end{multline*}
Then according to the above inequality and~\eqref{eq:proof:theorem:empirical:minimizer:1}, it follows that
\begin{equation*}
\begin{aligned}
&\bbE_{\euD,\euS}\Big[R(\what{f}_{\euD,\euS}^{\lambda})\Big]\leq\bbE_{\euD,\euS}\Big[\what{R}_{\euD}(\what{f}_{\euD,\euS}^{\lambda})\Big]+c^{\prime}\varepsilon_{\gen}(\calF,n) \\
&\leq\frac{1}{2}\bbE_{\euD,\euS}\Big[R(\what{f}_{\euD,\euS}^{\lambda})\Big]+(2c+c^{\prime})\varepsilon_{\gen}(\calF,n) \\
&\quad+2c\Big\{\tilde{\beta}\lambda^{2}+\varepsilon_{\app}(\calF,\lambda)+\varepsilon_{\gen}^{\reg}(\nabla\calF,m)\Big\},
\end{aligned}
\end{equation*}
which implies
\begin{multline}\label{eq:proof:theorem:empirical:minimizer:data:4:1}
\bbE_{\euD,\euS}\Big[\|\what{f}_{\euD,\euS}^{\lambda}-f_{0}\|_{L^{2}(\mu_{X})}^{2}\Big] \\
\lesssim\tilde{\beta}\lambda^{2}+\varepsilon_{\app}(\calF,\lambda)+\varepsilon_{\gen}(\calF,n)+\varepsilon_{\gen}^{\reg}(\nabla\calF,m).
\end{multline}
Finally, combining~\eqref{eq:proof:theorem:empirical:minimizer:data:1},~\eqref{eq:proof:theorem:empirical:minimizer:data:3} and ~\eqref{eq:proof:theorem:empirical:minimizer:data:4:1} completes the proof.
\end{proof}

\begin{proof}[Proof of Theorem~\ref{theorem:erm:rate:data}]
According to Lemma~\ref{lemma:approximation:H1}, we set the hypothesis class $\calF$ as ReQU neural networks $\calF=\calN(L,S)$ with $L=\calO(\log N)$ and $S=\calO(N^{d})$. Then there exists $f\in\calF$ such that
\begin{equation*}
\|f-\phi\|_{L^{2}(\mu_{X})}\leq CN^{-s}, \quad
\|\nabla(f-\phi)\|_{L^{2}(\nu_{X})}\leq CN^{-(s-1)}.
\end{equation*}
By using Lemma~\ref{lemma:covering:number:sparse:grad} and set $\delta=1/n$, we find 
\begin{multline*}
\log N(B_{1,k}n^{-1},D_{k}\calF,L^{2}(\euD)) \\
\lesssim L^{2}S\log S\log n\lesssim N^{d}\log^{3}N\log n.
\end{multline*}
Substituting these estimates and~\eqref{eq:proof:theorem:erm:rate:1} into Lemma~\ref{lemma:oracle:erm} yields
\begin{align*}
&\bbE_{\euD}\Big[\|\what{f}_{\euD}^{\lambda}-f_{0}\|_{L^{2}(\mu_{X})}^{2}\Big] \\
&\lesssim\tilde{\beta}\lambda^{2}+CN^{-2s}+C\lambda N^{-2(s-1)} \\
&\quad+C\log n\Big(\frac{N^{d}\log N\log n}{n}\Big)^{\frac{1}{2}}+C\log n\frac{N^{d}\log^{3}N}{m}.
\end{align*}
Setting $N=\calO(n^{\frac{1}{d+4s}})$, and letting the regularization parameter be $\lambda=\calO(n^{-\frac{s}{d+4s}}\log^{2}n)$ deduce the desired result.
\end{proof}

\section{Proofs in Results in Section~\ref{section:applications}}

\begin{proof}[Proof of Corollary~\ref{corollary:variable:selection}]
A direct conclusion of Theorem~\ref{theorem:erm:rate}.
\end{proof}

\begin{proof}[Proof of Corollary~\ref{corollary:selection:consistency}]
By Markov's inequality~\cite[Theorem C.11]{mohri2018foundations}, the following inequality holds for each $\epsilon>0$
\begin{equation}\label{eq:proof:corollary:selection:consistency:1}
\begin{aligned}
&\lim_{n\rightarrow\infty}\pr\Big\{\|D_{k}\what{f}_{\euD}^{\lambda}-D_{k}f_{0}\|_{L^{2}(\nu_{X})}>\epsilon\Big\} \\
&\leq\lim_{n\rightarrow\infty}\frac{\bbE_{\euD}\Big[\|D_{k}\what{f}_{\euD}^{\lambda}-D_{k}f_{0}\|_{L^{2}(\nu_{X})}\Big]}{\epsilon}=0,
\end{aligned}
\end{equation}
where the equality follows from Corollary~\ref{corollary:variable:selection}. For each irrelevant variable $k\notin\calI(f_{0})$, one has $\|D_{k} f_{0}\|_{L^{2}(\nu_{X})}=0$. Then~\eqref{eq:proof:corollary:selection:consistency:1} deduces that for each $\epsilon>0$
\begin{align*}
&\lim_{n\rightarrow\infty}\pr\Big\{\|D_{k}\what{f}_{\euD}^{\lambda}\|_{L^{2}(\nu_{X})}>\epsilon\Big\} \\
&\leq\lim_{n\rightarrow\infty}\pr\Big\{\|D_{k}\what{f}_{\euD}^{\lambda}-D_{k}f_{0}\|_{L^{2}(\nu_{X})}>\epsilon\Big\}=0,
\end{align*}
which implies $\|D_{k}\what{f}_{\euD}^{\lambda}\|_{L^{2}(\nu_{X})}$ goes to 0 in probability, and thus
\begin{equation}\label{eq:proof:corollary:selection:consistency:2}
\lim_{n\rightarrow\infty}\pr\Big\{\calI(\what{f}_{\euD}^{\lambda})\subseteq\calI(f_{0})\Big\}=1.
\end{equation}
On the other hand, for each relevant variable $k\in\calI(f_{0})$, it follows from~\eqref{eq:proof:corollary:selection:consistency:1} that
\begin{align*}
&\lim_{n\rightarrow\infty}\pr\Big\{\|D_{k}\what{f}_{\euD}^{\lambda}\|_{L^{2}(\nu_{X})}>\epsilon+\|D_{k}f_{0}\|_{L^{2}(\nu_{X})}\Big\} \\
&\leq\lim_{n\rightarrow\infty}\pr\Big\{\|D_{k}\what{f}_{\euD}^{\lambda}-D_{k}f_{0}\|_{L^{2}(\nu_{X})}>\epsilon\Big\}=0,
\end{align*}
where we used the triangular inequality. Since $\|D_{k}f_{0}\|_{L^{2}(\nu_{X})}>0$, we find that for each $\epsilon>0$
\begin{equation*}
\lim_{n\rightarrow\infty}\pr\Big\{\|D_{k}\what{f}_{\euD}^{\lambda}\|_{L^{2}(\nu_{X})}>\epsilon\Big\}=1,
\end{equation*}
As a consequence,
\begin{equation}\label{eq:proof:corollary:selection:consistency:3}
\lim_{n\rightarrow\infty}\pr\Big\{\calI(f_{0})\subseteq\calI(\what{f}_{\euD}^{\lambda})\Big\}=1.
\end{equation}
Combining~\eqref{eq:proof:corollary:selection:consistency:2} and~\eqref{eq:proof:corollary:selection:consistency:3} completes the proof.
\end{proof}

\section{Additional Experiments Results}
\label{appendix:experiments}

\par In this section, we present several supplementary numerical examples to complement the numerical studies in Section 6.

\subsection{Additional Examples for Derivative Estimation}
\label{appendix:experiments:derivative}

\begin{example}
We consider a toy problem in two-dimensions, where the support of the marginal distribution $\mu_{X}$ approximately coincides with the coordinate subspace $[0,1]\times\{0\}$. Precisely the first element of the covariate is uniformly sampled from $[-1,1]$, whereas the second one is drawn from a Gaussian distribution $N (0,0.05)$. The underlying regression function is $f_0(x) = x_1^2$, and labels are generated by $Y=f_0(X)+\xi$, where the noise term $\xi\sim N(0,0.1)$. The regularization parameter is set as $\lambda = 1.0\times10^{-4}$.
\end{example}

\par In all cases the accuracy on the $\supp(\mu_{X})$ is high, see Figure~\ref{fig:illustration} (top), but the least-squares regressor fails to extend the approximation and smoothness outside the support, as the least-squares loss is insensible to errors out of $\supp(\mu_{X})$. While the landscape of DORE is smoother compared to the least-squares regressor, SDORE further extends the smoothness to $[-1, 1]^2$ as it utilizes unlabeled samples from $\nu_{X}$.

\par We examine the partial derivative estimation with respect to $x_1$ and $x_2$ on $[-1, 1]^2$ and display the result in Figure~\ref{fig:illustration}.
As expected, the least-squares one is unstable compared to the DORE and SDORE. Also, we can tell from the bottom right of Figure~\ref{fig:illustration} that $x_2$ is the irrelevant variable.

\begin{figure*}[t!] 
\centering
\includegraphics[width=\linewidth]{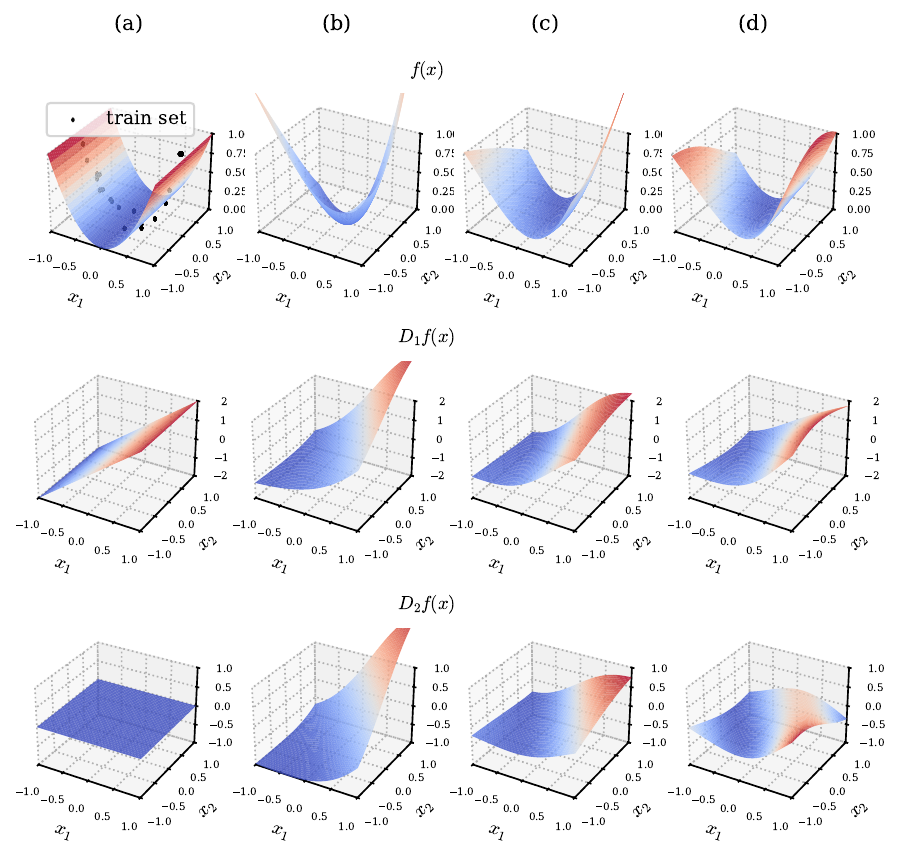}
\caption{Effect of the regularization technique on function fitting for a toy problem $f_0(x) = x_1^2$. (a) Landscape of the primitive function and its partial derivatives. The train samples are plotted in black dots. (b) least-squares fitting estimation. (c) DORE estimation. (d) SDORE estimation.} \label{fig:illustration}
\end{figure*}

\subsection{Additional Examples for Variable Selection}
\label{appendix:experiments:selection}

\begin{example}
Consider the regression function $f_0(x) = 2x_1^2 + e^{x_2} + 2 \sin (x_3) + 2 \cos (x_4 + 1)$, with observations $Y= f_0(X) + \xi$, where $X\in \mathbb{R}^{10}$ and $\xi$ is a white noise, sampled from a Gaussian distribution with the signal to noise ratio to be 25. The first four elements of $X$ are drawn from the uniform distribution on $[0,1]$, and the rest noise variables are drawn from the uniform distribution on $[0,0.05]$. The regularization parameter $\lambda$ is set as $1.0\times10^{-4}$ for SDORE.
\end{example}

\par We repeat the process for least-squares regressor and SDORE, respectively, and evaluate the models on a test set with sample size 1000. We report the estimated mean square of partial derivative by both estimators with respect to $x_i$, the mean selection error (mean of false positive rate and false negative rate) the root mean squared prediction error on the primitive function in Figure~\ref{fig:comparison}. The results indicate least-squares regression fails to identify the correct dependent variables and has larger prediction error. In contrast, SDORE yields smaller prediction error, and points out that $x_1$ to $x_4$ are the relevant variables.

\begin{figure}[htbp]  
\centering
\includegraphics[width=\linewidth]{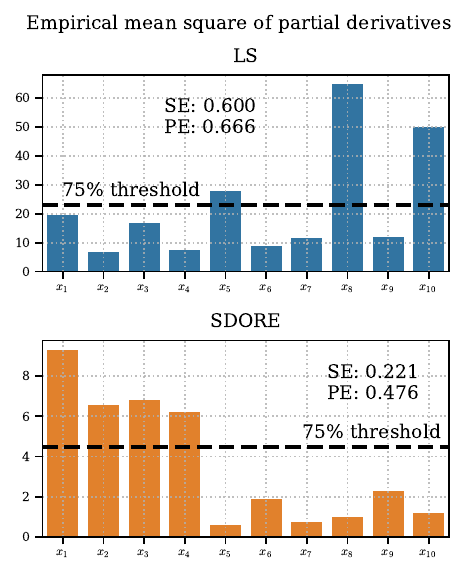}
\caption{(left) Empirical mean square of the partial derivatives estimated by least-squares regression (LS) and SDORE on a variable selection problem in $\mathbb{R}^{10}$ where $f_0$ is dependent on the $x_1$ to $x_4$. 
The dashed line is the 75 \% quantile threshold for variable selection. 
(center) Mean variable selection error for the estimated derivative function on test set. 
(right) Root mean squared prediction error for the primitive function on test set.}
\label{fig:comparison}
\end{figure}

\section*{Acknowledgment}
The authors would like to express their sincere gratitude to the Editor-in-Chief, Associate Editor, and all reviewers for their careful review of the manuscript. Their insightful comments and constructive suggestions have significantly improved the quality and clarity of this paper. This work is supported by the National Key Research and Development Program of China (No. 2024YFA1014202), by the National Natural Science Foundation of China (No. 12125103, No. U24A2002, No. 12371441), and by the Fundamental Research Funds for the Central Universities. The numerical calculations in this paper have been done on the supercomputing system in the Supercomputing Center of Wuhan University.

\bibliographystyle{IEEEtran}
\bibliography{references}

\begin{IEEEbiographynophoto}{Zhao Ding}
received the B.Sc. degree in Mathematics and Applied Mathematics from Wuhan University, Wuhan, China, in 2020, where he is also currently pursuing the Ph.D. degree with School of Mathematics and Statistics.

His research interests include the theory of deep learning and generative learning based on diffusion models.
\end{IEEEbiographynophoto}

\begin{IEEEbiographynophoto}{Chenguang Duan} received the B.S. degree in information and computational science from Wuhan University, Wuhan, China, in 2020, where he is currently pursuing the Ph.D. degree with the School of Mathematics and Statistics. He has authored research articles including \emph{Journal of Computational and Applied Mathematics}, and \emph{Communications in Computational Physics}. His research interests lie in deep learning theory, generative models, and scientific maching learning.
\end{IEEEbiographynophoto}

\begin{IEEEbiographynophoto}{Yuling Jiao} received the B.Sc. degree in applied mathematics from Shangqiu Normal University, Shangqiu, China, in 2008, and the Ph.D. degree in applied mathematics from Wuhan University, Wuhan, China, in 2014.
 
He is currently a Full Professor with the School of Artificial Intelligence, Wuhan University. His research works have been published in journals and conferences including \emph{SIAM J. Math. Anal.}, \emph{SIAM J.Control Optim.}, \emph{SIAM J. Numer. Anal.}, \emph{SIAM J. Sci. Comput.}, \emph{SIAM J. Math. Data. Sci.}, \emph{Appl. Comput. Harmon. Anal.} \emph{J. Mach. Learn. Res.}, \emph{IEEE Trans. Inf. Theory}, \emph{Ann. Stat.}, \emph{J. Amer. Statist. Assoc.}, \emph{Statist. Sci.}, \emph{Inverse Probl.}, \emph{IEEE Trans. Signal Process.}, \emph{Nat. Commun.}, \emph{ICML} and \emph{NeurIPS}. His research insterests include machine learning and scientific computing.
\end{IEEEbiographynophoto}

\begin{IEEEbiographynophoto}{Jerry Zhijian Yang} received the B.S. and M.S. degrees from Peking University, Beijing, China, in 1999 and 2001, respectively, and the Ph.D. degree in applied and computational mathematics from Princeton University, Princeton, NJ, USA, in 2006. 

He completed his post-doctoral research at the California Institute of Technology, Pasadena, CA, USA, in 2008. He is currently a Full Professor with the School of Mathematics and Statistics as well as Wuhan Institute for Math \& Al, Wuhan University, Wuhan, China. He has authored or coauthored research articles, including \emph{Numerische Mathematik}, \emph{Physical Review B}, \emph{SIAM Multiscale Modeling and Simulation}, \emph{SIAM Journal on Control and Optimization}, \emph{Journal of Computational Physics}, \emph{The Journal of Chemical Physics}, \emph{International Journal for Numerical Methods in Engineering}, and \emph{Communications in Computational Physics}. His research interests include multiscale modeling and simulation, machine learning, and scientific computing.
\end{IEEEbiographynophoto}

\end{document}